\documentclass{article}

\usepackage{microtype}
\usepackage{graphicx}
\usepackage{subcaption}
\usepackage{booktabs} 

\usepackage{hyperref}
\usepackage{multirow}
\usepackage{multicol}

\usepackage[accepted]{icml2024}

\usepackage{amsmath}
\usepackage{amssymb}
\usepackage{mathtools}
\usepackage{amsthm}

\usepackage[capitalize,noabbrev]{cleveref}

\theoremstyle{plain}
\newtheorem{theorem}{Theorem}[section]

\theoremstyle{definition}
\newtheorem{definition}[theorem]{Definition}

\theoremstyle{remark}
\newtheorem{remark}[theorem]{Remark}

\usepackage[textsize=tiny]{todonotes}

\usepackage{tikz}

\newcommand{\hide}[1]{} 

\newcommand{\ceil}[1]{\ensuremath{\left\lceil#1\right\rceil}}

\newcommand\Rbb{\ensuremath{{\mathbb{R}}}}

\newcommand\Fc{\ensuremath{\mathcal{F}}}

\newcommand\Xc{\ensuremath{{\mathcal{X}}}}

\usepackage{tikz}

\icmltitlerunning{Accelerating Parallel Sampling of Diffusion Models}

\begin{document}

\twocolumn[
\icmltitle{Accelerating Parallel Sampling of Diffusion Models}



\icmlsetsymbol{equal}{*}

\begin{icmlauthorlist}
\icmlauthor{Zhiwei Tang}{yyy,equal}
\icmlauthor{Jiasheng Tang}{sch,comp}
\icmlauthor{Hao Luo}{sch,comp}
\icmlauthor{Fan Wang}{sch}
\icmlauthor{Tsung-Hui Chang}{yyy,bbb}
\end{icmlauthorlist}

\icmlaffiliation{yyy}{School of Science and Engineering, The Chinese University of Hong Kong, Shenzhen, China}
\icmlaffiliation{comp}{Hupan Lab, Zhejiang Province}
\icmlaffiliation{sch}{DAMO Academy, Alibaba Group}
\icmlaffiliation{bbb}{Shenzhen Research Institute of Big Data, Shenzhen, China}

\icmlcorrespondingauthor{Zhiwei Tang}{zhiweitang1@link.cuhk.edu.cn}

\icmlkeywords{Machine Learning, ICML}

\vskip 0.3in
]



\printAffiliationsAndNotice{\textsuperscript{*}This work was done when Zhiwei Tang was intern at DAMO Academy.} 

\begin{abstract}
  Diffusion models have emerged as state-of-the-art generative models for image generation. However, sampling from diffusion models is usually time-consuming due to the inherent autoregressive nature of their sampling process.  In this work, we propose a novel approach that accelerates the sampling of diffusion models by parallelizing the autoregressive process. Specifically, we reformulate the sampling process as solving a system of \emph{triangular nonlinear equations} through fixed-point iteration. With this innovative formulation, we explore several systematic techniques to further reduce the iteration steps required by  the solving process. Applying these techniques, we introduce \textbf{ParaTAA}, a universal and \textbf{training-free} parallel sampling algorithm that can leverage extra computational and memory resources to increase the sampling speed. Our experiments demonstrate that ParaTAA can decrease the inference steps required by common sequential sampling algorithms such as DDIM and DDPM by a factor of \textbf{4$\sim$14 times}. Notably, when applying ParaTAA with 100 steps DDIM for Stable Diffusion, a widely-used text-to-image diffusion model, it can produce the same images as the sequential sampling in only \textbf{7 inference steps}. The code is available at \url{https://github.com/TZW1998/ParaTAA-Diffusion}.
\end{abstract}

\section{Introduction}

In recent years, diffusion models have been recognized as state-of-the-art for generating high-quality images, demonstrating exceptional resolution, fidelity, and diversity \cite{ho2020denoising,dhariwal2021diffusion,song2020score}. These models are also notably easy to train and can be effectively extended to conditional generation \cite{ho2022classifier}. Broadly speaking, diffusion models work by learning to reverse the diffusion of data into noise, a process that can be described by a stochastic differential equation (SDE) \cite{song2020score,karras2022elucidating}:
\begin{align}
\label{p:diffusion}
d {x}_t = f(t) {x}_tdt+g(t) d w_t,
\end{align}
where $d w_t$ is the standard Wiener process, and $f(t)$ and $g(t)$ are the drift and diffusion coefficients, respectively. The reverse process relies on the score function $ \epsilon({x}_t,t) \stackrel{\text{def.}}= \nabla_x\log p({x}_t)$, and its closed form can be expressed either as an ordinary differential equation (ODE) \cite{song2020score}:
\begin{align}
\label{p:ode}
&d {x}_t = \left(f(t){x}_t-\frac{1}{2}g^2(t)\epsilon({x}_t,t)\right)dt,
\end{align}
or as an SDE:
\begin{align}
\label{p:sde}
&d{x}_t = \left(f(t){x}_t-g^2(t)\epsilon({x}_t,t)\right)dt + g(t)d{ w_t}.
\end{align}
With the ability to evaluate $\epsilon({x}_t,t)$, it becomes possible to generate samples from noise by numerically solving the ODE \eqref{p:ode} or the SDE \eqref{p:sde}. The training process, therefore, involves learning a parameterized surrogate $\epsilon_\theta({x}_t,t)$ for $\epsilon({x}_t,t)$ following a denoising score matching framework described in  \cite{song2020score,karras2022elucidating}.

\textbf{Accelerating Diffusion Sampling.} As previously mentioned, the sampling process in diffusion generative models involves solving the ODE \eqref{p:ode} or SDE \eqref{p:sde}. This process requires querying the learned neural network $\epsilon_\theta$ in an autoregressive way, which can limit sampling speed particularly when $\epsilon_\theta$ represents a large model such as Stable Diffusion (SD) \cite{rombach2022high}. To accelerate the sampling process, existing works explore several avenues, which we summarize briefly here.

One avenue is to distill the ODE trajectory of the diffusion sampling process  into another neural network that enables fewer-step sampling, with representative works including \cite{song2023consistency,liu2023instaflow,sauer2023adversarial,salimans2022progressive,meng2023distillation,geng2023one}. However, this class of methods often leads to  degradation in image quality and diversity.

Another direction involves developing faster sequential ODE/SDE solvers for \eqref{p:ode}/\eqref{p:sde} based on mathematical principles, with contributions from \cite{lu2022dpm,song2020denoising,karras2022elucidating,zhao2023unipc}. However, the improvements from these approaches tend to be incremental, given the years of progress in the field.

A recent and promising direction, pioneered by \cite{shih2023parallel}, aims to parallelize the autoregressive sampling process of diffusion models by employing Picard-Lindelöf (PL) iterations for solving the corresponding ODE/SDE. This approach has three main advantages over other existing methods: 1. It does not require extra training; 2. It can lead to (almost) the same images as sequential sampling; 3. It can significantly reduce the inference steps by leveraging extra computing resources. Similar concepts of parallelizing autoregressive inference have also been investigated in the acceleration of Large Language Models (LLMs), such as speculative sampling \cite{leviathan2023fast,sun2023spectr}, and in common autoregressive procedures  \cite{song2021accelerating,lim2023parallelizing}. We focus on this direction in this work, proposing a novel and more efficient algorithm for parallelizing the sampling process of diffusion models.

\subsection{Prior Work}
To the best of our knowledge, the recent work by  \cite{shih2023parallel} stands as the only study focusing  on the parallel sampling of diffusion models.  For a general ODE expressed as $x_t = \int_{0}^t S(x_u,u)du$, the PL iteration adopted in \cite{shih2023parallel} refines an initial discretized trajectory $x_0^{\text{old}},...,x_{T}^{\text{old}}$ through the following fixed-point iteration:
\begin{align}
\label{p:PL}
x_{i}^{\text{new}} =\frac{1}{T} \sum_{u=0}^{i-1} S\left(x_u^{\text{old}},\frac{u}{T}\right),\ \text{for}\ i=0,...,T.
\end{align}
This approach allows the computationally intensive task, evaluating ${S\left(x_u^{\text{old}},\frac{u}{T}\right): u=0,...,T}$, to be executed in parallel. In practice, \cite{shih2023parallel} observed that the PL iteration \eqref{p:PL}  requires significantly fewer than $T$ steps to converge, thus expediting the autoregressive sampling process.

\subsection{Our Contributions}
In this paper, we introduce a novel and principled formulation for the parallel sampling of diffusion models, which includes the method proposed by  \cite{shih2023parallel} as a special case. The primary advantage of this new formulation is that it enables us to rigorously investigate its convergence properties, thus new techniques to improve sampling efficiency are made possible. Besides, differing from \cite{shih2023parallel}, our study is exclusively concentrated on image generation. Specifically,  our contributions are:

\textbf{(1)} We formulate the parallel sampling of diffusion models as solving a system of \emph{triangular nonlinear equations} using fixed-point iteration (FP), which can be seamlessly integrated with any existing sequential sampling algorithms by adjusting the coefficients in the equations.

\textbf{(2)} Inspired by classical optimization theory on nonlinear equations, we develop several techniques to enhance the efficiency of FP. Firstly, we reveal that the convergence behavior of FP is largely attributed to the iteration function, and propose a systematic way to construct an improved iteration function via equivalent transformation on the nonlinear equations. Secondly, to efficiently bootstrap the information from previous iterations, we propose a new variant of the Anderson Acceleration technique \cite{walker2011anderson} tailored for the triangular nonlinear equations. Lastly, we identify two practical tricks through experiments: early stopping—terminating the iteration once a perceptual criterion is met in the generated image; and a useful initialization strategy—initializing the process with the solution from a similar, previously solved equation.

\textbf{(3)} As a byproduct, particularly for text-to-image generation with Stable Diffusion, we observe that when initializing with the sampling trajectory of a similar prompt, one can obtain a smooth interpolation between the source image and the target image in very few steps. This can have implications for tasks such as image variation, editing \cite{meng2022sdedit}, and prompt optimization \cite{hao2022optimizing}.

\textbf{Paper Outline.} We begin by formulating the diffusion sampling problem as solving triangular nonlinear systems in Section \ref{sec:formulation}, and then discuss how to obtain a better iteration function for FP. In Section \ref{sec:anderson}, we introduce how data from previous iterations should be used to speed up the iteration process. Subsequently, we discuss the two useful tricks to further enhance sampling efficiency in Section \ref{sec:early_init}. Lastly, Section \ref{sec:exp} presents experimental results on cutting-edge image diffusion models, demonstrating the effectiveness of our proposed methods.

\section{Formulating Diffusion Sampling as Solving Triangular Nonlinear Equations}
\label{sec:formulation}

We observe that every existing sampling algorithm for diffusion models, such as DDIM \cite{song2020denoising}, DPM-Solver \cite{lu2022dpm}, and Heun \cite{karras2022elucidating}, follows the autoregressive procedure in \eqref{p:autoregressive_complicated}. Let $T$ denote the discretization steps for the ODE/SDE, and $\xi_0,..,\xi_{T}$ be noise vectors drawn from standard Gaussian distribution. Starting with $x_{T} = \xi_{T}$, one computes $x_{T-1},...,x_0$ sequentially via  the following equation from $t=T$ to $t=1$:
\begin{align}
    \label{p:autoregressive_complicated}
    x_{t-1} &= \sum_{i=t}^{T}a_{t,i} x_i+\sum_{i=t}^{T}b_{t,i}\epsilon_\theta(x_i,i)+c_{t-1}\xi_{t-1},
\end{align}
where $a_{t,i},b_{t,i},c_t$ are coefficients determined by the specific sampling algorithm. Notably, for ODE solvers like DDIM \cite{dhariwal2021diffusion}, it holds that $c_0=...=c_{T-1}=0$, whereas for SDE solvers like DDPM \cite{ho2020denoising}, $c_0,...,c_{T-1}$ are all non-zero.

For simplicity and due to a limit time, this work focuses on commonly used first-order solvers such as DDIM and DDPM, while leaving extensions to higher-order solvers like DPM-Solver and Heun as future works. For first-order solvers, \eqref{p:autoregressive_complicated} can be simplified to:
\begin{align}
    \label{p:autoregressive}
    x_{t-1} &= a_tx_t+b_t\epsilon_\theta(x_t,t)+c_{t-1}\xi_{t-1},\ t=1,...,T.
\end{align} 
Following the insights from \cite{song2021accelerating}, we found that this autoregressive procedure \eqref{p:autoregressive} can be viewed as  triangular nonlinear equations with $x_0,...,x_{T-1}$ as the unknown variables. Besides, by further examination on \eqref{p:autoregressive}, we reveal that these equations can be expressed in various equivalent forms. For instance, by incorporating the $(t+1)$-th equation into the first term of the $t$-th equation in \eqref{p:autoregressive}, we derive an alternative $t$-th equation:
\begin{align}
    \label{p:order2}
    x_{t-1}=& a_t\bigg(\underbrace{{ a_{t+1}x_{t+1}+b_{t+1}\epsilon_\theta(x_{t+1},t+1)+c_{t}\xi_{t}}}_{{=x_t}}\bigg)\notag\\ &\qquad\qquad+b_t\epsilon_\theta(x_t,t)+c_{t-1}\xi_{t-1}.
\end{align}

This leads us to define a series of equivalent nonlinear systems for the autoregressive procedure \eqref{p:autoregressive}.
\begin{definition}[$k$-th order nonlinear equations]\label{def:orderk} 
    For any $1\leq k\leq T$ with $x_T = \xi_T$, we define 
    \begin{align}
        \label{p:orderk}
        x_{t-1}=F^{(k)}_{t-1}(x_t,x_{t+1},...,x_{t_k})\ , t=1,...,T
    \end{align}
     as {\it the $k$-th order nonlinear equations} for the  autoregressive sampling procedure \eqref{p:autoregressive}, where $F^{(k)}_{t-1}$ is defined as
    \begin{align}
        \label{p:F_orderk}
        &F^{(k)}_{t-1}(x_t,x_{t+1},...,x_{t_k}) \stackrel{\text{def.}}=
       \bar a_{t,t_k}x_{t_k}\notag\\ +&\sum_{j=t}^{t_k}\bar a_{t,j-1}b_j\epsilon_\theta(x_j,j)  +\sum_{j=t}^{t_k}\bar a_{t,j-1}c_{j-1}\xi_{j-1},
    \end{align} and $t_k\stackrel{\text{def.}}=\min\{t+k-1,T\}$, $\bar a_{i,s}=\prod_{j=i}^{s}a_j$.  We denote $\bar a_{i,s}=1$ for $s<i$.
\end{definition}

From this definition, it is evident that the equations \eqref{p:orderk} with $k=1$ correspond exactly to the autoregressive sampling procedure \eqref{p:autoregressive}. Regarding this family of nonlinear equations, we assert the following:
\begin{theorem}
\label{thm:unique}
 The nonlinear equations \eqref{p:orderk} with different orders $k$  are all equivalent and possess a unique solution.
\end{theorem}

Fixed-point iteration is a classical method for solving nonlinear equations like \eqref{p:orderk}. Given the set of variables $x_0^{i},...,x_{T-1}^{i}$ at the $i$-th iteration, the fixed-point iteration calculates the $(i+1)$-th iteration as follows:
\begin{align}
    \label{p:fixedpoint}
    x_{t-1}^{i+1}=F^{(k)}_{t}(x_{t}^{i},x_{t+1}^{i},...,x_{t_k}^{i}), \quad t=1,...,T.
\end{align}
As can be seen, performing one iteration in \eqref{p:fixedpoint} involves evaluating $\epsilon_\theta(x_1^{i},1),...,\epsilon_\theta(x_{T}^{i},T)$, which equates to inferring the neural network $\epsilon_\theta$ $T$ times. Fortunately, with sufficient computational resources like GPUs, these evaluations can be processed all in parallel, making the time cost comparable to a single query of $\epsilon_\theta$. Crucially, as demonstrated in Section \ref{sec:exp} and also \cite{shih2023parallel}, fixed-point iteration \eqref{p:fixedpoint} typically requires significantly less than $T$ steps to generate a sample matching the one obtained via autoregressive procedure \eqref{p:autoregressive}, thus accelerating the sampling process. 

Notably, the selection of order $k$ for the nonlinear equations influences the computational graph in the fixed-point iteration \eqref{p:fixedpoint}—determining the number of variables from later timesteps that are employed to update the variables from earlier timesteps. We will explore the effect of order $k$ on the convergence of the fixed-point iteration in Section \ref{sec:order_effect} with greater details.

\subsection{Stopping Criterion} 
\label{sec:stopping}
To examine the convergence of the fixed-point iteration \eqref{p:fixedpoint}, we can employ the residuals of the nonlinear equations \eqref{p:orderk} for a stopping criterion. Furthermore, given the equivalence of nonlinear equations \eqref{p:orderk} across different orders $k$, a universal stopping criterion is applicable for all. In this study, we choose to use the residuals of the first-order equations for the stopping criterion. Specifically, the residual for the $t$-th equation in \eqref{p:orderk} is defined as:
\begin{align}
    \label{p:residual}
    r_{t-1}\stackrel{\text{def.}}=\|x_{t-1}-a_tx_t-b_t\epsilon_{\theta}(x_t,t)-c_{t-1}\xi_{t-1}\|_2^2
\end{align} 
Owing to the triangular structure of \eqref{p:orderk}, for any $0<t\leq T$, we can conclude the convergence of the variables $x_{t-1},...,x_{T-1}$ if the conditions $r_{t-1}\leq \varepsilon_{t-1},...,r_{T-1}\leq \varepsilon_{T-1}$ are met, where $\varepsilon_0,...,\varepsilon_{T-1}$ represent predetermined time-dependent thresholds. Following previous research \cite{shih2023parallel}, we set $\varepsilon_t$ to $\tau^2g^2(t)d$, with $\tau$ as the tolerance hyperparameter, $d$ as the data dimension, and $g(t)$ as the diffusion coefficient from \eqref{p:diffusion}. Once the variables $x_{t-1},...,x_{T-1}$ have converged, further updates are unnecessary, and they can remain fixed.

\subsection{Saving Computation By Solving Subequations }When $T$ is large, computing $\epsilon_\theta(x_1^i,1),...,\epsilon_\theta(x_T^i,T)$ simultaneously may demand substantial memory. To address this, prior work \cite{shih2023parallel} introduced the concept of a sliding window—solving only a  lower triangular subequations in \eqref{p:orderk} at a time. For instance, with a window size $w$, one could initially iterate over the variables $x_{T-w},...,x_{T-1}$ by resolving the corresponding subequations. Once the variables $x_{t-1},...,x_{T-1}$ converge, as determined by the stopping criterion detailed in Section \ref{sec:stopping}, the iteration window can be shifted to update $x_{t-w},...,x_{t-1}$ through their respective subequations.

\subsection{Effect of the Order of Nonlinear Equations}
\label{sec:order_effect}
We have found that despite the equivalence of the nonlinear system \eqref{p:orderk} across different orders $k$, the order $k$ influences the optimization landscape of the nonlinear system \eqref{p:orderk}, and consequently, the convergence speed of the fixed-point iteration. 
It is known that the speed of convergence is associated with the Lipschitz constant of the function $F^{(k)}_{t-1}$ \cite{argyros2013computational}.
If $k$ is excessively large, the Lipschitz constant of $F^{(k)}_{t-1}$ could be potentially large, since it incorporates more variables,  leading to instability and slower convergence. Conversely, the fixed-point iteration \eqref{p:fixedpoint} generally requires at least $\ceil{\frac{T-1}{k}}$ steps to converge due to the structure of the computational graph. This is because $x_{t-1}$ is updated using information from $x_{t},...,x_{t_k}$, meaning the initial condition $x_{T}=\xi_T$ can only influence $x_0$ after $\ceil{\frac{T-1}{k}}$ iterations. 

Hence, an appropriate value of $k$ is crucial for expediting the fixed-point iteration. We examined this by running fixed-point iteration \eqref{p:fixedpoint} under various $k$ for the DDIM \cite{song2020denoising} and DDPM \cite{ho2020denoising} sampling algorithms with 100 steps, using the DiT model \cite{peebles2023scalable}. The window size $w$ is set to 100. Figure \ref{fig:order} illustrates the impact of $k$ on the convergence of residuals $\sum_{t=1}^T r_{t-1}$. As observed, small values of $k$ lead to slow convergence of residuals, whereas large $k$ values result in instability, particularly at the beginning for DDIM with $T=100$.

\begin{remark}
While we provide insight into how the order affects fixed-point iteration convergence, predicting the optimal $k$ from a theoretical standpoint is generally not feasible, since the neural network $\epsilon_\theta$ is a black-box.  Thus, we recommend treating $k$ as a hyperparameter and selecting the optimal one based on empirical performance. Appendix \ref{app:hyperparameter} contains grid search results on the effect of $k$ on the convergence speed for different sampling algorithms.
\end{remark}

\begin{remark}
It is noteworthy that the PL iteration employed by prior work \cite{shih2023parallel} is equivalent to applying a fixed-point iteration to solve the nonlinear equations \eqref{p:orderk} with order $k$ equal to the chosen window size $w$, and thus it corresponds to the $k=100$ in Figure \ref{fig:order}.
\end{remark}

\begin{figure}[htpb]
	\centering
	\begin{subfigure}[b]{0.23\textwidth}
	\centering
	\includegraphics[width=\textwidth]{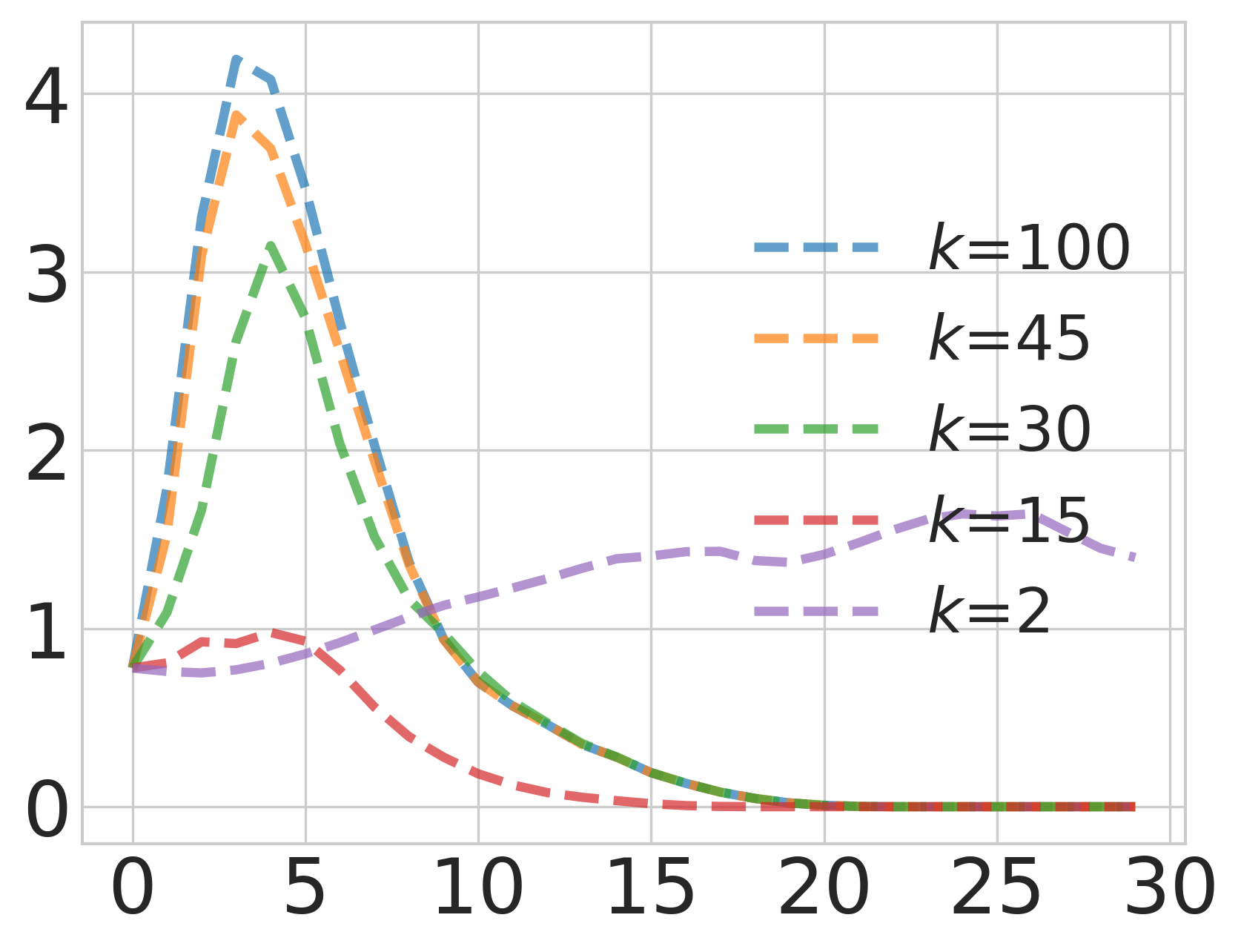}
 	\caption{DDIM 100 steps}
	\label{}
\end{subfigure}
	\begin{subfigure}[b]{0.22\textwidth}
	\centering
	\includegraphics[width=\textwidth]{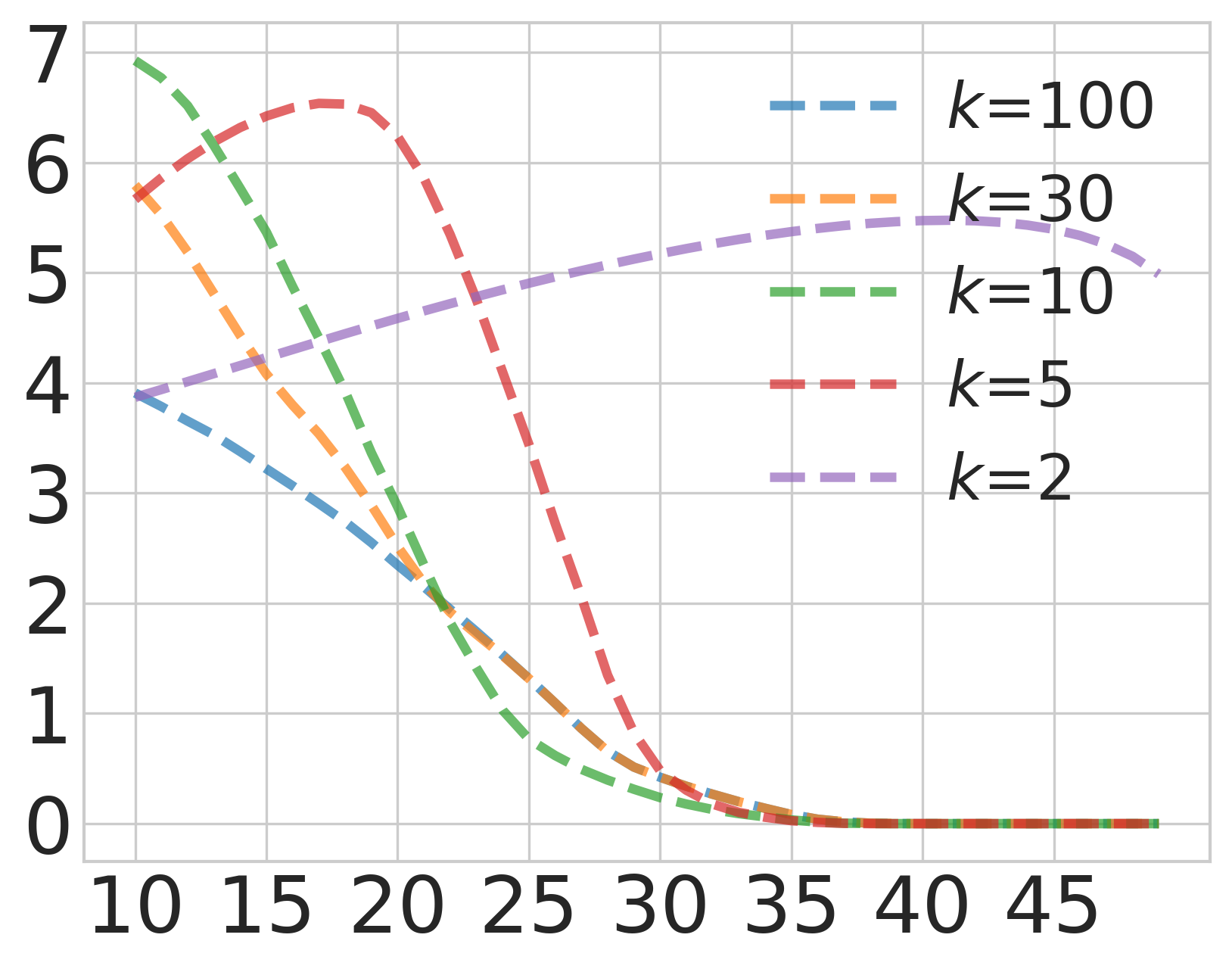}
	\caption{DDPM 100 steps}
	\label{}
\end{subfigure}
\caption{Convergence of residuals under different orders. x-axis is the iteration steps while y-axis is the value of $\sum_{t=1}^T r_{t-1}$.}
	\label{fig:order}
\end{figure}

\section{Anderson Acceleration for Triangular Nonlinear Equations}
\label{sec:anderson}

\textbf{Anderson Acceleration (AA)} \cite{anderson1965iterative} is a classical method for expediting fixed-point iterations, which is extensively utilized across various engineering disciplines \cite{walker2011anderson}. The central idea of AA is to leverage information from previous iterations to approximate the inverse Jacobian of the nonlinear system and to implement a Newton-like update using this approximation.
In this section, we explore the use of AA within the context of parallel sampling of diffusion models and the resolution of triangular nonlinear systems. First of all, let us describe a straightforward implementation of standard AA for the fixed-point iteration \eqref{p:fixedpoint} with the use of up to $m \ (m\geq1)$ previous iterations. With the initialization $x_0^0,...,x_{T-1}^0$, the process begins with a standard fixed-point iteration as indicated by \eqref{p:fixedpoint}. For the $i$-th iteration with $i\geq1$, we introduce the following notations.

\textbf{Notations. } Let $m_i=\min\{m,i\}$, $\Delta x_{t}^i = x_t^{i+1} - x_t^i$, $\Xc^i_t=\left[\Delta x_{t}^{i-m_i},...,\Delta x_{t}^{i-1}\right]$, 
$R_t^i= F^{(k)}_{t}\left(x_{t+1}^{i},...,x_{{(t+1)}_k}^{i}\right) - x_{t}^{i}$, 
    $\Delta R_t^i=R^{i+1}_t-R^i_t$, 
    $\Fc^i_t= \left[\Delta R^{i-m_i}_t,...\Delta R^{i-1}_t\right]$. For any $0\leq t_1\leq t_2<T$ and any vectors/matrixes $v_1,...,v_{T-1}$, we denote $v_{t_1:t_2}=\left[v^\top_{t_1},...,v^\top_{t_2}\right]^\top$. For any matrix $V$, we denote $V[i:j,t:s]$ as the submatrix of $V$ with rows $i,...,j$ and columns $t,...,s$. If $j$ and $s$ are not specified, denote $j=T-1$ and $s=T-1$.

Assuming that the subequations in \eqref{p:orderk} for $t=t_1,...,t_2$ are being solved and that $\Fc^{i\top}_{t_1:t_2} \Fc^i_{t_1:t_2}$ has full rank, the update rule for  (AA) is provided by the following equation:
\begin{align}
    \label{p:anderson}
    x^{i+1}_{t_1:t_2} = x^i_{t_1:t_2} - G^i R^i_{t_1:t_2},
\end{align}
where $G^i$ is considered an approximate inverse Jacobian of $R^i_{t_1:t_2}$, and is computed as follows:
\begin{align}
    \label{p:G}
    &G^i = -I + (\Xc^i_{t_1:t_2}+\Fc^i_{t_1:t_2})(\Fc^{i\top}_{t_1:t_2} \Fc^i_{t_1:t_2})^{-1}\Fc^{i\top}_{t_1:t_2}
\end{align}
The justification for \eqref{p:G} is that $G^i$ satisfies the \textbf{Inverse Multisecant Condition} \cite{fang2009two}:
\begin{align}
    \label{p:inverse_multisecant}
    G^i \Fc^i_{t_1:t_2} = \Xc^i_{t_1:t_2},
\end{align}
and the Frobenius norm $\left\|G^i+I\right\|_F$ is the smallest possible for all matrices meeting this condition \eqref{p:inverse_multisecant} \cite{walker2011anderson}. It is evident from \eqref{p:anderson} that when $G^i$ is set to $-I$, the AA update simplifies to the standard fixed-point iteration.


\subsection{Triangular Anderson Acceleration}

We identified a critical issue in the update rule of the AA as given in \eqref{p:anderson}: For some timesteps  $j<t$, the update of $x^{i+1}_t$ could be influenced by the value of $x_{j}^i$ due to the matrix $G^i$ potentially being dense. This has occasionally led to numerical instability in our practices\footnote{Specifically, we have observed instances of numerical overflow when applying AA with 16-bit precision.}. To understand this instability, we find that $x_t$ always converges before $x_j$\footnote{Refer to Figure \ref{fig:residual_pattern} in Appendix \ref{app:triangular_anderson} for empirical evidence.}, which suggests that using the state of $x_j$ to update $x_t$ can be counterproductive, particularly when $x_t$ is near convergence but $x_j$ is not.

Armed with this key observation, we propose an adapted version of AA that is well-suited for triangular nonlinear equations like \eqref{p:orderk}. The principal idea is to constrain the matrix $G^i$ in \eqref{p:anderson} to be \emph{block upper triangular}—A formal definition is given as follows:
\begin{definition}[Block Upper Triangular Matrix] \label{def:block_upper_triangular}
Consider a matrix $G\in\Rbb^{(t_2-t_1)d\times (t_2-t_1)d}$. We define $G$ as block upper triangular if, for any $t_1 \leq t \leq t_2$, $j \leq (t-t_1)d$, and $1 \leq s \leq d$, it holds that $G[(t-t_1)d + s,j] = 0$.
\end{definition}

By doing so, the updated value $x^{i+1}_t$ in \eqref{p:anderson} receives information exclusively from those $x_{j}^i$ with $j \geq t$. In the subsequent theorem, we present a closed-form solution that fulfills both the inverse multisecant condition \eqref{p:inverse_multisecant} and the block upper triangular stipulation as defined in Definition \ref{def:block_upper_triangular}, while also being optimally close to $-I$ with respect to the Frobenius norm.

\begin{theorem}
\label{thm:G_upper}
Assume $m<d$ and that $\Fc^{i\top}_{t_2} \Fc^i_{t_2}$ has full rank. Let $Q^i\in \Rbb^{(t_2-t_1)d\times (t_2-t_1)d}$ be a block upper triangular matrix, and for any $t_1\leq t\leq t_2$:
\begin{align}
\label{p:Q}
Q^i\left[t':t'',t':\right] = (\Xc^i_{t}+\Fc^i_{t})(\Fc^{i\top}_{t:t_2} \Fc^i_{t:t_2})^{-1}\Fc^{i\top}_{t:t_2},
\end{align}where $t' \stackrel{\text{def.}}{=} (t-t_1)d + 1$ and $t'' \stackrel{\text{def.}}{=} (t-t_1)d + d$.
Then the matrix $T^i=-I+Q^i$ meets both the inverse multisecant condition \eqref{p:inverse_multisecant} and the block upper triangular requirement from Definition \ref{def:block_upper_triangular}, and $\left\|T^i+I\right\|_F$ is minimal among all matrices that comply with these conditions.
\end{theorem}

Employing the $T^i$ derived from Theorem \ref{thm:G_upper}, we introduce a tailored update rule for AA in the context of triangular nonlinear equations:
$x^{i+1}_{t_1:t_2} = x^i_{t_1:t_2} - T^i R^i_{t_1:t_2}$, and we refer this method as \textbf{Triangular Anderson Acceleration} (\textbf{TAA}). In this study, we do not undertake a detailed theoretical analysis on TAA. This omission is because even the theoretical aspects of standard AA are still actively being researched in the field of optimization \cite{evans2020proof,rebholz2023effect}. Instead, we concentrate on assessing the empirical performance of this new type of Anderson Acceleration approach. 

Figure \ref{fig:ad5_order} shows the results of comparing fixed-point iteration, AA, and TAA in the same scenario as Figure \ref{fig:order}. We observe that both AA and TAA improve upon the optimal fixed-point iteration from Figure \ref{fig:order} by a large margin, regardless of the $k$ used. Moreover, TAA is notably faster than AA, especially for the DDPM with 100 steps, and it remains stable even when using 16-bit precision for calculations. Additionally, similar to the fixed-point iteration, TAA  can also benefit from selecting an optimal $k$.

\begin{remark}
In practice, we utilize $(\Fc^{i\top}_{t:t_2} \Fc^i_{t:t_2} + \lambda I)^{-1}$ with $\lambda > 0$ being a small constant, for stabilizing the computation of $T^i$ in \eqref{p:Q}. 
\end{remark}

\begin{remark}
Apart from the method for determining $T^i$ as outlined in Theorem \ref{thm:G_upper}, we also explored a heuristic approach to acquire a block upper triangular matrix by directly extracting the upper triangular portion of $G^i$ from \eqref{p:G}. While this method also enhances standard AA, it still faced numerical instability and was less effective compared to the approach using $T^i$ from Theorem \ref{thm:G_upper}. Further details are available in Appendix \ref{app:triangular_anderson}.
\end{remark}

\begin{remark}
The computation of the matrix $T^i$ in Theorem \ref{thm:G_upper} adds only minimal computational and memory overhead to the standard fixed-point iteration \eqref{p:fixedpoint}. Firstly, the storage for the history matrices $\Fc^{i}_{t_1:t_2}$ and $\Xc^i_{t_1:t_2}$, of dimension $(t_2-t_1)d \times m_i $, is neglectable compared to that of the neural network $\epsilon_\theta$. Secondly, the operations in \eqref{p:Q} consist of simple matrix multiplication and inversion; the matrix $\Fc^{i\top}_{t:t_2} \Fc^i_{t:t_2} \in \Rbb^{m_i \times m_i}$ can be efficiently computed as the value of $m$ is typically chosen to be between 2 and 5.
\end{remark}

\begin{figure}[htpb]
	\centering
	\begin{subfigure}[b]{0.23\textwidth}
	\centering
	\includegraphics[width=\textwidth]{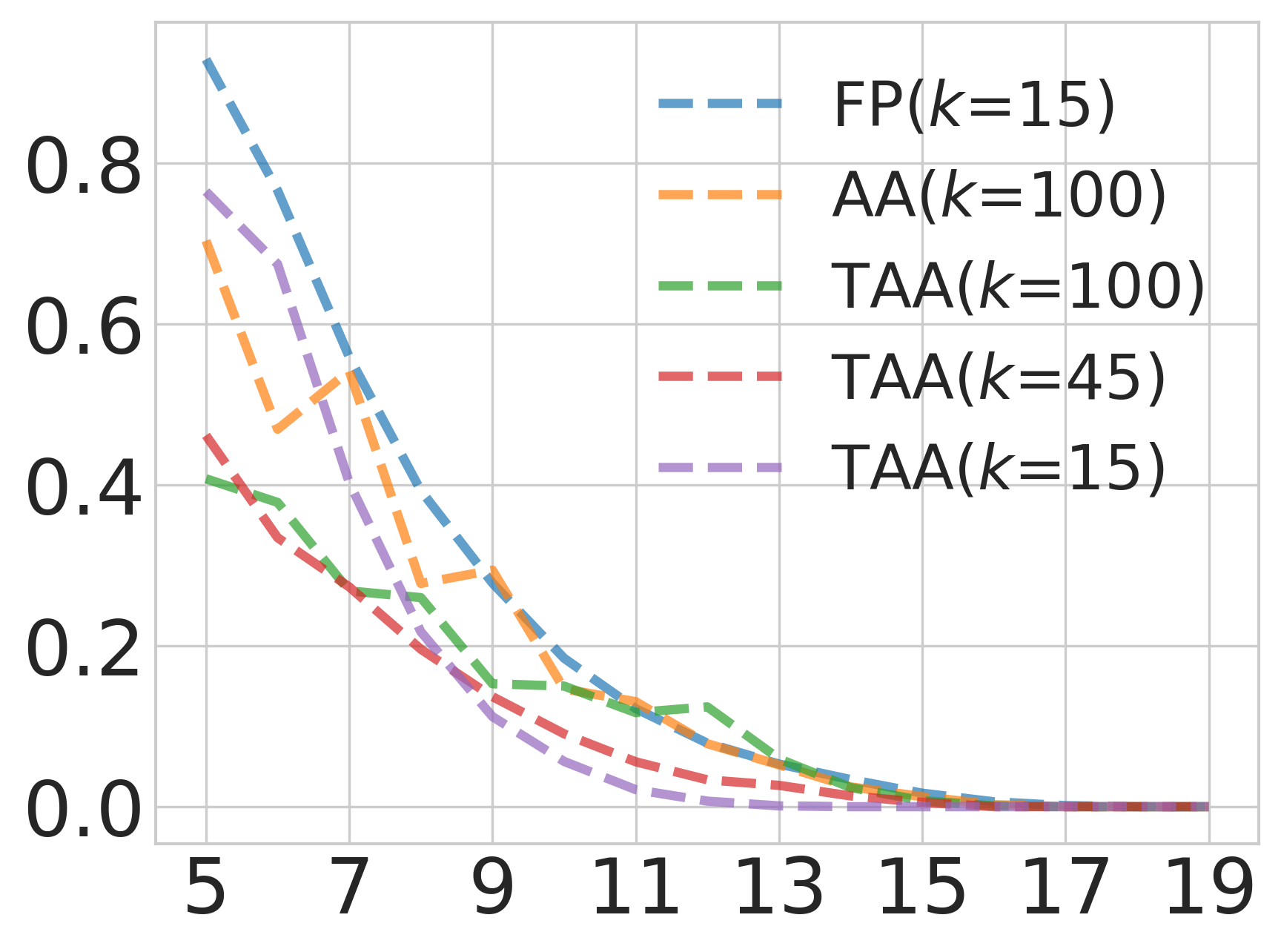}
 	\caption{DDIM 100 steps}
	\label{}
\end{subfigure}
	\begin{subfigure}[b]{0.22\textwidth}
	\centering
	\includegraphics[width=\textwidth]{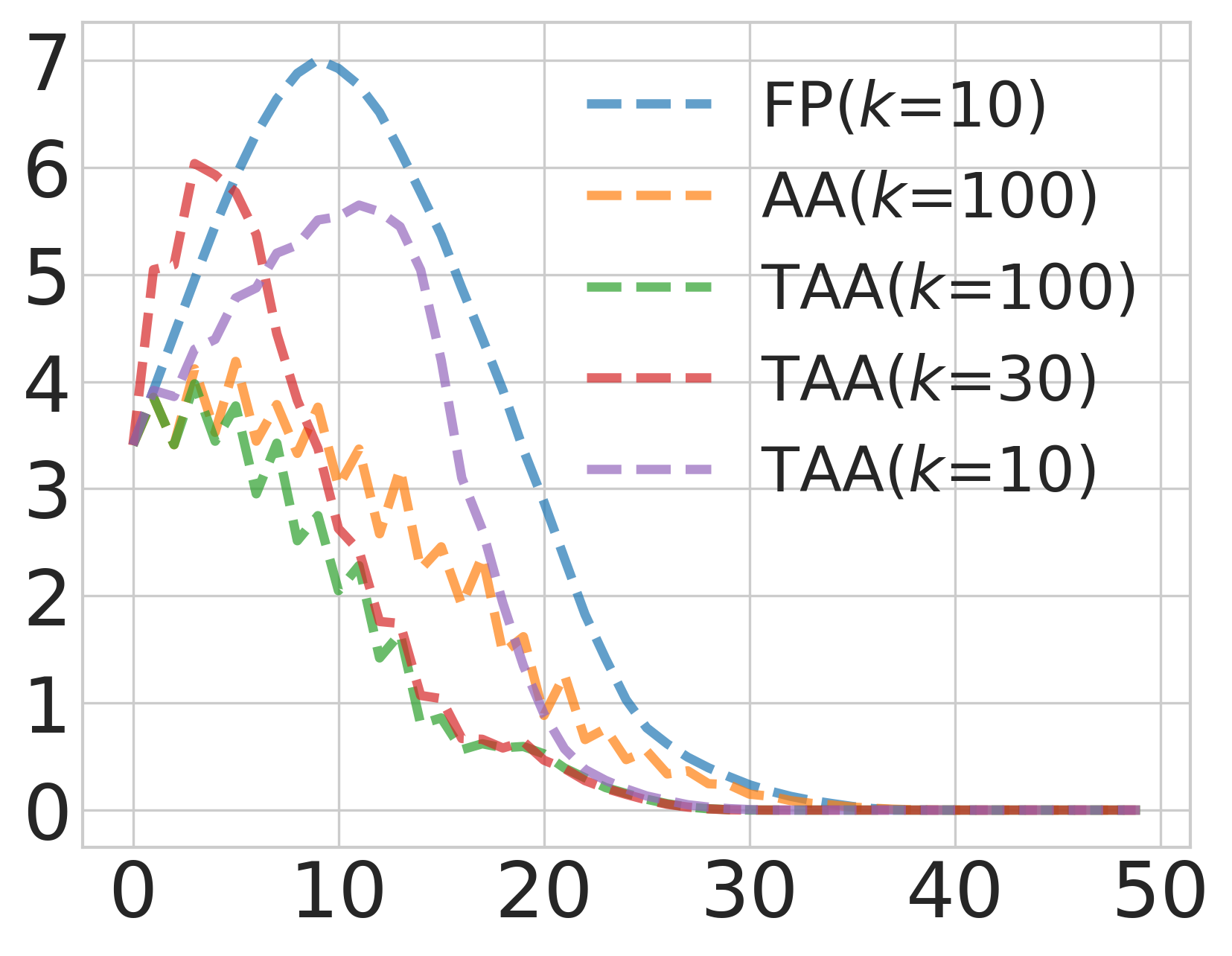}
	\caption{DDPM 100 steps}
	\label{}
\end{subfigure}
\caption{Convergence of FP, AA, TAA under different $k$.}
	\label{fig:ad5_order}
\end{figure}

\subsection{Safeguarding Triangular Anderson Acceleration}
Fixed-point iteration, as described in \eqref{p:fixedpoint}, is known to converge within $T$ steps for triangular systems like \eqref{p:orderk} (see Proposition 1 in \cite{song2021accelerating}). Unfortunately, neither the original AA nor the TAA possesses this worst-case convergence guarantee. To address this, we have identified a sufficient condition for the general update rule in the form of \eqref{p:anderson} to ensure convergence within $T$ steps.

\begin{theorem}
\label{thm:convergence}
Consider a general update rule: In the $i$-th iteration, the update is $x^{i+1}_{0:T-1} = x^i_{0:T-1} - G^i R^i_{0:T-1}$, with $G^i$ being any arbitrary matrix. If for any $j$ where $R_{j+1}^i=...=R_{T}^i=\boldsymbol{0}$ \footnote{Note that $R_{T}^i\equiv\boldsymbol{0}$ since the initial variable $x_T$ is set to $\xi_T$.}, the matrix $G^i$ satisfies $G^i[jd:,jd:]=-I$, then the update rule will converge within $T$ steps.
\end{theorem}

Please see Appendix \ref{app:proof} for the proof and a more detailed explanation of this theorem. In practice, we impose this condition from Theorem \ref{thm:convergence} on the TAA update rule by post-processing. Fortunately, this post-processing step applied to the matrix $T^i$ has virtually no impact on the empirical performance of TAA in our experiments. Additional details can be found in Appendix \ref{app:triangular_anderson}.

\section{Early Stopping and Initialization}
\label{sec:early_init}
This section introduces two practical techniques that can further accelerate the fixed-point iteration.

\subsection{Early Stopping}
In our experiments, we observe that high-quality images are often produced much earlier than the residuals of fixed-point iteration meet the stopping criterion. As an example, while using TAA with the fixed-point iteration for DDIM 100 steps, we find that high-quality images, nearly indistinguishable from those generated sequentially, can appear as early as step $7\sim 11$. In contrast, the stopping criterion is typically met at around step $13\sim 17$. More demonstrations on this point can be found in Section \ref{sec:exp} and Appendix \ref{app:generated_image}. Consequently, it is feasible to halt the parallel sampling process whenever a satisfactory image is obtained.

From a practical standpoint, such early termination is easy to implement, particularly for tasks like interactive image generation where users can judge the quality of the current image and decide when to terminate the iteration process.

\subsection{Initialize from Existing Sampling Trajectory}
To further accelerate parallel sampling, one can initialize the fixed-point iteration using an existing sampling trajectory with a similar input condition. The underlying principle is that if two nonlinear equations are similar, the solution to one can serve as the initial point for solving the other.  For instance, in text-to-image generation with two similar prompts P1 and P2, if we have already obtained a sampling trajectory $x_{0}, ..., x_{T-1}$ for P1, we can use this  to initialize parallel sampling for P2. This is a common scenario as users often adjust prompts to achieve the desired image, leading to a wealth of available trajectories for initialization. Additionally, akin to the method in SDEdit \cite{meng2022sdedit}, one can also fix the later steps of the trajectory (e.g., $x_{T_{\text{init}}}, ..., x_{T-1}$) when initializing sampling for P2, and only update the earlier steps (e.g., $x_{0}, ..., x_{T_{\text{init}}-1}$). This ensures the resulting image to be  close to the one from prompt $p_1$. 

As we will show in Section \ref{sec:init} and Appendix \ref{app:image_variation}, starting from an existing sampling trajectory can greatly reduce the steps needed for parallel sampling to converge. Furthermore, this  method often transforms the image to abide the new prompt in a smooth way if $T_{\text{init}}$ is properly set.

With all the techniques discussed in Sections \ref{sec:formulation}, \ref{sec:anderson}, and here, the complete version of our proposed algorithm, ParaTAA, is summarized in Algorithm \ref{alg:ParaTAA}.

\begin{algorithm}[htbp]
    \small
        \begin{algorithmic}[1]
            \REQUIRE Diffusion model $\epsilon_\theta$, $k$-th order nonlinear equations $\left[F_{0}^{(k)},...,F_{T-1}^{(k)}\right]$, histroy size $m$, tolerance $\tau$, diffusion coefficients $g(t)$, window size $w$, initialization $x_{0:T-1}^0$ and $x_T$, fixed initaizliation steps $T_{\text{init}}$, maximum iteration steps $s_{\max}$.
            
            \STATE $t_1,t_2\leftarrow \max\{0,T_{\text{init}}-w\},T_{\text{init}}-1$

            \FOR{$s=1$ to $s=s_{\max}$}
            \STATE Compute $\epsilon_\theta(x_{t+1}^{s-1},t+1)$, $t=t_1,...,t_2$ in parallel.
            \STATE Compute the residuals $r_{t_1:t_2}$ as \eqref{p:residual}.
            \STATE Update $t_2\leftarrow \max\{t_1\leq t\leq t_2|r_t>\tau g^2(t)d\}$
            \IF{$t_2$ is Null}
            \STATE Break loop
            \ENDIF
            \STATE Update $t_1\leftarrow \max\{0,t_2-w\}$
            \STATE Compute and store $R_{t_1:t_2}^{s-1}$, $\Xc^{s-1}_{t_1:t_2}$, $\Fc^{s-1}_{t_1:t_2}$ as in Sec. \ref{sec:anderson}.
            \STATE Compute $T^{s-1}$ as in Theorem \ref{thm:G_upper} and \ref{thm:convergence}, do $$
            x^{s}_{t_1:t_2} = x^{s-1}_{t_1:t_2} - T^{s-1} R^{s-1}_{t_1:t_2}$$
            \ENDFOR
            \STATE Return $x_{0:T-1}^s$.

    
    
    
        \end{algorithmic}
        \caption{ParaTAA: Parallel Sampling of Diffusion Models with Triangular Anderson Acceleration}
        \label{alg:ParaTAA}
        \end{algorithm}

\section{Experiments}
\label{sec:exp}

\subsection{Accelerating Image Diffusion Sampling}
\label{sec:main_exp}
In this section, we present the effectiveness of our approach in accelerating the sampling process for two prevalent diffusion models: DiT \cite{peebles2023scalable}, a class-conditioned diffusion model trained on the Imagenet dataset at a resolution of 256x256, and text-conditioned Stable Diffusion v1.5 (SD) \cite{rombach2022high} with a resolution of 512x512.

\textbf{Scenarios.} We consider accelerating four typical sequential sampling algorithms: Euler-type ODE sampling algorithm DDIM \cite{song2020denoising} with 25, 50, and 100 steps, respectively, and SDE sampling algorithm DDPM \cite{ho2020denoising}\footnote{Following \cite{song2020denoising}, we treat DDIM with $\eta=1$ as a DDPM sampler.} with 100 steps.

\textbf{Algorithms.} We compare our proposed algorithm ParaTAA with two baselines: (1) The fixed-point (FP) iteration \eqref{p:fixedpoint} with the order of equations $k$ equal to the window-size $w$, equivalent to the method proposed in \cite{shih2023parallel}, and (2) The fixed-point iteration \eqref{p:fixedpoint} with the optimal order of equations $k$ determined by grid search, referred as FP+. ParaTAA has two hyperparameters: the history size $m$ and the order $k$, both of which are chosen via grid search. For hyperparameter analysis in all four tested scenarios, please refer to Appendix \ref{app:hyperparameter}. For all algorithms, we use the same stopping threshold $\varepsilon_t=\tau^2g^2(t)d$ with $\tau=10^{-3}$, and initialize all variables with standard Gaussian Distribution.

\begin{table*}
    \small
        \centering 
    \begin{tabular}{c|c|c|c|c|c|c|c|c|c|c|c|c|c|c|}
        \toprule
        \multirow{2}{*}{Method}& \multicolumn{4}{c|}{DiT  DDIM-25}& \multicolumn{4}{c|}{DiT DDIM-50} & \multicolumn{3}{c|}{SD  DDIM-25} & \multicolumn{3}{c}{SD DDIM-50} \\
        
        & Steps
        & Time & FID $\downarrow$ &IS$\uparrow$& Steps& Time & FID $\downarrow$&IS$\uparrow$& Steps & Time & CS$\uparrow$ & Steps& Time & CS$\uparrow$ \\
    
        Sequential& 25  & 0.41s & 20.5&442.6 & 50 &  0.84s & 20.3 & 443.4 & 25  & 0.73s&23.9&50&1.44s  & 24.0\\
    
        FP& 17.8 & 0.42s & 19.8& 441.2 & 21.6 &  0.69s & 20.2 & 442.0 & 14.1  & 0.98s &23.8& 15.7& 1.36s & 24.0  \\
     
        FP+ & 13 & 0.32s &18.7&436.5& 17 &  0.58s  & 18.7 & 436.8& \textbf{7} & \textbf{0.62s} &23.8&\textbf{7}& \textbf{0.93s} &23.9 \\
     
        ParaTAA& \textbf{9} & \textbf{0.25s} & 18.8&441.0& \textbf{9} &  \textbf{0.34s} & 19.1 & 441.9 & \textbf{7} & 0.63s &23.8&\textbf{7}& \textbf{0.93s} &23.9 \\
        \midrule
       & \multicolumn{4}{c|}{DiT DDIM-100}& \multicolumn{4}{c|}{DiT DDPM-100} & \multicolumn{3}{c|}{SD DDIM-100} & \multicolumn{3}{c}{SD DDPM-100} \\
        
    
        Sequential& 100&1.65s& 20.6& 446.9 &100  & 1.69s & 22.7 & 464.8 & 100& 2.95s & 24.2 &100 & 2.98s & 24.8  \\
    
        FP& 23.0& 0.98s & 19.7 & 444.2 & 42.3 & 1.90s & 21.4 & 459.6 & 15.8& 2.16s & 24.2&28.9& 3.23s & 24.8  \\
     
        FP+ & 19 & 0.81s & 19.8 & 443.7 & 31 & 1.29s & 17.0 & 432.3& \textbf{7}& 1.56s &24.2 &21&2.45s  &24.5 \\
     
        ParaTAA& \textbf{11} & \textbf{0.56s} & 20.0 &448.3& \textbf{21}& \textbf{0.95s} & 22.1 & 457.8 & \textbf{7}& \textbf{1.53s} & 24.2&\textbf{15} &\textbf{1.97s} & 24.8  \\
        \bottomrule
    \end{tabular}
    \caption{Performance comparison of various parallel sampling methods across different scenarios. It should be noted that for FP+ and ParaTAA, the early-stopping step is selected based on insights from Figure \ref{fig:main-result}, whereas for FP, early-stopping is not employed, and the step value indicates the average number of inference steps required to satisfy the stopping criterion. }
    \label{tab:time} 
\end{table*}

\textbf{Setting.} We run these experiments using 8 A800 GPUs, each with 80GB of memory. We set the window size $w$ to match the number of sampling steps for all scenarios, except in the case of the DDPM 100 steps for SD, where we select $w=40$ to aim for an acceptable wall-clock time speedup. In all scenarios, we employ classifier-free guidance \cite{ho2022classifier} with a guidance scale of 5.

\textbf{Evaluation.} For DiT models, we assess the quality of sampled images using the FID score \cite{heusel2017gans} and the Inception Score (IS) \cite{salimans2016improved} with 5000 generated samples. For SD models, we generate random text prompts combining a color and an animal, such as "green duck," and evaluate the quality of the sampled images by computing the CLIP Score (CS) \cite{radford2021learning} with 1000 samples.

\begin{figure}[htpb]
	\centering
    \begin{subfigure}[b]{0.48\textwidth}
        \centering
        \includegraphics[width=\textwidth]{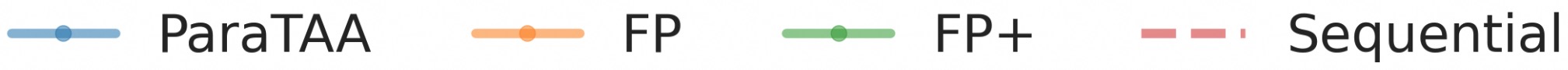}
    \end{subfigure}
	\begin{subfigure}[b]{0.155\textwidth}
	\centering
	\includegraphics[width=\textwidth]{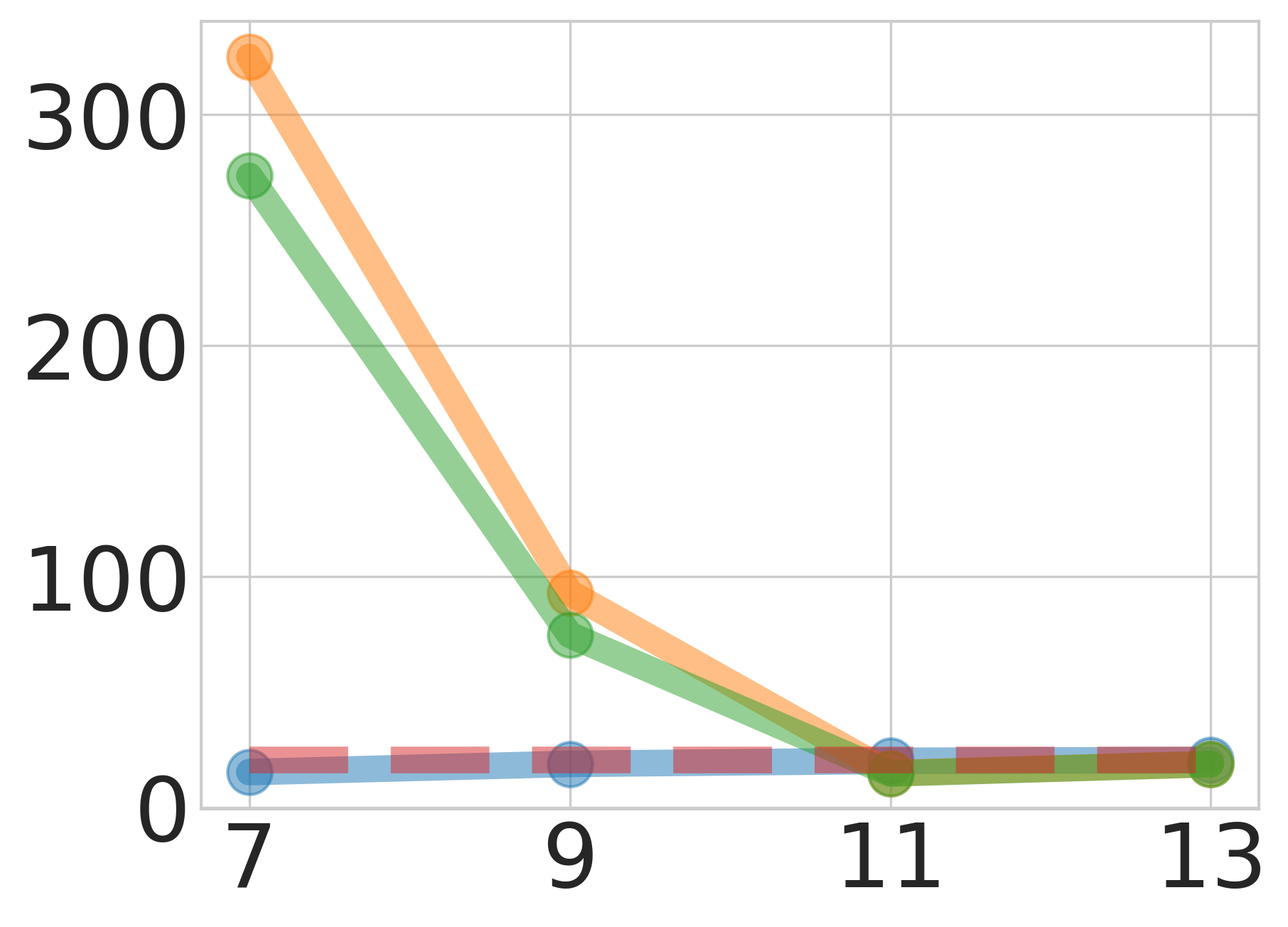}
\end{subfigure}
	\begin{subfigure}[b]{0.155\textwidth}
	\centering
	\includegraphics[width=\textwidth]{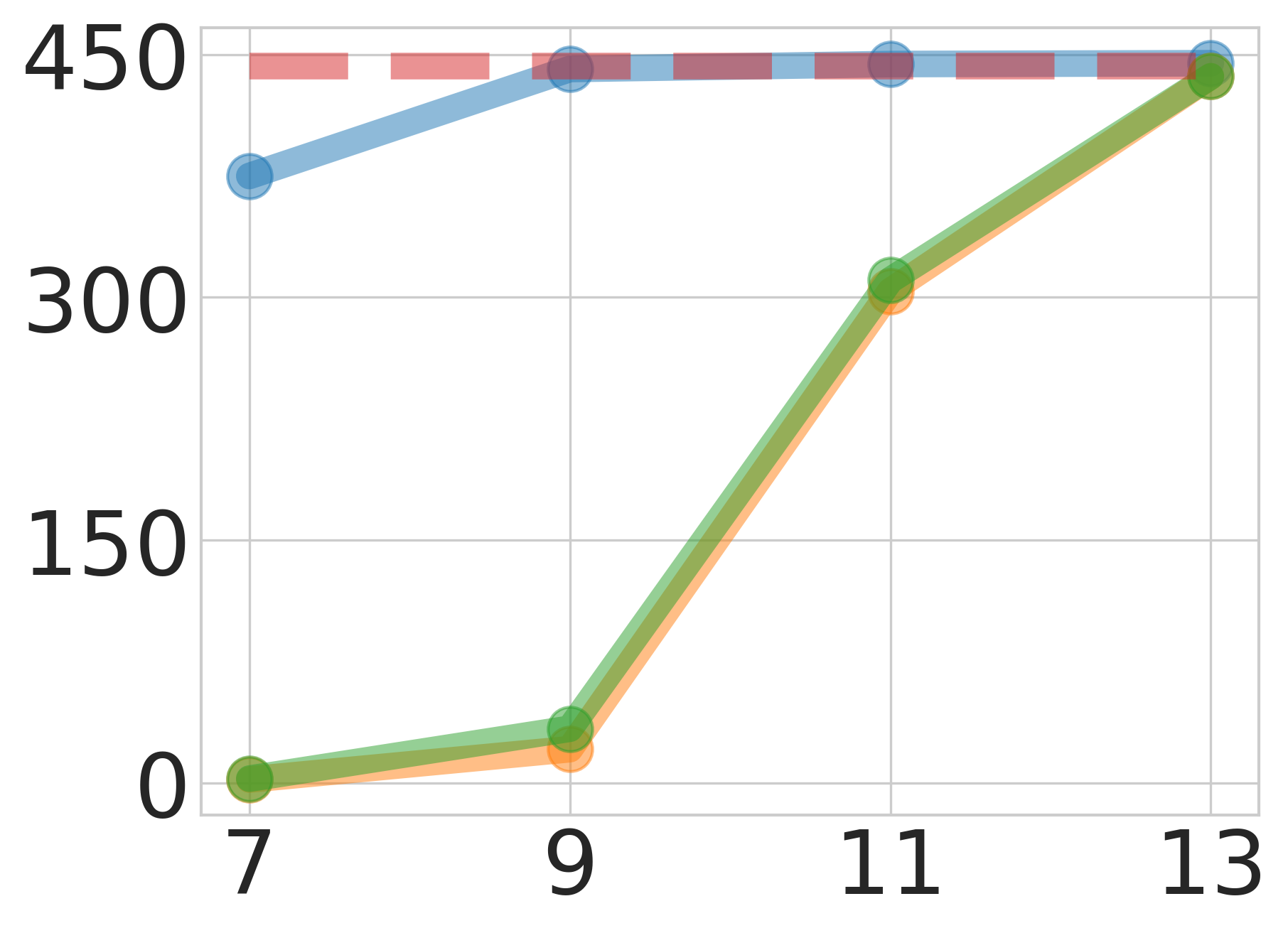}
\end{subfigure}
\begin{subfigure}[b]{0.155\textwidth}
	\centering
	\includegraphics[width=\textwidth]{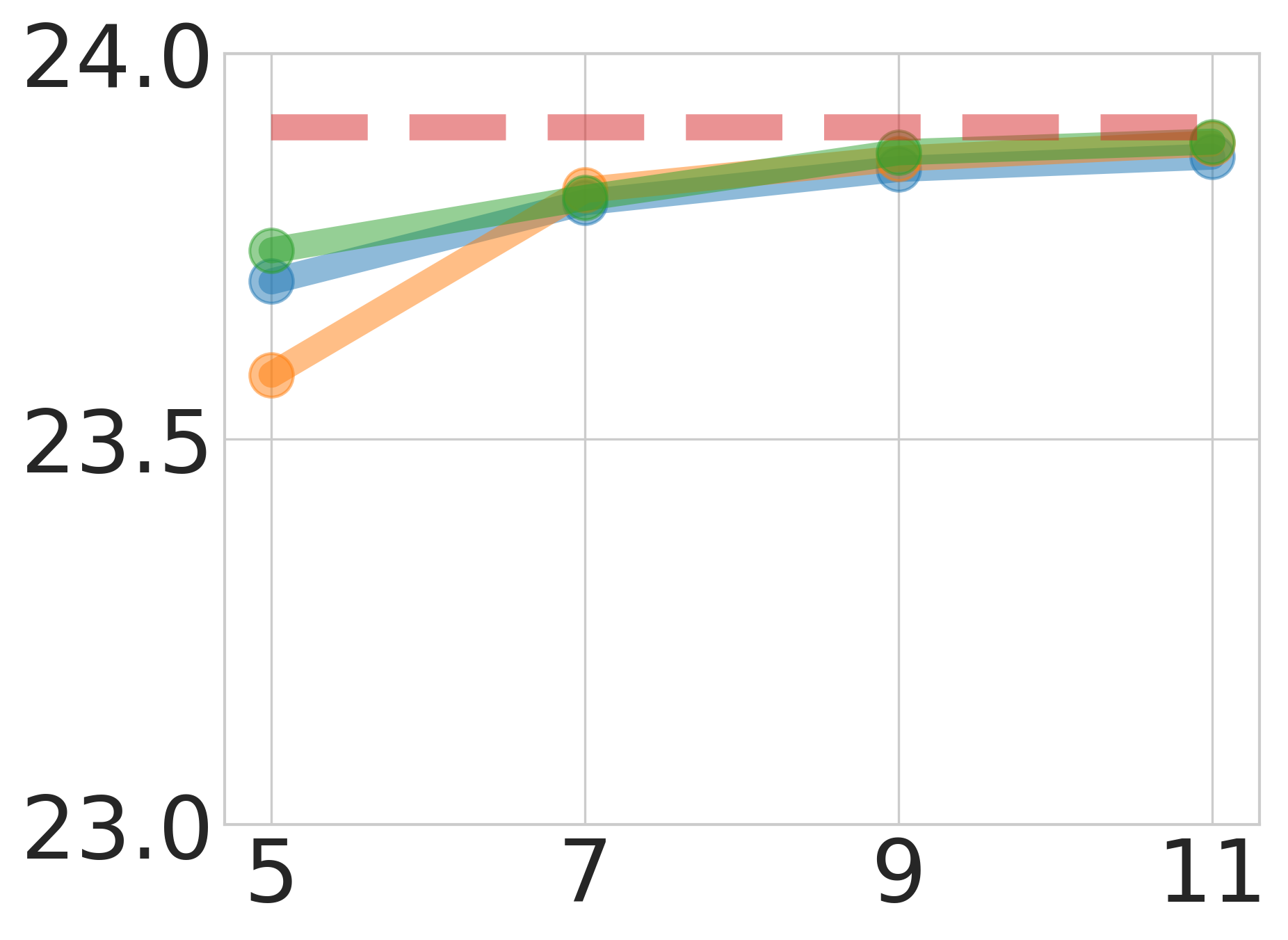}
\end{subfigure}
	\begin{subfigure}[b]{0.155\textwidth}
	\centering
	\includegraphics[width=\textwidth]{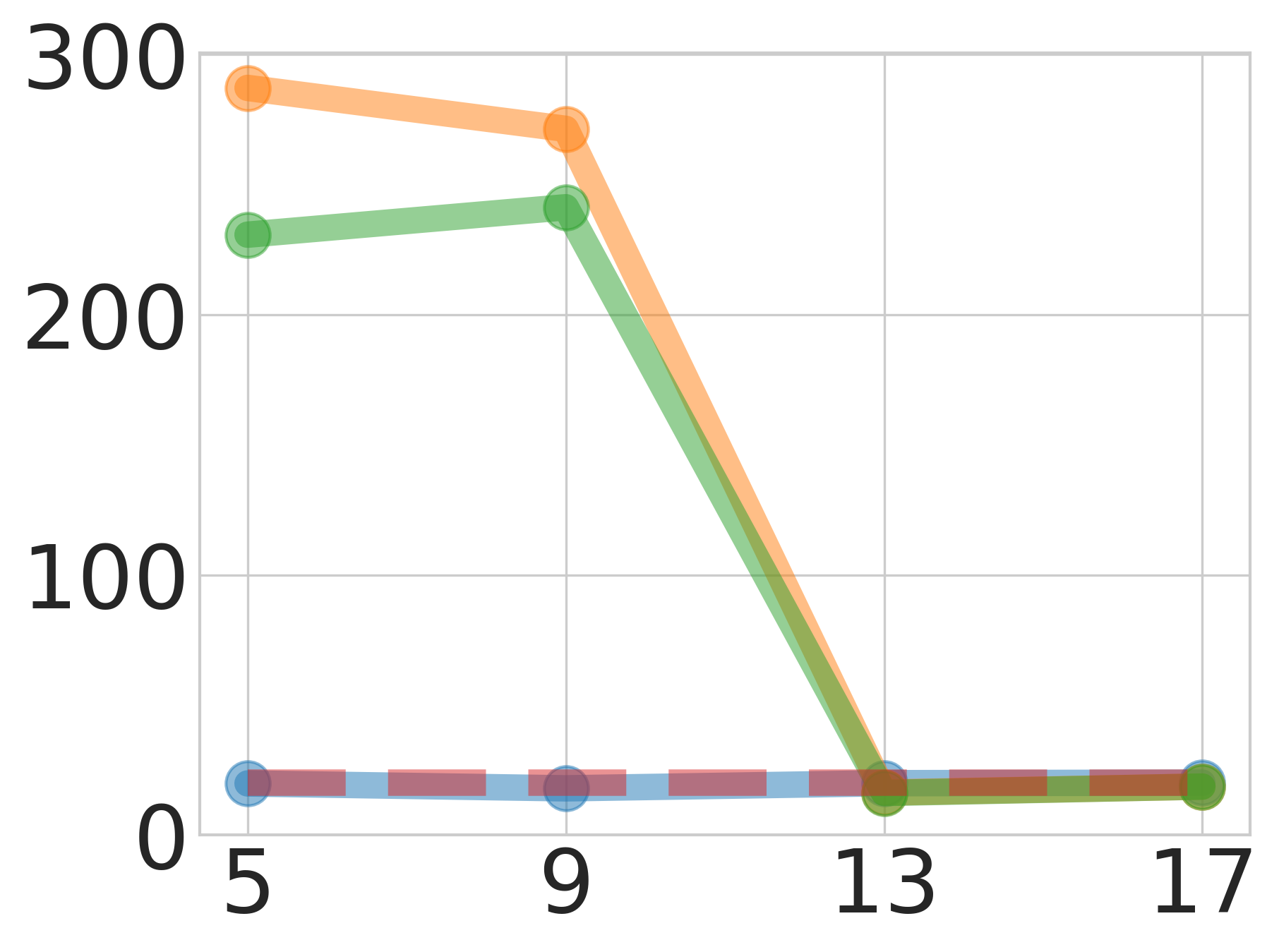}
\end{subfigure}
\begin{subfigure}[b]{0.155\textwidth}
	\centering
	\includegraphics[width=\textwidth]{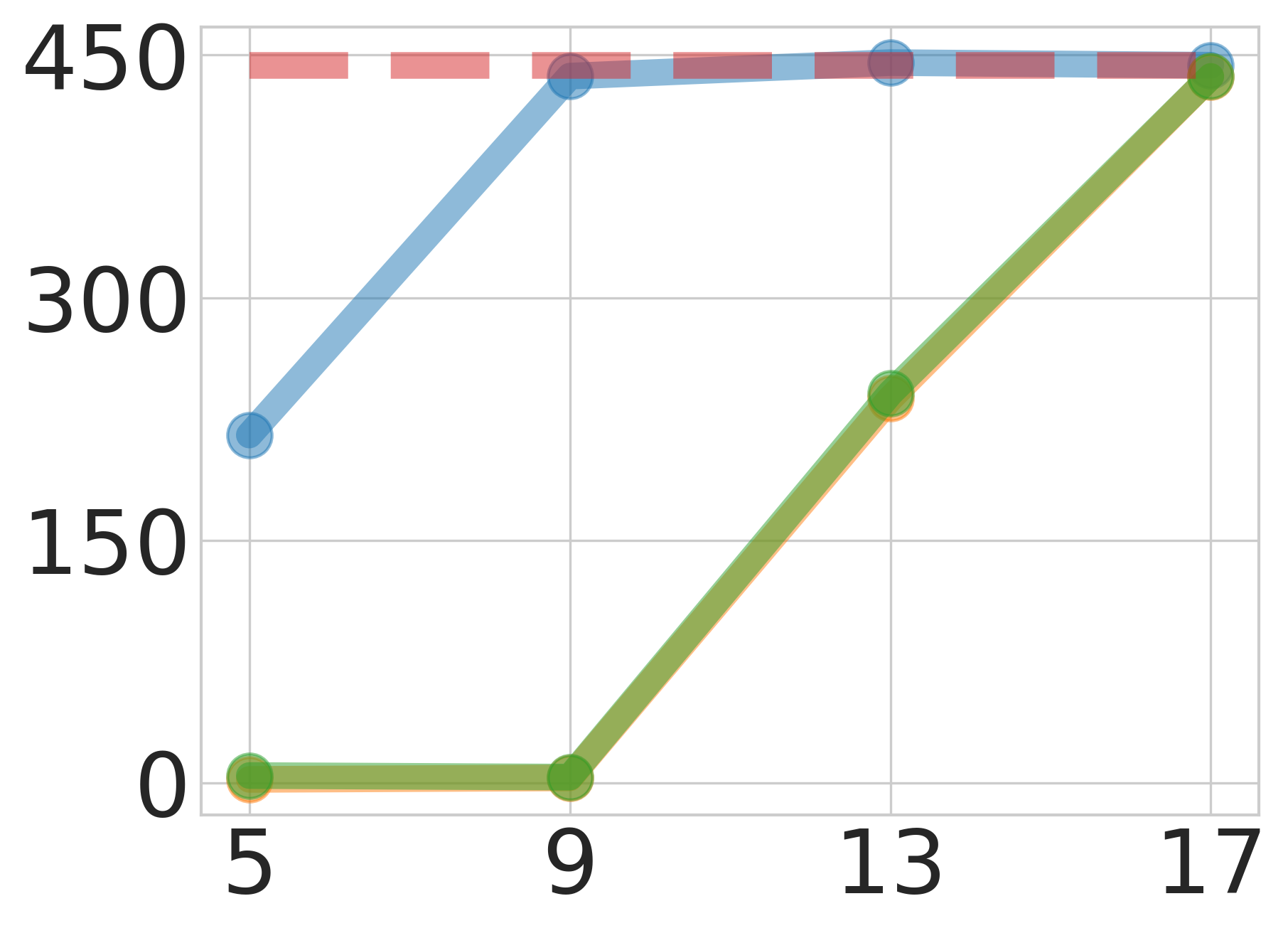}
\end{subfigure}
	\begin{subfigure}[b]{0.155\textwidth}
	\centering
	\includegraphics[width=\textwidth]{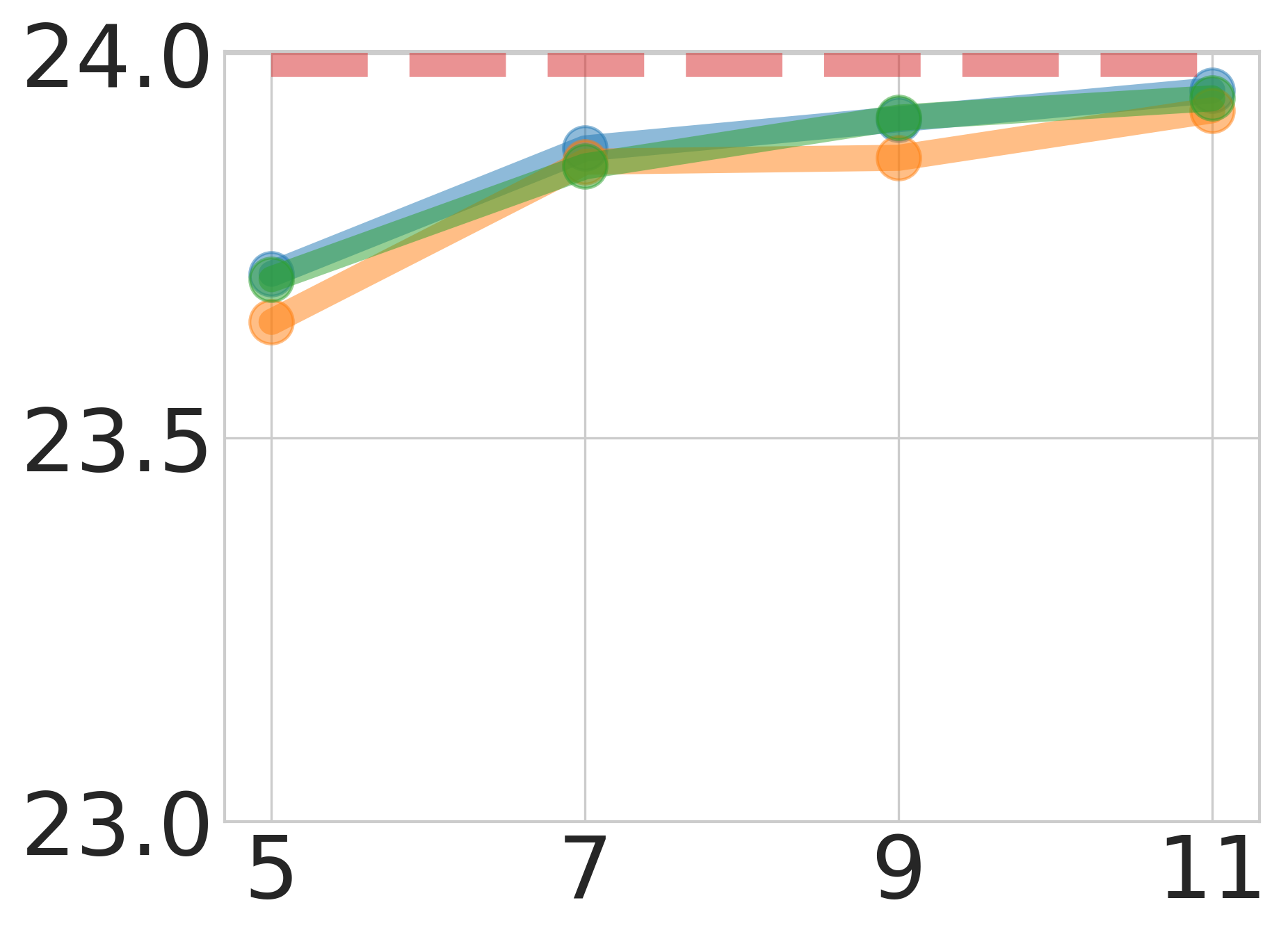}
\end{subfigure}
\begin{subfigure}[b]{0.155\textwidth}
	\centering
	\includegraphics[width=\textwidth]{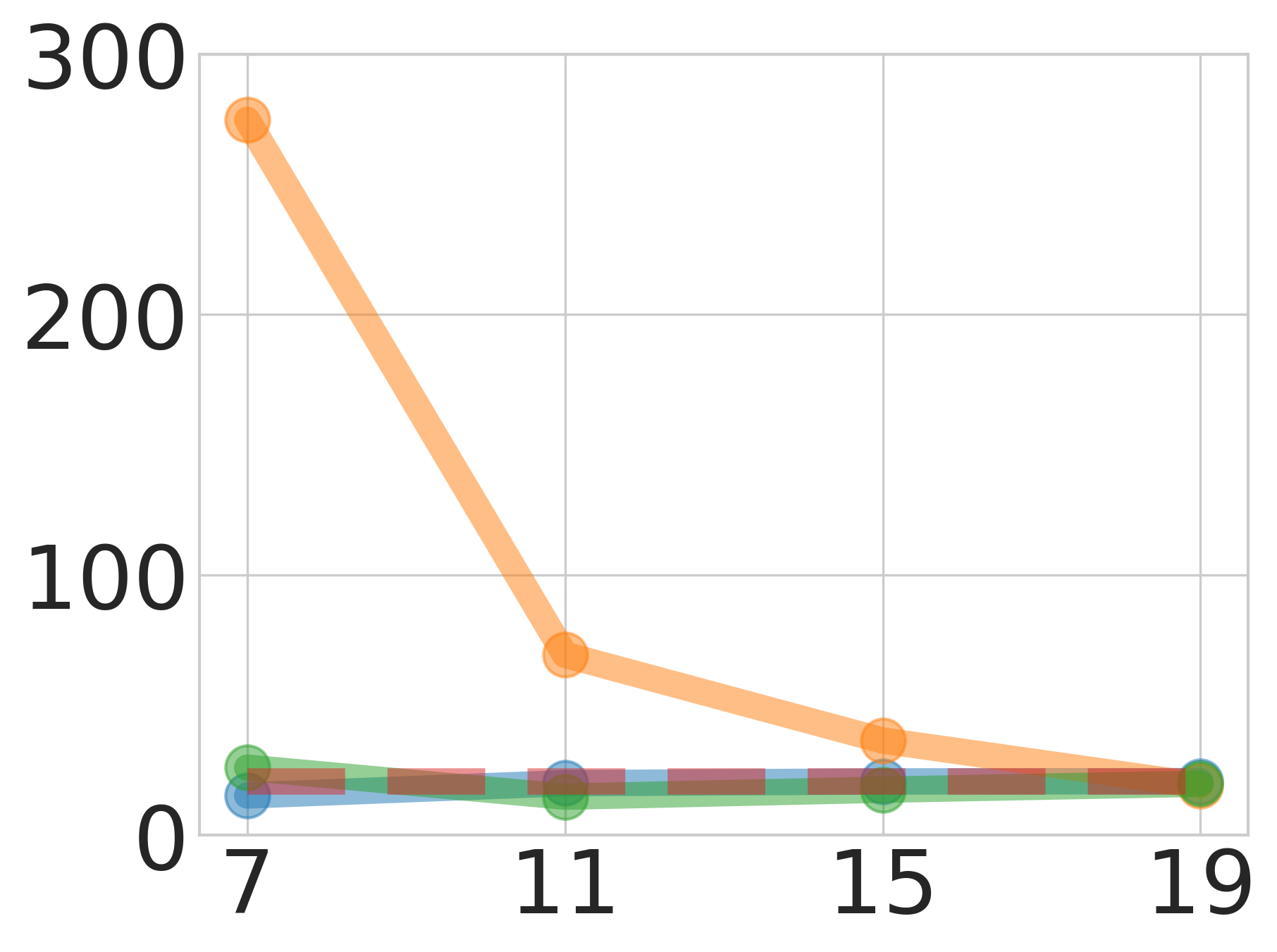}
\end{subfigure}
	\begin{subfigure}[b]{0.155\textwidth}
	\centering
	\includegraphics[width=\textwidth]{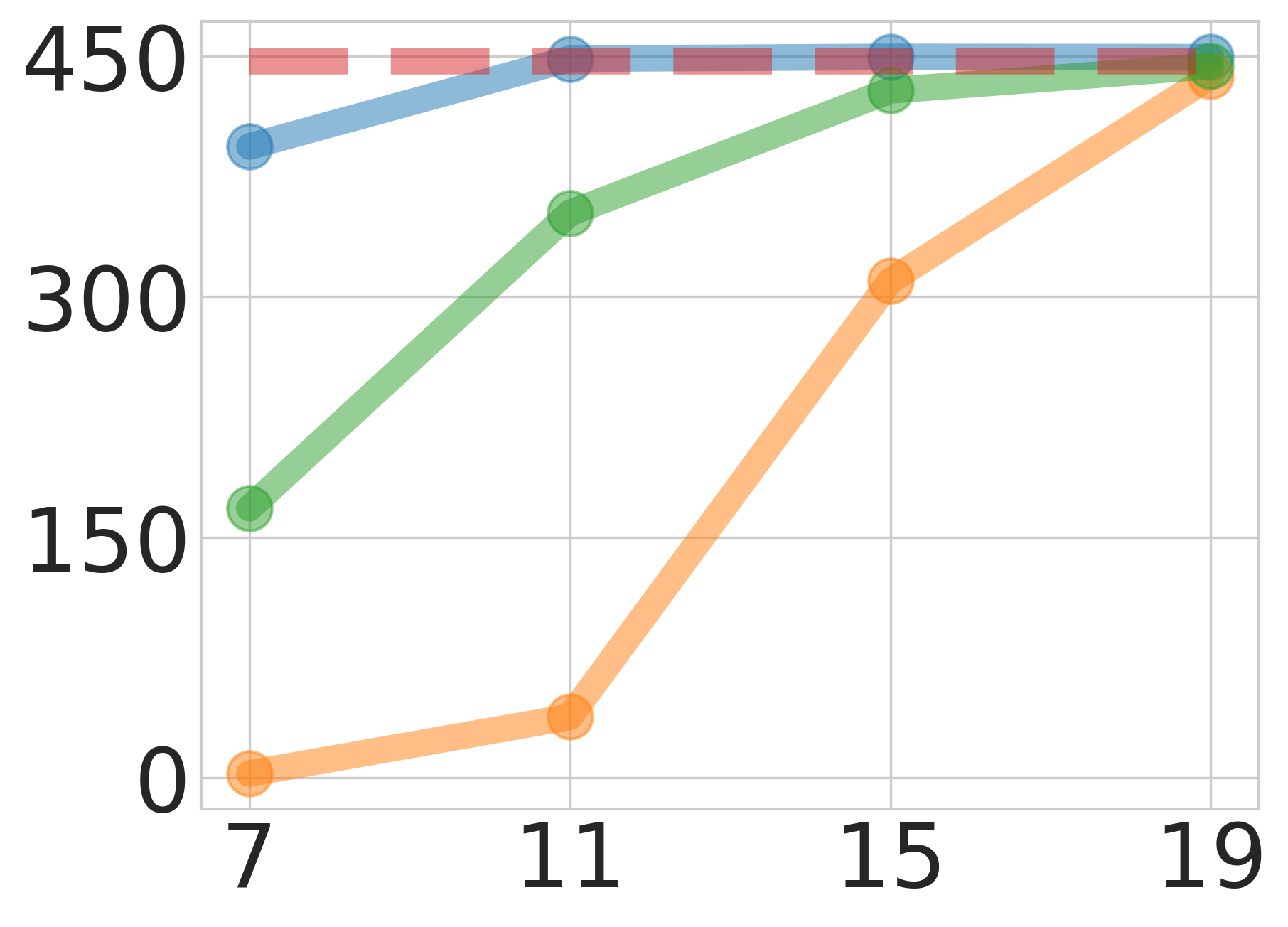}
\end{subfigure}
\begin{subfigure}[b]{0.155\textwidth}
	\centering
	\includegraphics[width=\textwidth]{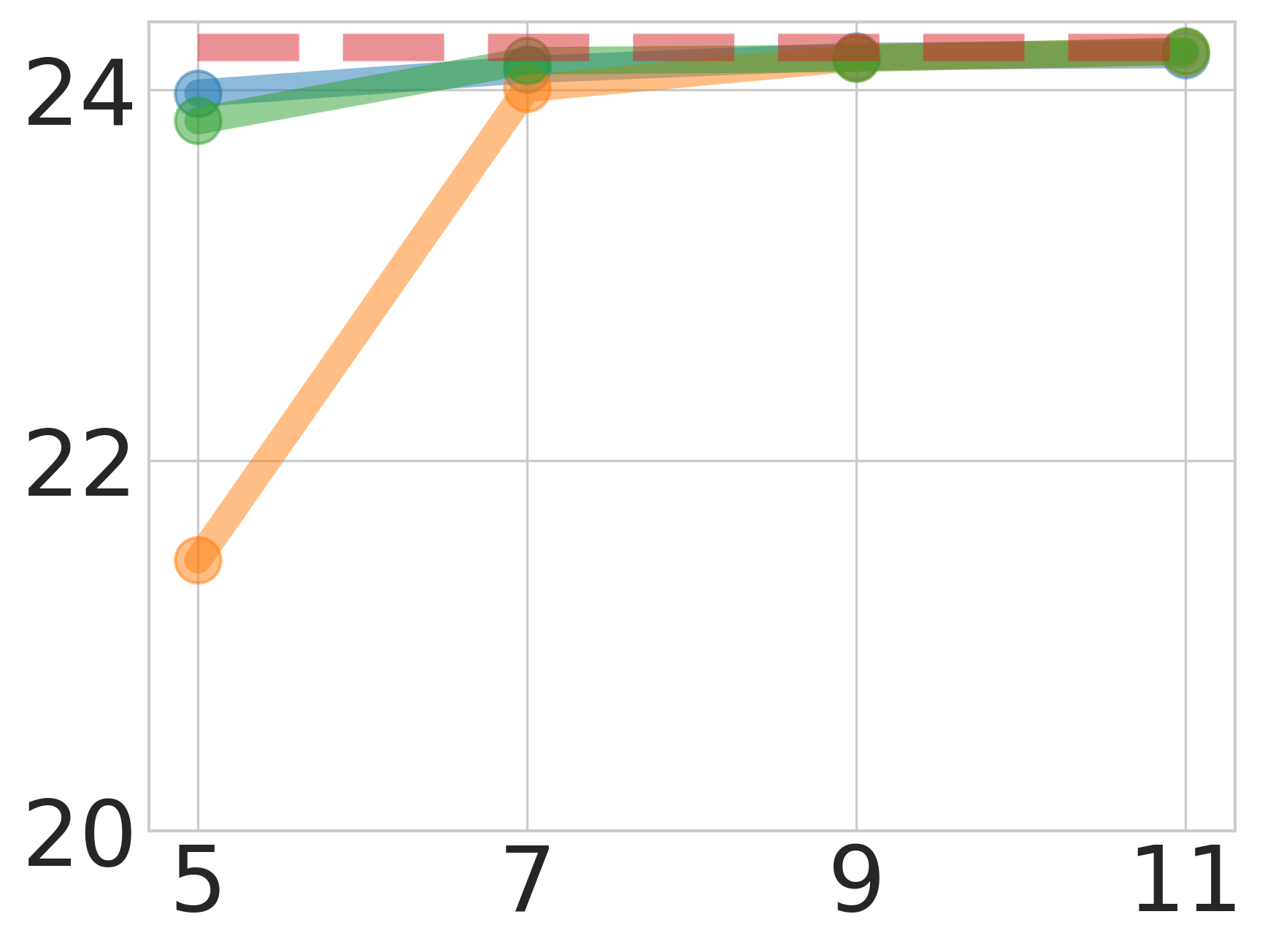}
\end{subfigure}
	\begin{subfigure}[b]{0.155\textwidth}
	\centering
	\includegraphics[width=\textwidth]{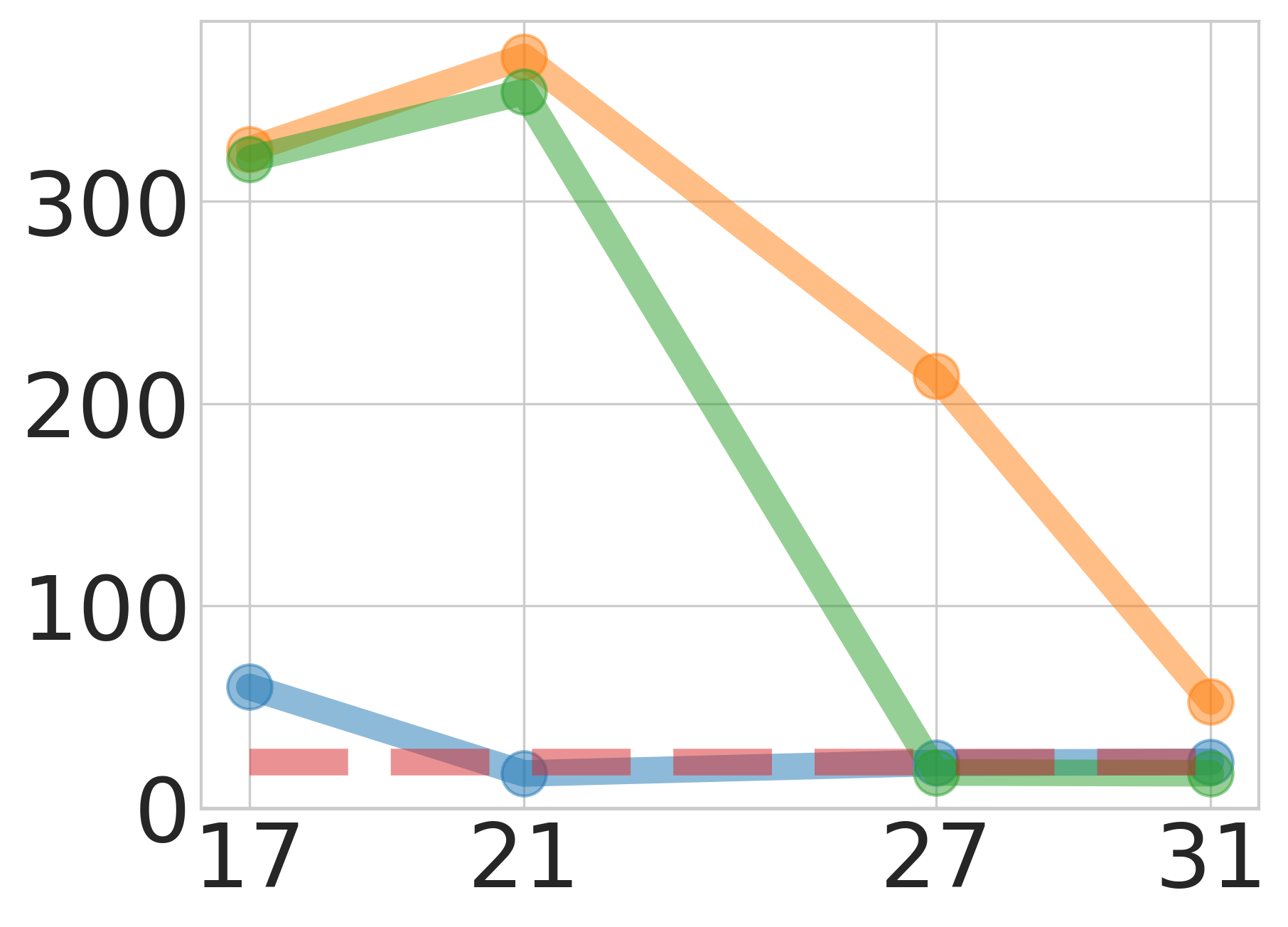}
\end{subfigure}
\begin{subfigure}[b]{0.155\textwidth}
	\centering
	\includegraphics[width=\textwidth]{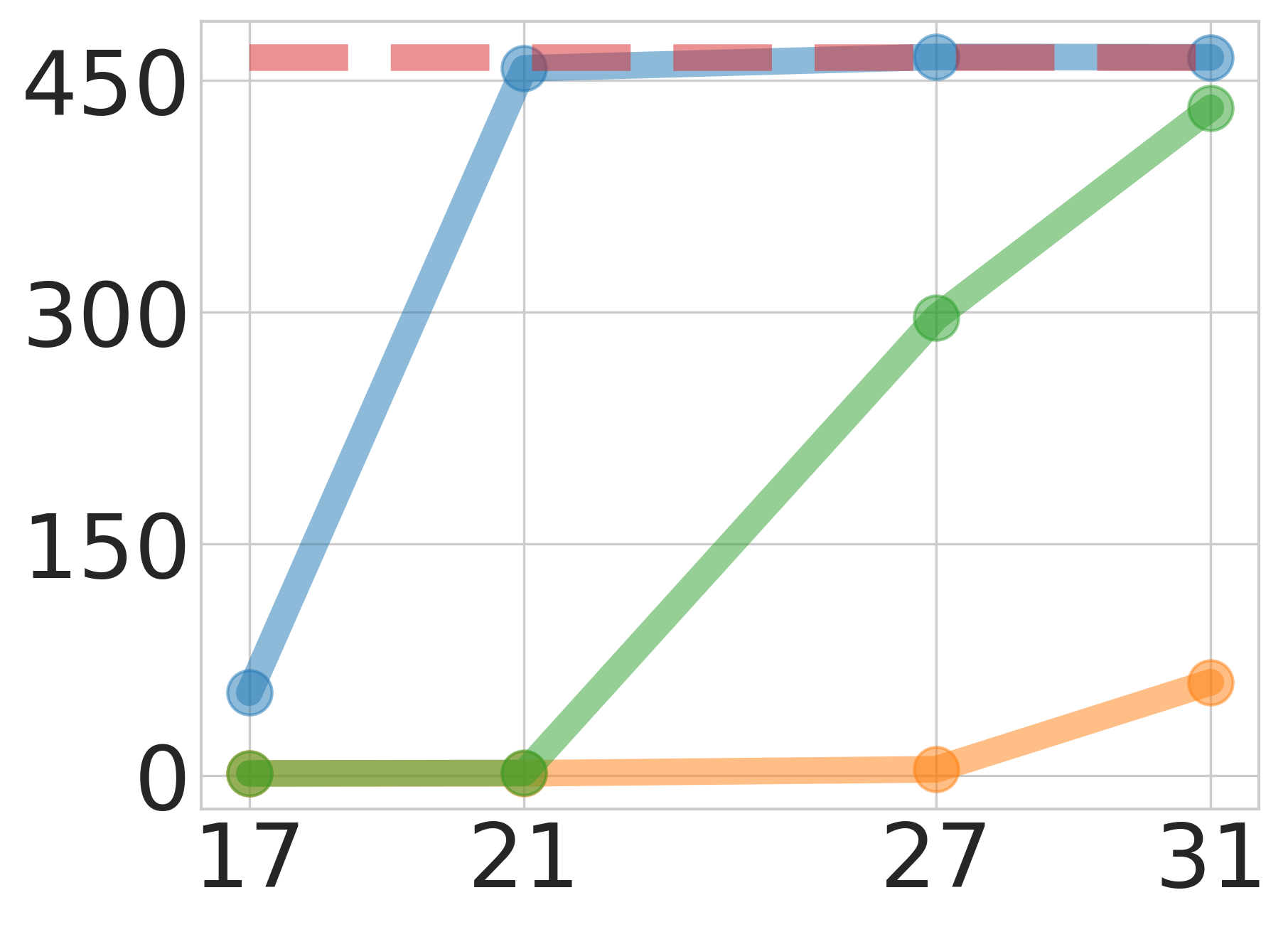}
\end{subfigure}
	\begin{subfigure}[b]{0.155\textwidth}
	\centering
	\includegraphics[width=\textwidth]{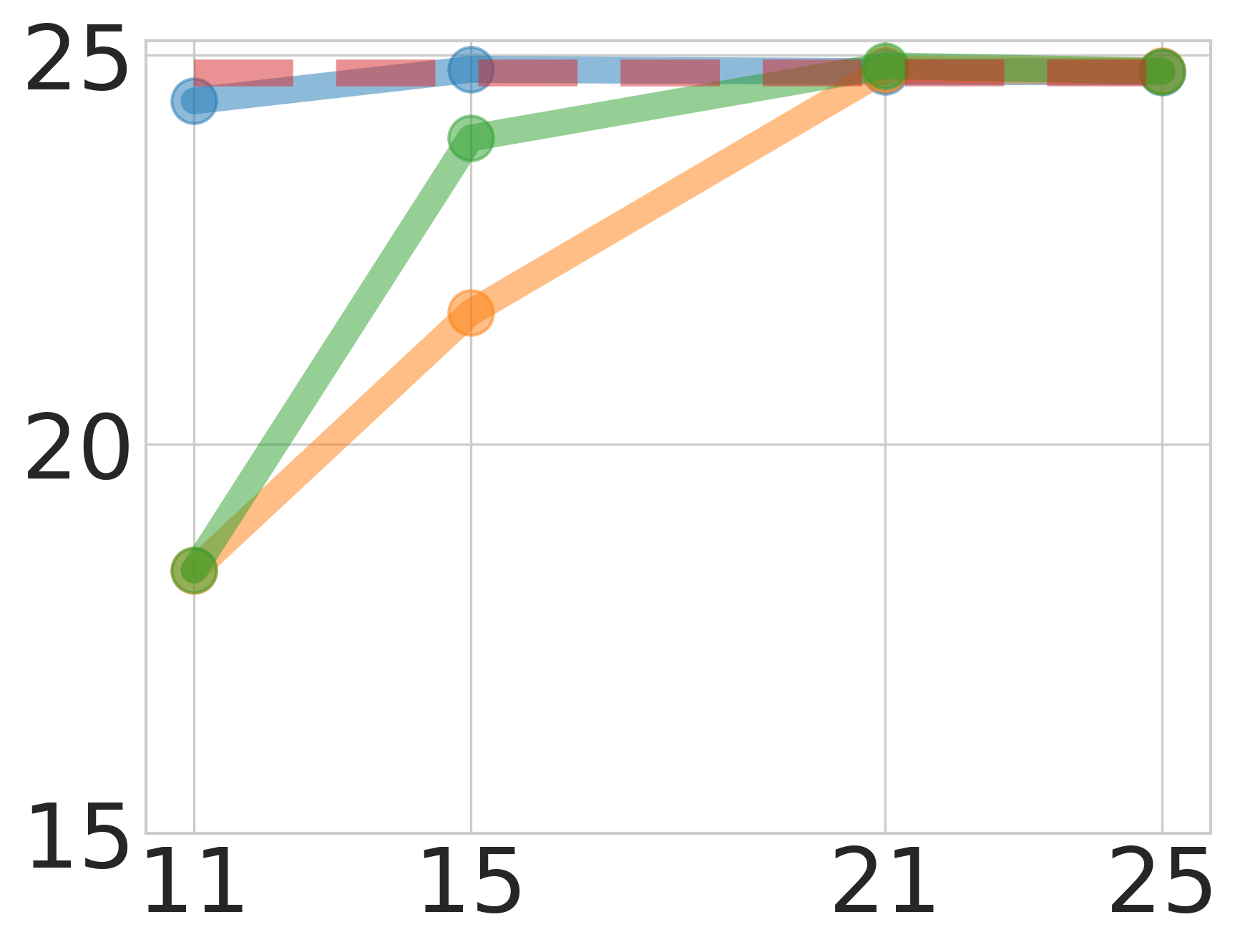}
\end{subfigure}
\caption{Comparison of parallel sampling methods and sequential sampling across various scenarios. The x-axis for all plots represents the maximum number of steps, $s_{\max}$. The first two columns from the left show the FID and IS scores for the DiT model, respectively, while the third column depicts the CS for the SD model. The rows, from top to bottom, correspond to the scenarios with DDIM 25 steps, DDIM 50 steps, DDIM 100 steps, and DDPM 100 steps, respectively. For visual examples of generated images related to these results, please refer to Appendix \ref{app:generated_image}.}
	\label{fig:main-result}
\end{figure}

\begin{figure}[htpb] 
	\centering
	\begin{subfigure}[b]{0.19\textwidth}
	\centering
	\includegraphics[width=\textwidth]{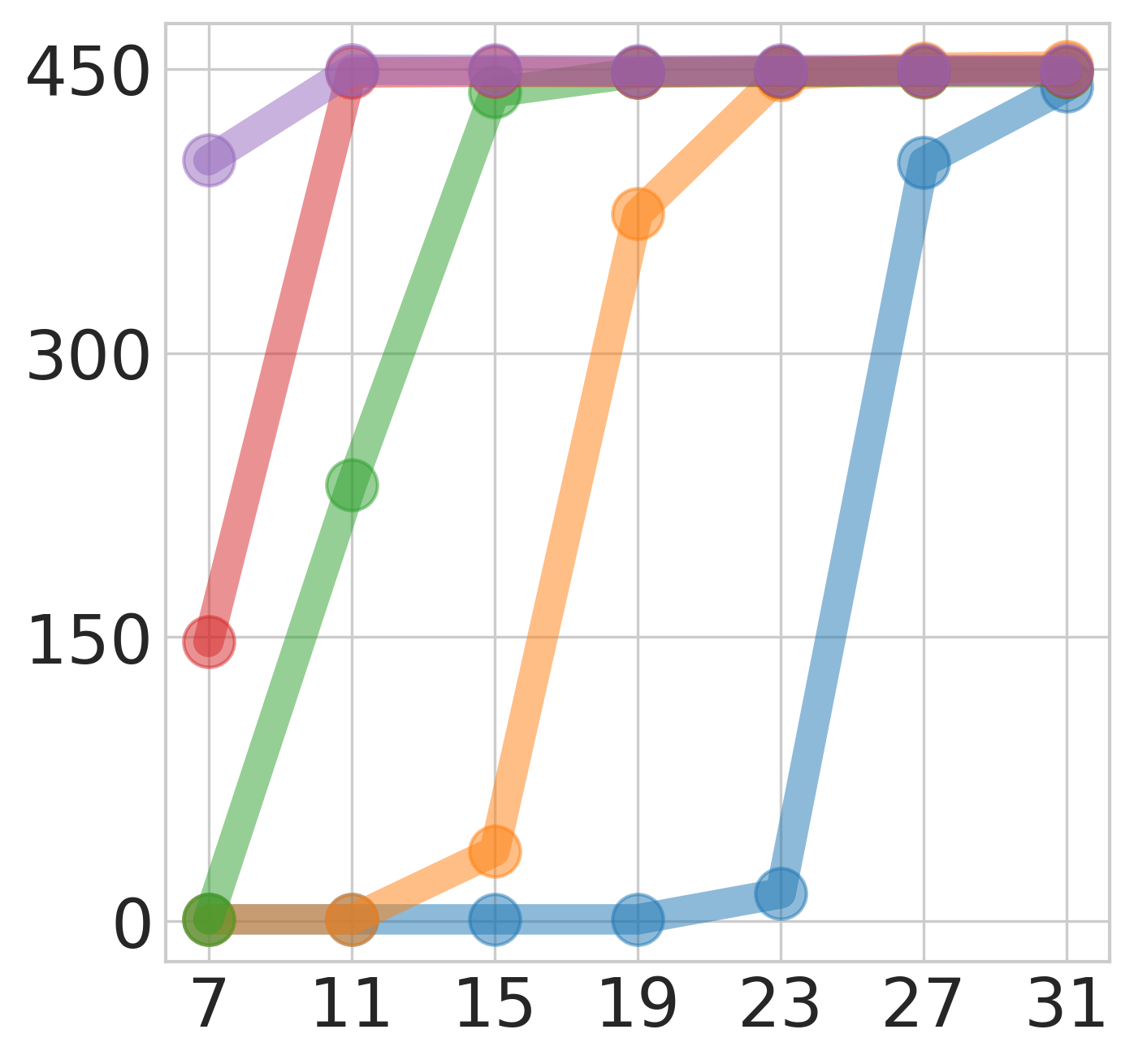}
 	\caption{DiT - IS}
	\label{}
\end{subfigure}
	\begin{subfigure}[b]{0.185\textwidth}
	\centering
	\includegraphics[width=\textwidth]{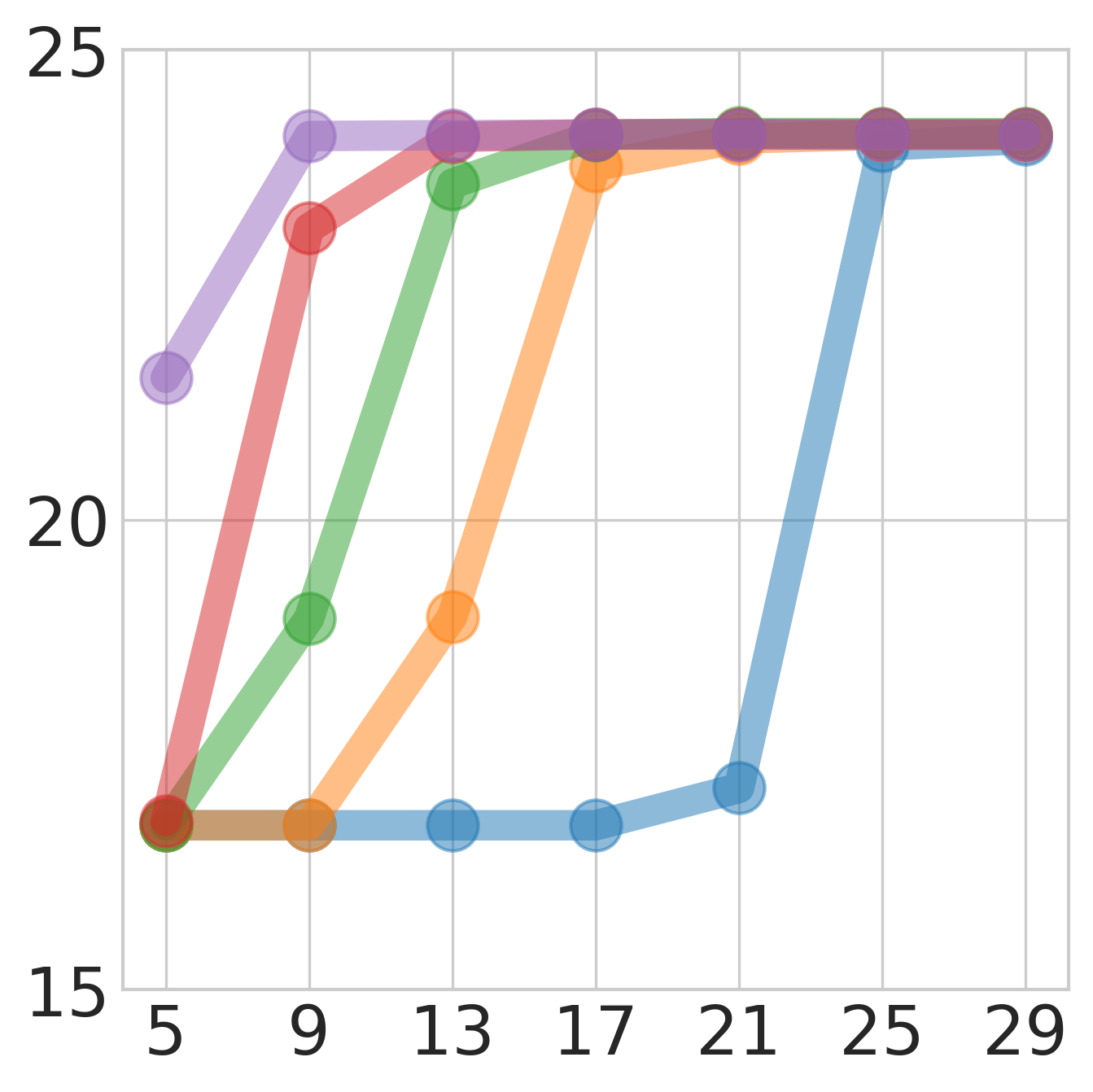}
	\caption{SD - CS}
	\label{}
\end{subfigure}
\begin{subfigure}[b]{0.09\textwidth}
	\centering
	\includegraphics[width=\textwidth]{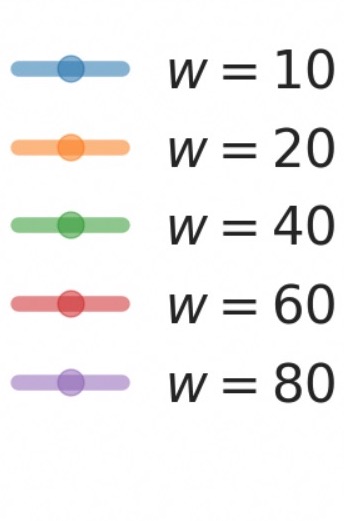}
    
\end{subfigure}
\caption{Convergence of ParaTAA under different window sizes. The x-axis and y-axis are the same as Figure \ref{fig:main-result}}
	\label{fig:window_size}
\end{figure}

\textbf{Results. }Our primary findings are detailed in Figure \ref{fig:main-result} and Table \ref{tab:time}, offering several insightful observations. Firstly, Figure \ref{fig:main-result} corroborates that early-stopping is a valid approach\footnote{For the information on the number of inference steps to reach the stopping criterion, we refer readers to Appendix \ref{app:hyperparameter}.}. Across all algorithms, the quality metrics of the generated images match those of sequentially sampled images in significantly fewer steps. By comparing FP and FP+, we can clearly see the importance of choosing a properly order $k$ for nonlinear equations \eqref{p:orderk}. Furthermore, our proposed ParaTAA outperforms both fixed-point algorithms substantially, underscoring the effectiveness of our Triangular Anderson Acceleration technique. In Table \ref{tab:time}, we encapsulate the key outcomes from Figure \ref{fig:main-result}, including data on wall-clock time and inference steps. Notably, "Steps" in Table \ref{tab:time} refers to the number of parallelizable inference steps for the neural network $\epsilon_\theta$. It is evident that all parallel sampling algorithms greatly reduce inference steps, particularly for larger $T$ scenarios, with ParaTAA consistently spending the fewest steps in every case and cutting down the steps required by sequential sampling by \textbf{4$\sim$14x}! Additionally, it is apparent that, generally, DDPM needs more steps to converge compared to DDIM. In terms of wall-clock time speedup, ParaTAA can achieve a \textbf{1.5$\sim$2.9x} improvement.

\begin{remark}
    The wall-clock time reported in Table \ref{tab:time}  can be further enhanced with optimized implementation, computing devices, and inter-GPU communication environments. Theoretically, the achievable speedup is determined by the ratio of inference steps required by sequential versus parallel sampling, ranging from 4 to 14 times as discussed earlier.
\end{remark}
 
\begin{remark}
    Our own implementation of the fixed-point iteration achieves results comparable to those in \cite{shih2023parallel}. However, we opted not to adjust the stopping criterion, as we observed it impacts the uniqueness of the generated image. In our SD experiments, we used 16-bit precision instead of the 32-bit used in \cite{shih2023parallel}, which made our measured wall-clock time significantly faster.
\end{remark}

\begin{remark}
    A key advantage of parallel sampling over other acceleration methods is its ability to produce images that are (almost) identical to those from sequential sampling. Theorem \ref{thm:unique} provides a guarantee for this assertion. For real examples on this point, please refer to Appendix \ref{app:generated_image}.
\end{remark}

\begin{figure}[htpb] 
	\centering
	\begin{subfigure}[b]{0.32\columnwidth}
	\centering
	\includegraphics[width=\columnwidth]{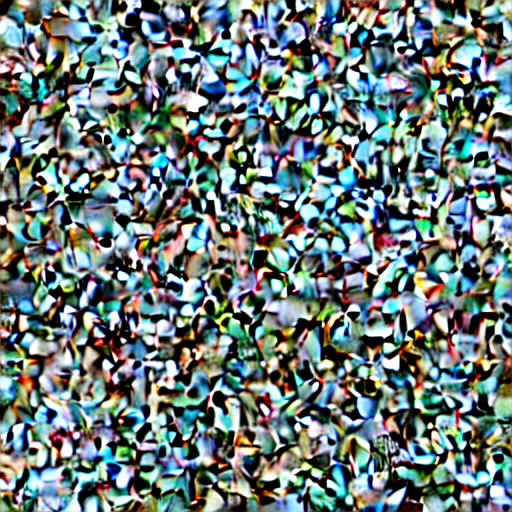}
\end{subfigure}
\begin{subfigure}[b]{0.32\columnwidth}
	\centering
	\includegraphics[width=\columnwidth]{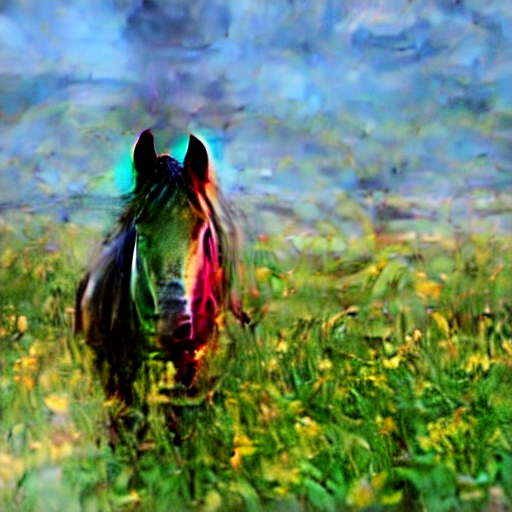}
\end{subfigure}
\begin{subfigure}[b]{0.32\columnwidth}
	\centering
	\includegraphics[width=\columnwidth]{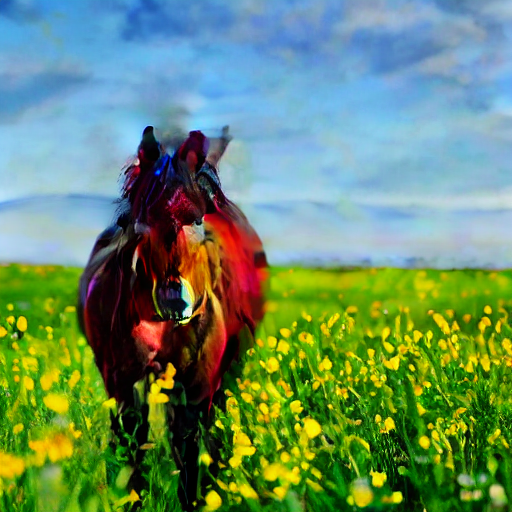}
\end{subfigure}

\begin{subfigure}[b]{0.32\columnwidth}
	\centering
	\includegraphics[width=\columnwidth]{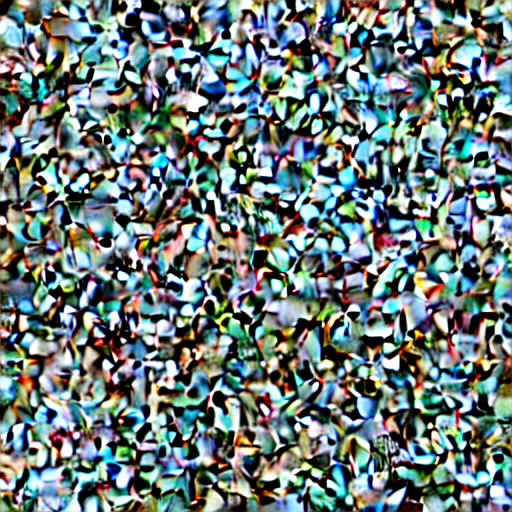}
\end{subfigure}
\begin{subfigure}[b]{0.32\columnwidth}
	\centering
	\includegraphics[width=\columnwidth]{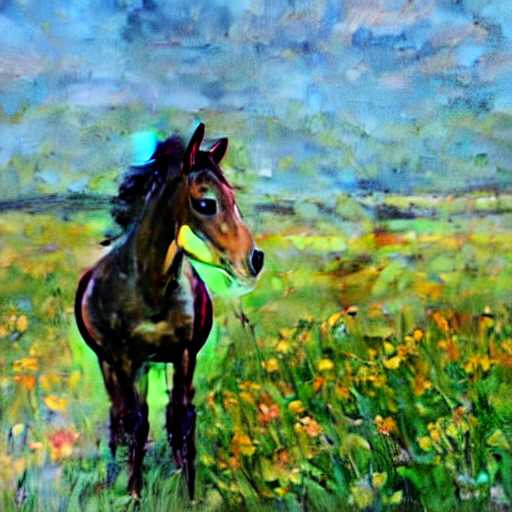}
\end{subfigure}
\begin{subfigure}[b]{0.32\columnwidth}
	\centering
	\includegraphics[width=\columnwidth]{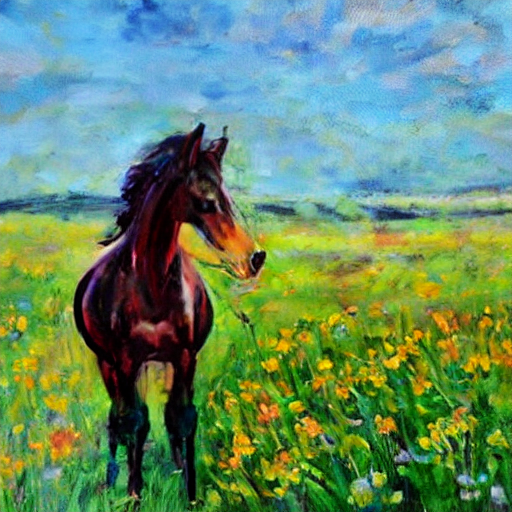}
\end{subfigure}

\begin{subfigure}[b]{0.32\columnwidth}
	\centering
	\includegraphics[width=\columnwidth]{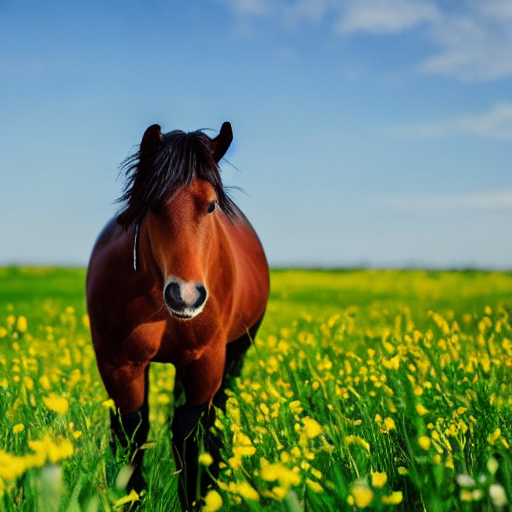}
\end{subfigure}
\begin{subfigure}[b]{0.32\columnwidth}
	\centering
	\includegraphics[width=\columnwidth]{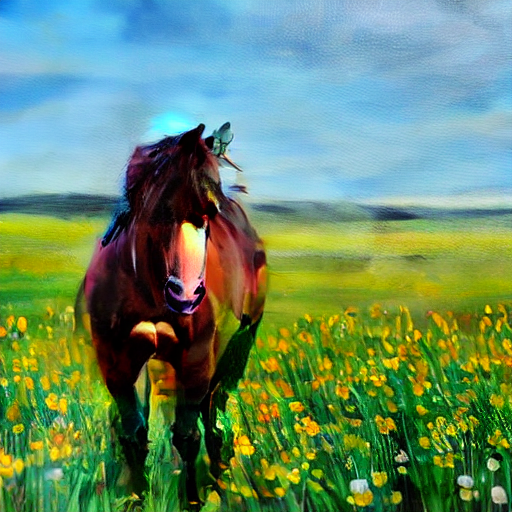}
\end{subfigure}
\begin{subfigure}[b]{0.32\columnwidth}
	\centering
	\includegraphics[width=\columnwidth]{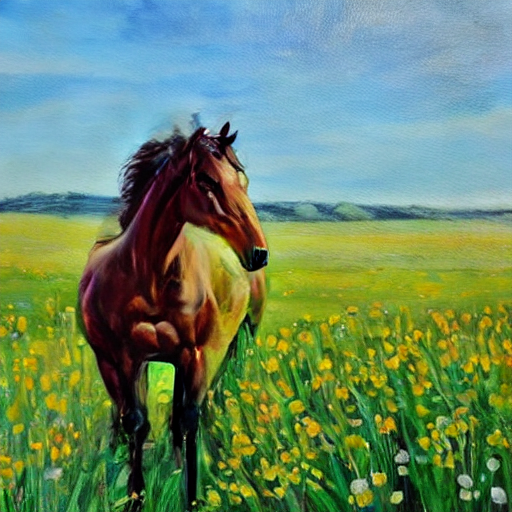}
\end{subfigure}

\begin{subfigure}[b]{0.32\columnwidth}
	\centering
	\includegraphics[width=\columnwidth]{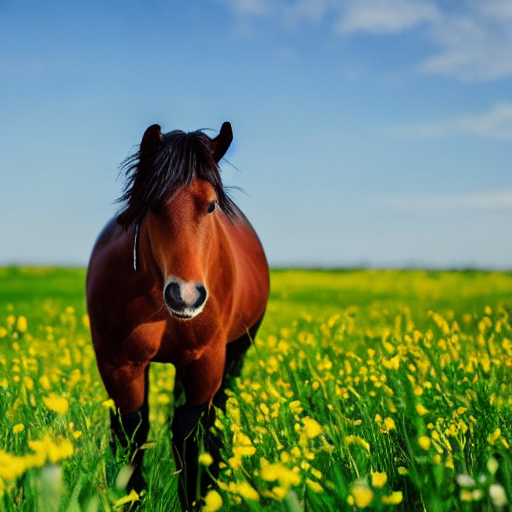}
    \caption{Initialization}
\end{subfigure}
\begin{subfigure}[b]{0.32\columnwidth}
	\centering
	\includegraphics[width=\columnwidth]{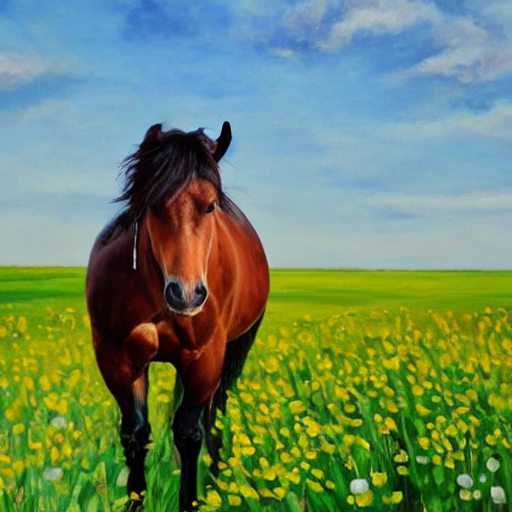}
    \caption{After 3 steps}
\end{subfigure}
\begin{subfigure}[b]{0.32\columnwidth}
	\centering
	\includegraphics[width=\columnwidth]{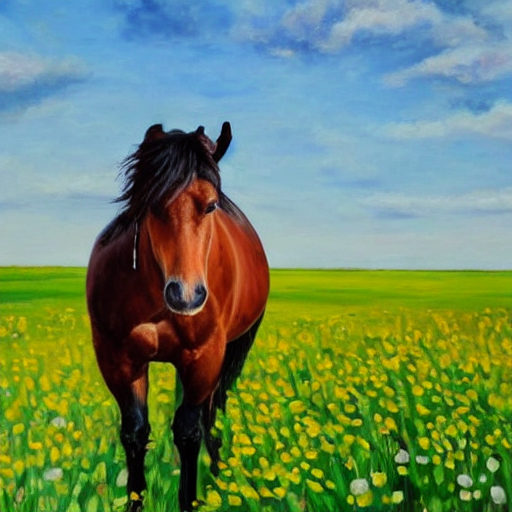}
    \caption{After 5 steps}
\end{subfigure}

\caption{Iterations of ParaTAA with different initializations.  The rows from top to bottom shows: 1. Sampling with P1 with random initialization; 2. Sampling with P2 with random initialization. 3. Sampling with P2 with trajectory of P1 as initialization and $T_{\text{init}}=50$. 4. Same as 3 except that  $T_{\text{init}}=35$. For optimal viewing, please zoom in on the figure.}
	\label{fig:early_stop}
\end{figure}

\subsection{Effect of Window Size}
\label{sec:window_size}

In this section, we examine how the window size $w$ affects the trade-off between convergence and computation on the DDIM 100 steps scenario for both DiT and SD models, testing ParaTAA with varying window sizes.

\textbf{Results.}
As depicted in Figure \ref{fig:window_size}, the relationship between the increase in window size and the reduction in inference steps is not proportional. For instance, with SD, at $w=10$, ParaTAA needs 25 steps to achieve the desired CS level, which is $4$x fewer than that of sequential sampling. However, when we double the computation by setting  $w=20$, the inference steps reduce only marginally to 21. This implies that users should select a window size that balances convergence speed and computational effort to optimize wall-clock time speedup.

\subsection{Initialization from Existing Trajectory}
\label{sec:init}

This section explores the impact of initializing parallel sampling using a pre-existing trajectory through a case study. We conduct an experiment utilizing the SD model with DDIM 50 steps and two similar prompts, P1: "A 4k detailed photo of a horse in a field of flowers", P2: "An oil painting of a horse in a field of flowers". Our objective is to investigate the difference in image generation for P2 when starting from a random initialization versus using the trajectory of P1. Additionally, we assess how the initial step count, $T_{\text{init}}$, influences the convergence of sampling for P2 with this initialization.

\textbf{Results.} As shown in Figure \ref{fig:early_stop}, using random initialization for prompts P1 and P2, ParaTAA does not yield high-quality images within the first 5 steps. In contrast, initializing the sampling of P2 with the trajectory from P1 results in a considerably better image by the 5th step. By setting $T_{\text{init}}$ to 35, ParaTAA manages to produce a good image by the 3rd step, with a smooth transition from the initial image. Hence, we conclude that starting parallel sampling with an existing trajectory can significantly decrease the number of inference steps needed for the sampling process to converge. For an extended and  quantitative evaluation of these findings, please refer to Appendix \ref{app:quant_init}.



\section{Conclusion}

In this study, we frame parallel sampling for diffusion models as solving a system of triangular nonlinear equations. We introduce a novel parallel sampling algorithm, ParaTAA, which can substantially decrease the inference steps required by sequential sampling while maintaining image quality. Moreover, the triangular Anderson acceleration technique developed in this work could be a subject of independent interest, and we expect that the optimization research community will be interested in further exploring its theoretical aspects in the near future.

While this work primarily demonstrates the acceleration for image diffusion models, we anticipate that our proposed method could have broader applications on tasks that involve an autoregressive process, and one notable example is autoregressive video generative models in \cite{ho2022imagen,esser2023structure,gupta2023photorealistic}. 

Currently, for large models like SD, ParaTAA requires the use of multiple GPUs to achieve considerable speedup in wall-clock time. Nonetheless, as advancements in GPU technology and parallel computing infrastructures evolve, we anticipate that the cost will be significantly lower and ParaTAA will become increasingly important for accelerating the sampling of large-scale diffusion models.

\section*{Acknowledgements}
This work was supported by Alibaba Group through Alibaba Research Intern Program. The work was also supported by Shenzhen Science and Technology Program under Grant No. ZDSYS20230626091302006 and RCJC20210609104448114, and by Guangdong Provincial Key Laboratory of Big Data Computing.

\section*{Impact Statement} This work focuses on accelerating the sampling process of existing diffusion generative models. As far as we can see, there is no foreseeable negative impact on the society.



\nocite{*}
\bibliography{example_paper} 

\bibliographystyle{icml2024}

\newpage
\appendix
\onecolumn
\section{Proof}
\label{app:proof}
\begin{proof}[Proof of Theorem \ref{thm:unique}]
    Initially, it is simple to verify that the sequential sampling procedure \eqref{p:autoregressive} has a unique solution. Considering the initial conditions $x_T=y_T=\xi_T$, let us assume for the sake of contradiction that there exist two distinct solutions $x_{0:T-1}$ and $y_{0:T-1}$. 
    \begin{align*}
        x_{t-1} &= a_tx_t+b_t\epsilon_\theta(x_t,t)+c_{t-1}\xi_{t-1},\ t=1,...,T,\\
        y_{t-1} &= a_ty_t+b_t\epsilon_\theta(y_t,t)+c_{t-1}\xi_{t-1},\ t=1,...,T.
    \end{align*}
    Using an induction argument, let us assume that for some $0 < t \leq T$, we have $x_t = y_t$. Under this assumption, we can show that
    \begin{align*}
        x_{t-1} &= a_tx_t+b_t\epsilon_\theta(x_t,t)+c_{t-1}\xi_{t-1}\\
        &= a_ty_t+b_t\epsilon_\theta(y_t,t)+c_{t-1}\xi_{t-1}\\
        &= y_{t-1}.
    \end{align*}
    Hence the two solutions $x_{0:T-1}$ and $y_{0:T-1}$ are the same.

We will now demonstrate that for any $1 \leq k \leq T$, the nonlinear equations given by \eqref{p:orderk} are equivalent. This implies that all sets of nonlinear equations share the same unique solution, since the case of $k=1$ corresponds to the sequential procedure outlined in \eqref{p:autoregressive}.

For the purpose of this proof, we define two sets of nonlinear equations to be equivalent if any solution to one set is also a solution to the other, and vice versa. To simplify the exposition, we will prove the equivalence of the 1st order equations to the 2nd order equations, while noting that the proof that $k$-th order equations are equivalent to $(k+1)$-th order equations follows a similar procedure. Assume that $x_{0:T-1}$ is a solution to the 1st order equations. It follows directly that $x_{0:T-1}$ satisfies the 2nd order equations as well, which can be seen by \eqref{p:order2}. Conversely, if we consider $x_{0:T-1}$ as a solution to the 2nd order equations, it can be shown that
\begin{align*}
    x_{t-1} &= \begin{cases}
        a_t\bigg({{ a_{t+1}x_{t+1}+b_{t+1}\epsilon_\theta(x_{t+1},t+1)+c_{t}\xi_{t}}}\bigg)+b_t\epsilon_\theta(x_t,t)+c_{t-1}\xi_{t-1},\ t<T,\\
        a_tx_t+b_t\epsilon_\theta(x_t,t)+c_{t-1}\xi_{t-1},\ t=T.
    \end{cases}.
\end{align*} 
With $x_{T-1}=a_Tx_T+b_T\epsilon_\theta(x_T,T)+c_{T-1}\xi_{T-1}$, we can show that 
\begin{align*}
    x_{T-2}&=a_{T-1}\bigg({{ a_Tx_T+b_T\epsilon_\theta(x_T,T)+c_{T-1}\xi_{T-1}}}\bigg)+b_{T-1}\epsilon_\theta(x_{T-1},T-1)+c_{T-2}\xi_{T-2}\\
    &=a_{T-1}x_{T-1}+b_{T-1}\epsilon_\theta(x_{T-1},T-1)+c_{T-2}\xi_{T-2}.
\end{align*}
With the same procedure, we can show that $x_{t-1}=a_tx_t+b_t\epsilon_\theta(x_t,t)+c_{t-1}\xi_{t-1}$ for  $t=1,...,T-2$. Hence, $x_{0:T-1}$ is a solution of $1$st order equations.

\end{proof}

\begin{proof}[Proof of Theorem \ref{thm:G_upper}]

 Since $T^i=-I+Q^i$, the inverse multisecant condition \eqref{p:inverse_multisecant} can be written as
\begin{align}
    \label{p:q_inv}
    Q^i \Fc^i_{t_1:t_2} = \Xc^i_{t_1:t_2} + \Fc^i_{t_1:t_2}.
\end{align}
Since $Q^i$ is a block upper triangular matrix, the condition \eqref{p:q_inv} can be further simplified as
\begin{align}
    \label{p:q_inv2}
    Q^i[t':t'',t':] \Fc^i_{t:t_2} = \Xc^i_{t:t_2} + \Fc^i_{t:t_2},\quad t=t_1,t_1+1,...,t_2.
\end{align}

As one can see,  for each $t$, the linear equations $Q^i[t':t'',t':] \Fc^i_{t:t_2} = \Xc^i_{t:t_2} + \Fc^i_{t:t_2}$ is underdertermined, because $Q^i[t':t'',t':]\in \Rbb^{d\times(t_2-t)d}$, $\Fc^i_{t:t_2}\in\Rbb^{(t_2-t)d\times m_i}$, $\Xc^i_{t:t_2}\in\Rbb^{(t_2-t)d\times m_i}$, rank$(\Fc^i_{t:t_2})=m_i$ and $m_i=\min\{m,i\}<d$. 

As a classical result in linear regression analysis \cite{dennis1979least}, the minimum-norm solution for $Q^i[t':t'',t':]$ is given by
\begin{align}
    Q^i[t':t'',t':] &= \underset{Q\Fc^i_{t:t_2} = \Xc^i_{t:t_2} + \Fc^i_{t:t_2}}{\arg\min} \|Q\|_F \\
   &= (\Xc^i_{t}+\Fc^i_{t})(\Fc^{i\top}_{t:t_2} \Fc^i_{t:t_2})^{-1}\Fc^{i\top}_{t:t_2}.
\end{align}

Therefore,  
$\left\|T^i+I\right\|_F$ is minimal among all matrices satisfying both the inverse multisecant condition \eqref{p:inverse_multisecant} and the block upper triangular condition in Definition \ref{def:block_upper_triangular}.

\end{proof}

\begin{proof}[Proof of Theorem \ref{thm:convergence}]
    In this proof, we aim to establish that following $t$ iterations of the update rule, the variables $x_{T-t}, \ldots, x_{T-1}$ converge. We will demonstrate this result using inductive reasoning. At the initial step of the induction, given that $R^0_T$ is identically zero (denoted by $R^0_T\equiv \boldsymbol{0}$), it follows that $$G^0[(T-1)d:,(T-1)d:]=-I.$$ Therefore, the update of $x_{T-1}$ is given by $$x_{T-1}^1=x_{T-1}^0 - (-I) R_{T-1}^0=x_{T-1}^0 + R_{T-1}^0 = x_{T-1}^0 +  F^{(k)}_{T-1}(x_{T}) - x_{t}^{0}=F^{(k)}_{T-1}(x_{T}).$$
    Therefore, $x_{T-1}^1$ converges and hence $R_{T-1}^1=F^{(k)}_{T-1}(x_{T}) - x_{t}^{1}=\boldsymbol{0}$. Now we suppose that after $t<T$ steps, $x_{T-t},...,x_{T-1}$ converges, i.e., $R_{T-t}^t=R_{T-1}^t=\boldsymbol{0}$. Then, we have 
    $$G^{t}[(T-t-1)d:,(T-t-1)d:]=-I.$$ Hence similarly, we have
    \begin{align*}
        x_{T-t-1}^{t+1}&=x_{T-t-1}^{t} - (-I) R_{T-t-1}^t=x_{T-t-1}^{t} + R_{T-t-1}^t \\
        &= x_{T-t-1}^{t} +  F^{(k)}_{T-t-1}(x_{T-t}^{t},...,x_{{(T-t)}_k}^{t}) - x_{T-t-1}^{t}
        \\ &=F^{(k)}_{T-t-1}(x_{T-t}^{t},...,x_{{(T-t)}_k}^{t}) ,
    \end{align*} 
    and hence $R_{T-t-1}^{t+1}=F^{(k)}_{T-t-1}(x_{T-t}^{t+1},...,x_{{(T-t)}_k}^{t+1})-x_{T-t-1}^{t+1}=F^{(k)}_{T-t-1}(x_{T-t}^{t},...,x_{{(T-t)}_k}^{t})-x_{T-t-1}^{t+1}=\boldsymbol{0}$. Thus $x_{T-t-1}$ converges after $T+1$ steps. By this induction, we can conclude that all the variables $x_{0},...,x_{T-1}$ converges after $T$ steps.

\end{proof}

\section{Further Exploration}
\label{app:triangular_anderson}

In this section, we delve deeper into the study of the nonlinear equations \eqref{p:orderk} and the performance of Algorithm \ref{alg:ParaTAA} in this section. For all the experiments in this section, we adopt the DDPM with 100 steps as the sequential sampling algorithm and employing DiT models. We present our findings in Figure \ref{fig:app_taa}.

In Figure \ref{fig:residual_pattern}, we plot the convergence behavior of the variables $x_0, \ldots, x_{T-1}$ under the fixed-point iteration \eqref{p:fixedpoint}, and notice that their residuals do not converge uniformly. This is largely attributed to the triangular structure present in \eqref{p:orderk}. Specifically, the earlier step variables $x_{81}, \ldots, x_{100}$ reach convergence within fewer than 10 steps, while the later step variables $x_{0}, \ldots, x_{20}$ take approximately 35 steps to converge. This observation reinforces our motivation for introducing Triangular Anderson Acceleration—to prevent the updating of near-converged variables with information from those that have not yet converged.

In Figure \ref{fig:safeguard}, we examine the impact of the safeguarding technique described in Theorem \ref{thm:convergence}. While this technique offers a worst-case guarantee, we find that it does not detract from the empirical effectiveness of Triangular Anderson Acceleration.

The third figure, Figure \ref{fig:diagnal}, demonstrates that simply extracting the upper triangular portion of the original Anderson Acceleration matrix \eqref{p:G} (denoted as AA+), despite of an improvement over the standard Anderson Acceleration, still falls short of our proposed Triangular Anderson Acceleration. More importantly, as shown in \eqref{p:G}, utilizing only the upper triangular component of $G^i$ does not ensure that $x^{i+1}_t$ in \eqref{p:anderson} is exclusively updated using information from previous iterations $x_{j}^i$ where $j \geq t$. This is because the inputs from $x_{j}^i$, with $j < t$, are still incorporated into $G^i$ through the inversion of the matrix $(\Fc^{i\top}_{t_1:t_2} \Fc^i_{t_1:t_2})^{-1}$. It is also important to note that for these experiments, we utilize 32-bit precision in our computations, as the AA and AA+ methods do not exhibit stability with 16-bit precision.

\begin{figure}[htpb]
	\centering
	\begin{subfigure}[b]{0.3\textwidth}
	\centering
	\includegraphics[width=\textwidth]{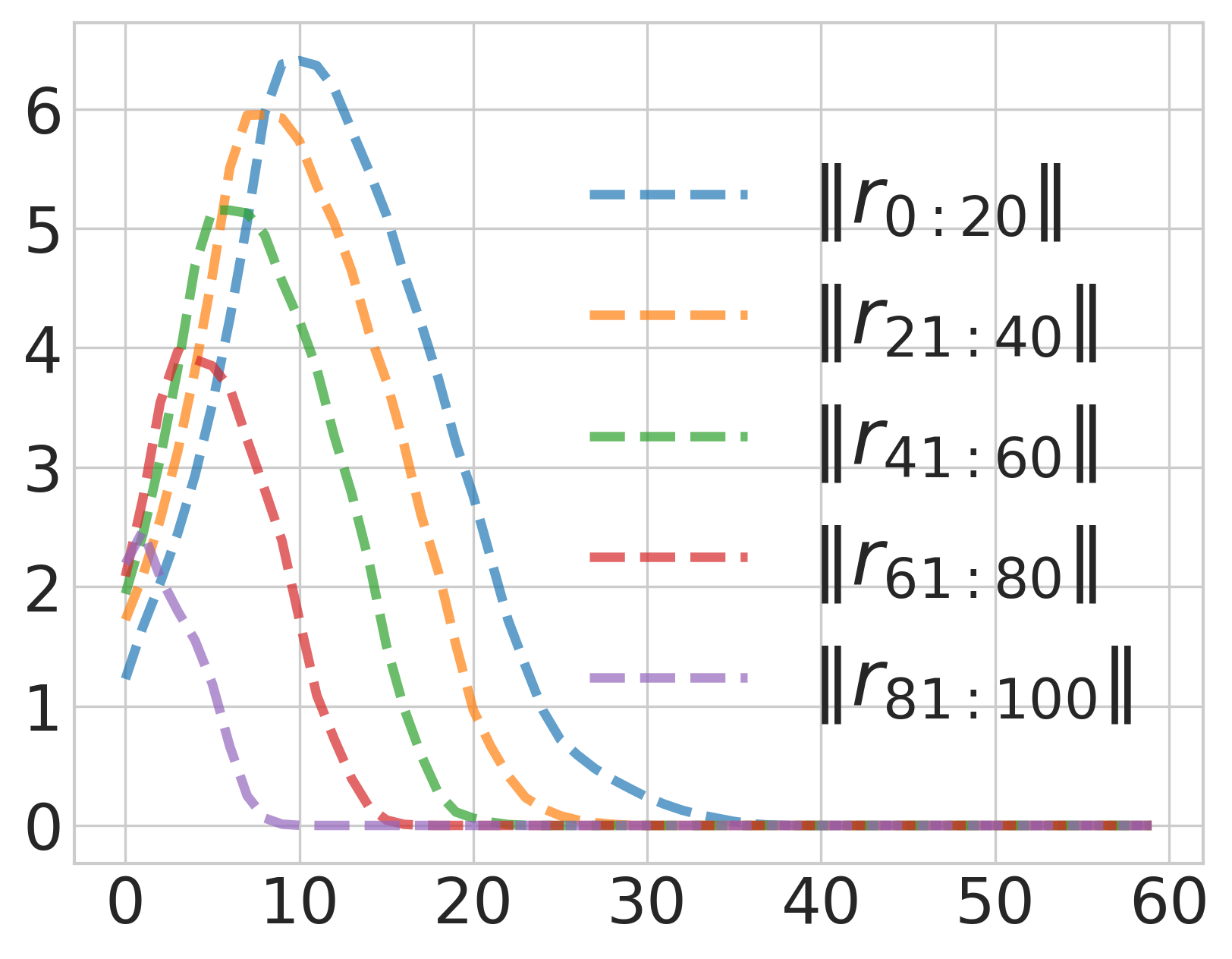}
 	\caption{Convergence of residuals under different timesteps}
	\label{fig:residual_pattern}
\end{subfigure}
	\begin{subfigure}[b]{0.3\textwidth}
	\centering
	\includegraphics[width=\textwidth]{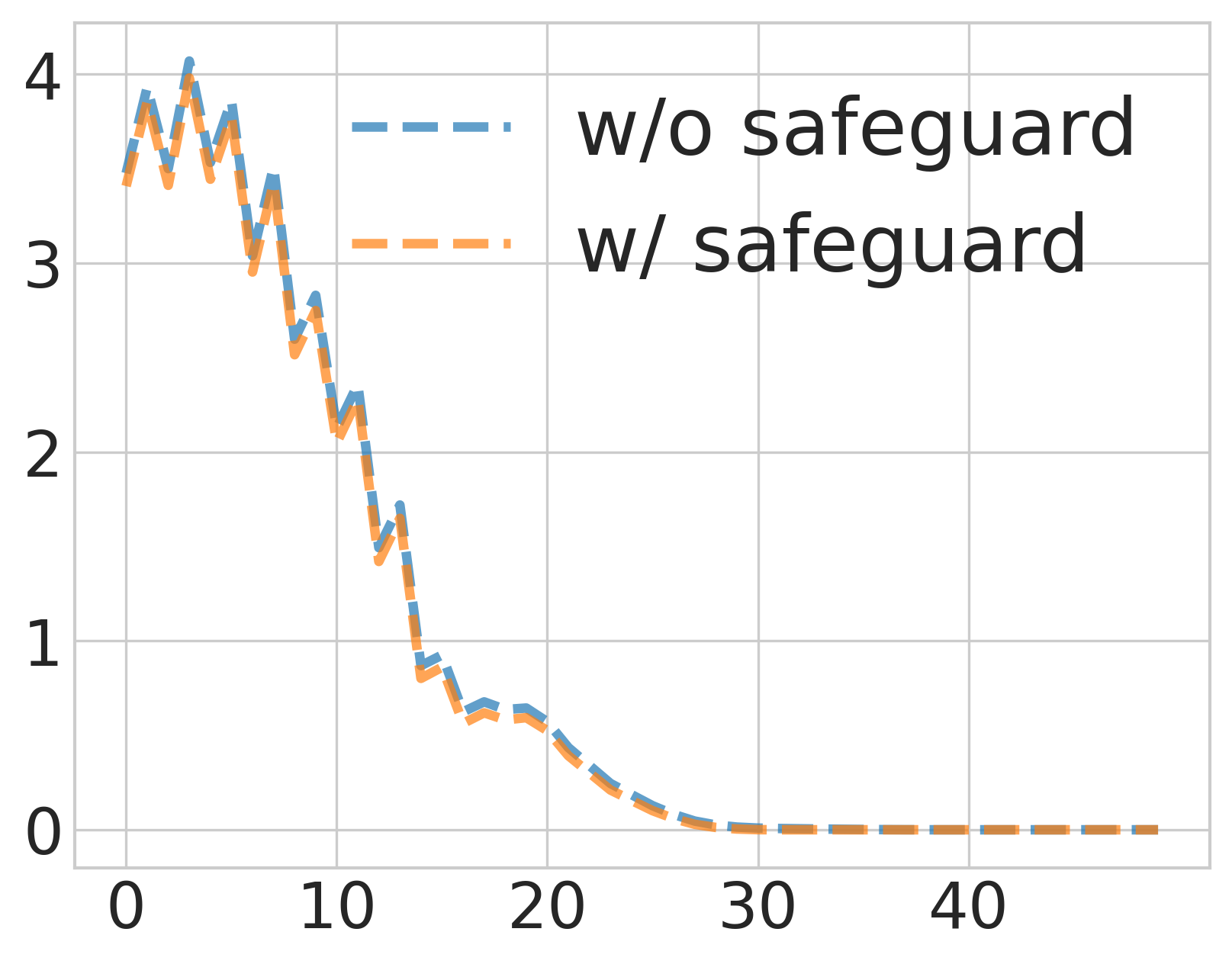}
	\caption{TAA with/without Safeguarding}
	\label{fig:safeguard}
\end{subfigure}
\begin{subfigure}[b]{0.3\textwidth}
	\centering
	\includegraphics[width=\textwidth]{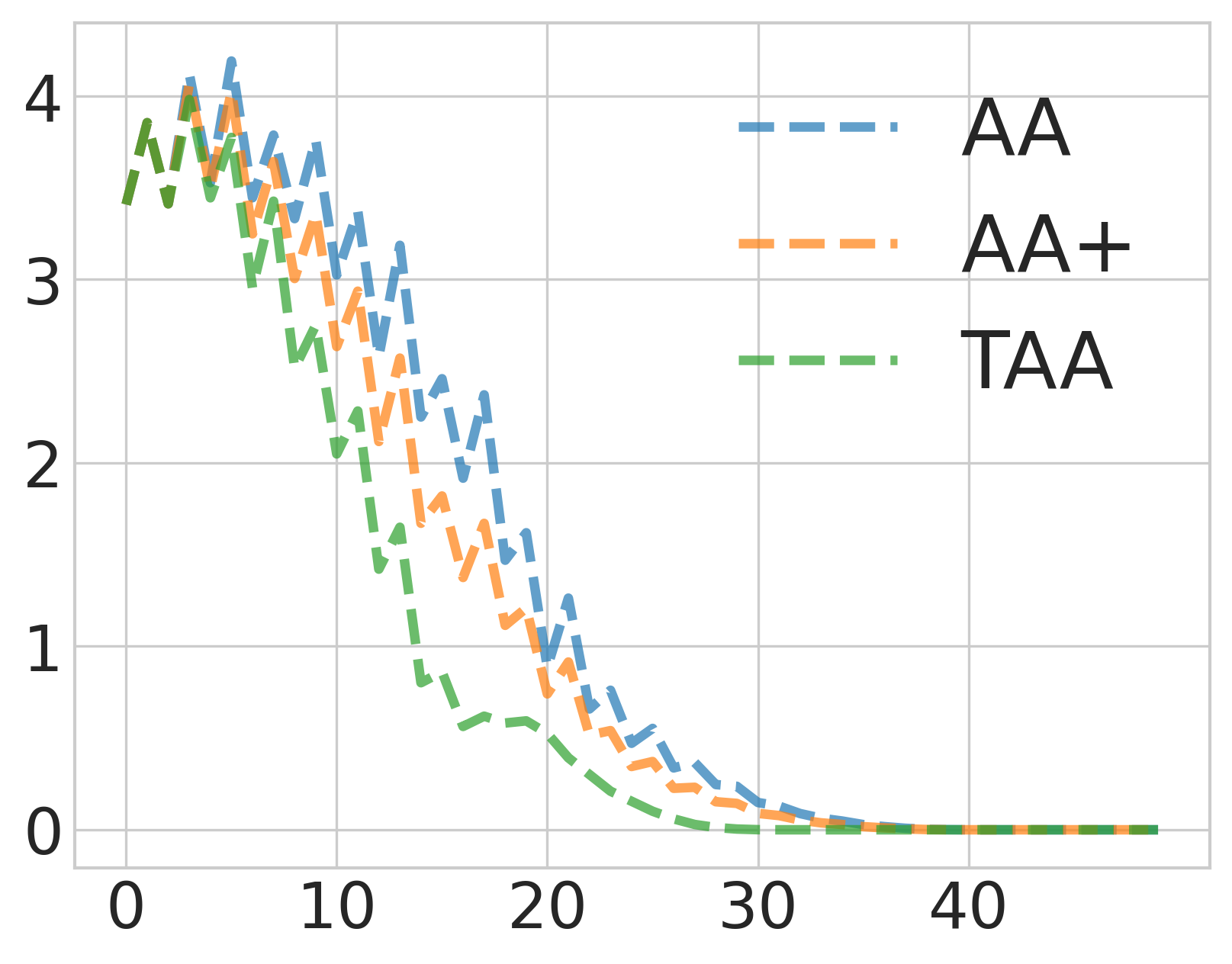}
	\caption{Comparing TAA to AA+}
	\label{fig:diagnal}
\end{subfigure}
\caption{More investigation on TAA.}
	\label{fig:app_taa}
\end{figure}

\section{Hyperparameter Analysis}
\label{app:hyperparameter}

In this section, we present grid search results for sampling with DiT models  under the four scenarios outlined in Section \ref{sec:main_exp}: DDIM with 25 steps, DDIM with 50 steps, DDIM with 100 steps, and DDPM with 100 steps. We fixed the window size $w$ to match the total number of sequential sampling steps. The grid search is performed for the order $k$ and history size $m$ in Algorithm \ref{alg:ParaTAA}. We used the average number of steps required to achieve convergence for 100 different seeds as a metric to assess the performance of different hyperparameters.
 It is important to observe that when $m=1$, Algorithm \ref{alg:ParaTAA} reverts to the fixed-point iteration \eqref{p:fixedpoint} since it does not utilize historical information. The results are summarized in Figure \ref{fig:app_hyperparameter}.

Based on the grid search results shown in Figure \ref{fig:app_hyperparameter}, we can draw several conclusions. Firstly, the optimal history size $m$ appears to be between 2 and 4, as utilizing additional historical information may be detrimental to performance. Secondly, for $m\geq 2$, the algorithm becomes quite resilient to changes in the order $k$, provided that $k$ is sufficiently large. Conversely, with $m=1$, corresponding to fixed-point iteration, the algorithm performs best with a smaller $k$.

An interesting observation from Figure \ref{fig:app_hyperparameter} is that for all DDIM scenarios (25, 50, and 100 steps), the ParaTAA algorithm tends to converge in roughly the same number of steps. Furthermore, we note that the DDPM typically demands more steps to reach convergence compared to the DDIM. We speculate that the inclusion of a noise term in the nonlinear equations \eqref{p:orderk} may exacerbate the optimization landscape for the fixed-point iteration. 

Since the SD models exhibit a similar pattern of hyperparameters as shown in Figure \ref{fig:app_hyperparameter}, we choose to use the same optimal hyperparameters for both the DiT and SD experiments in Section \ref{sec:main_exp}.

\begin{figure}[htpb]
	\centering
	\begin{subfigure}[b]{0.24\textwidth}
	\centering
	\includegraphics[width=\textwidth]{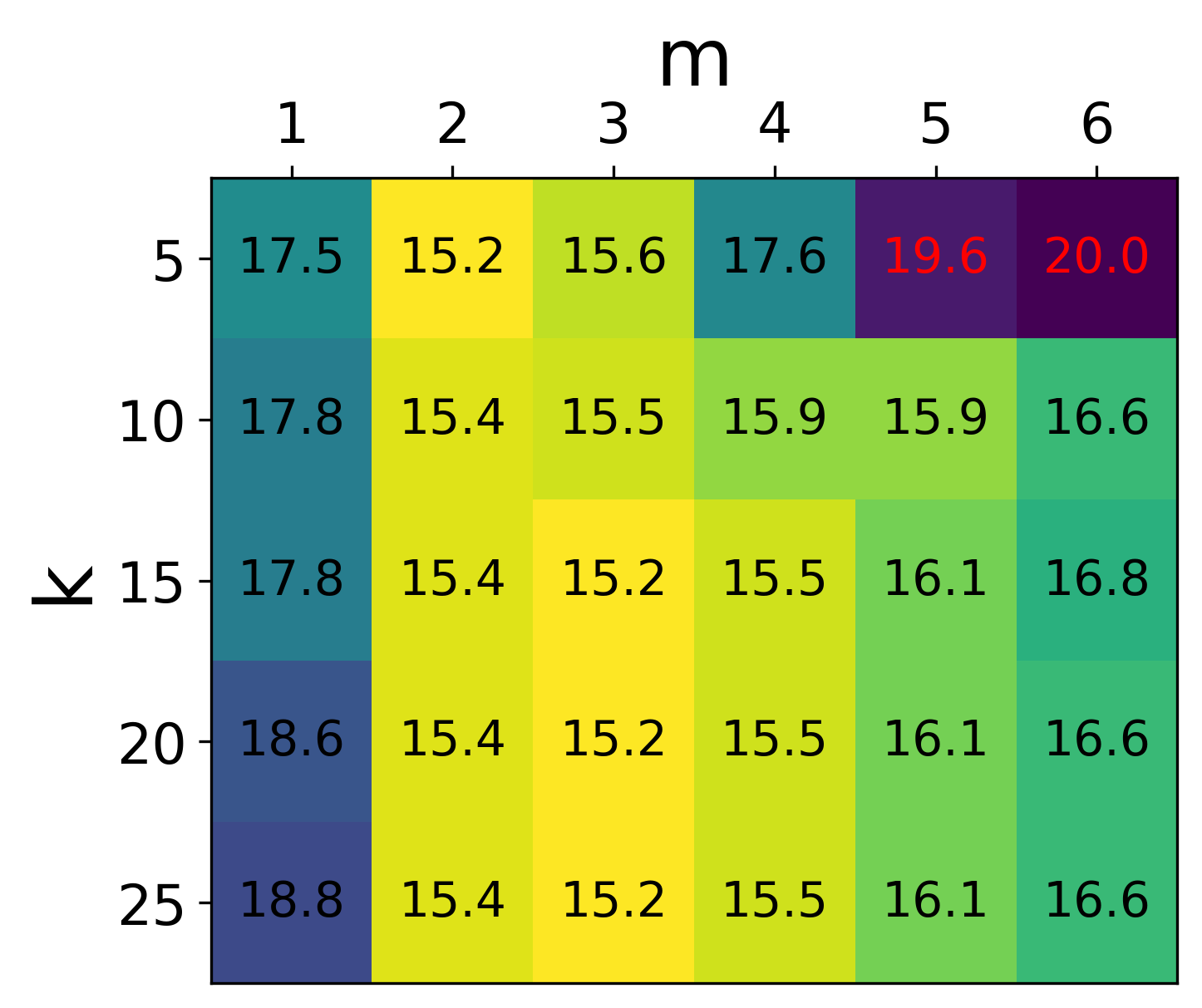}
 	\caption{DDIM 25 steps}
	\label{}
\end{subfigure}
	\begin{subfigure}[b]{0.24\textwidth}
	\centering
	\includegraphics[width=\textwidth]{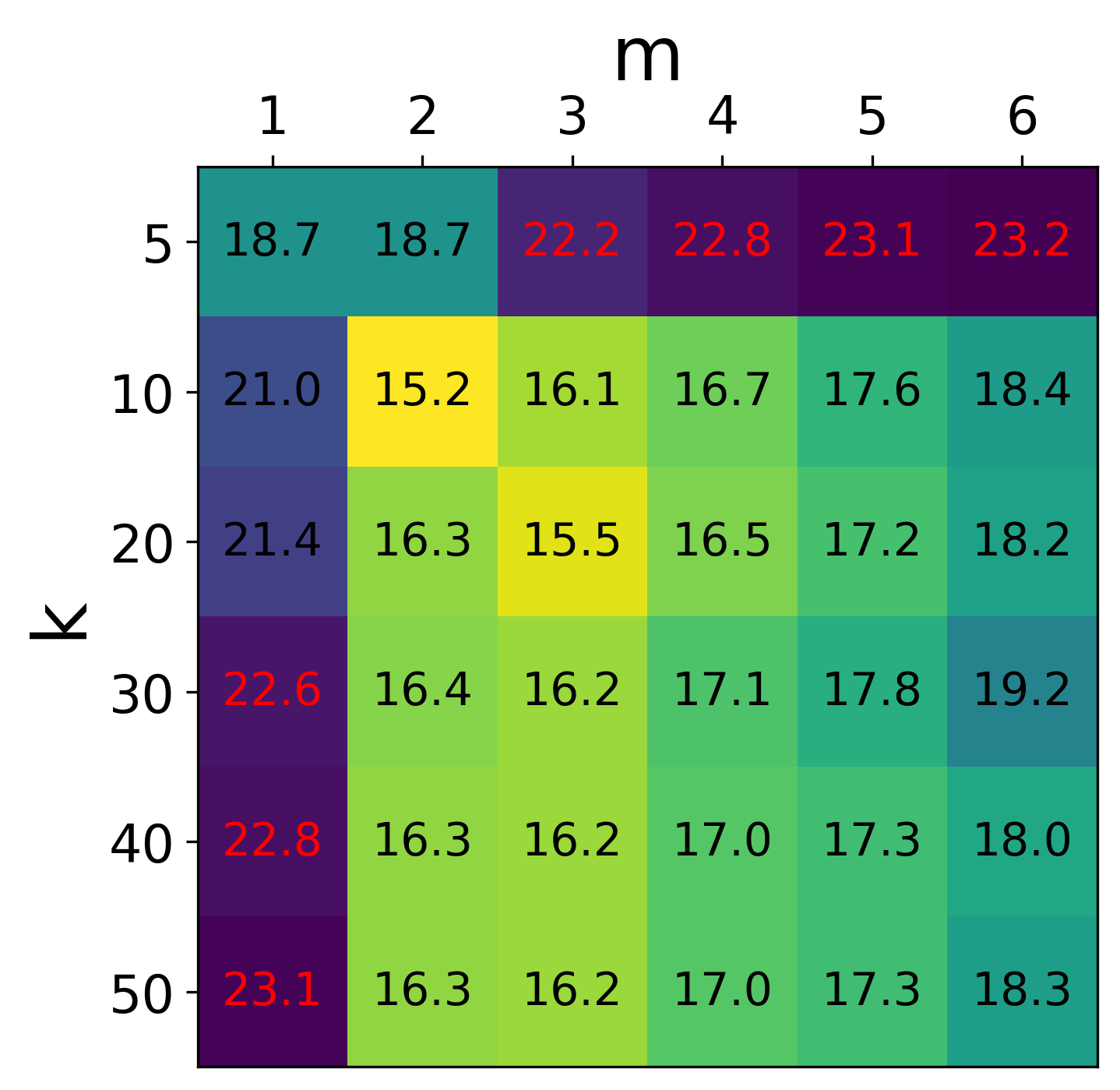}
	\caption{DDIM 50 steps}
	\label{}
\end{subfigure}
\begin{subfigure}[b]{0.24\textwidth}
	\centering
	\includegraphics[width=\textwidth]{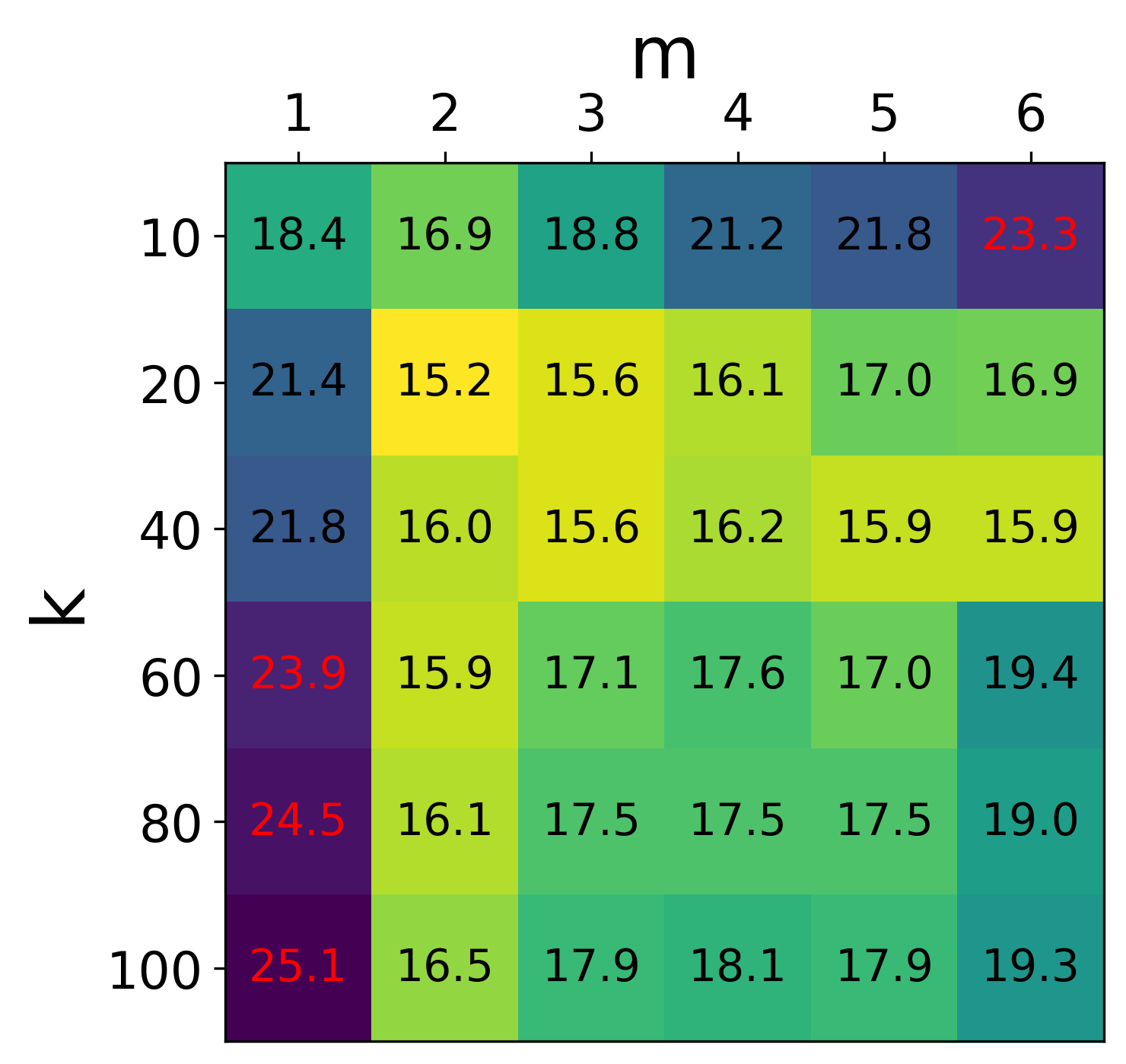}
 	\caption{DDIM 100 steps}
	\label{}
\end{subfigure}
	\begin{subfigure}[b]{0.24\textwidth}
	\centering
	\includegraphics[width=\textwidth]{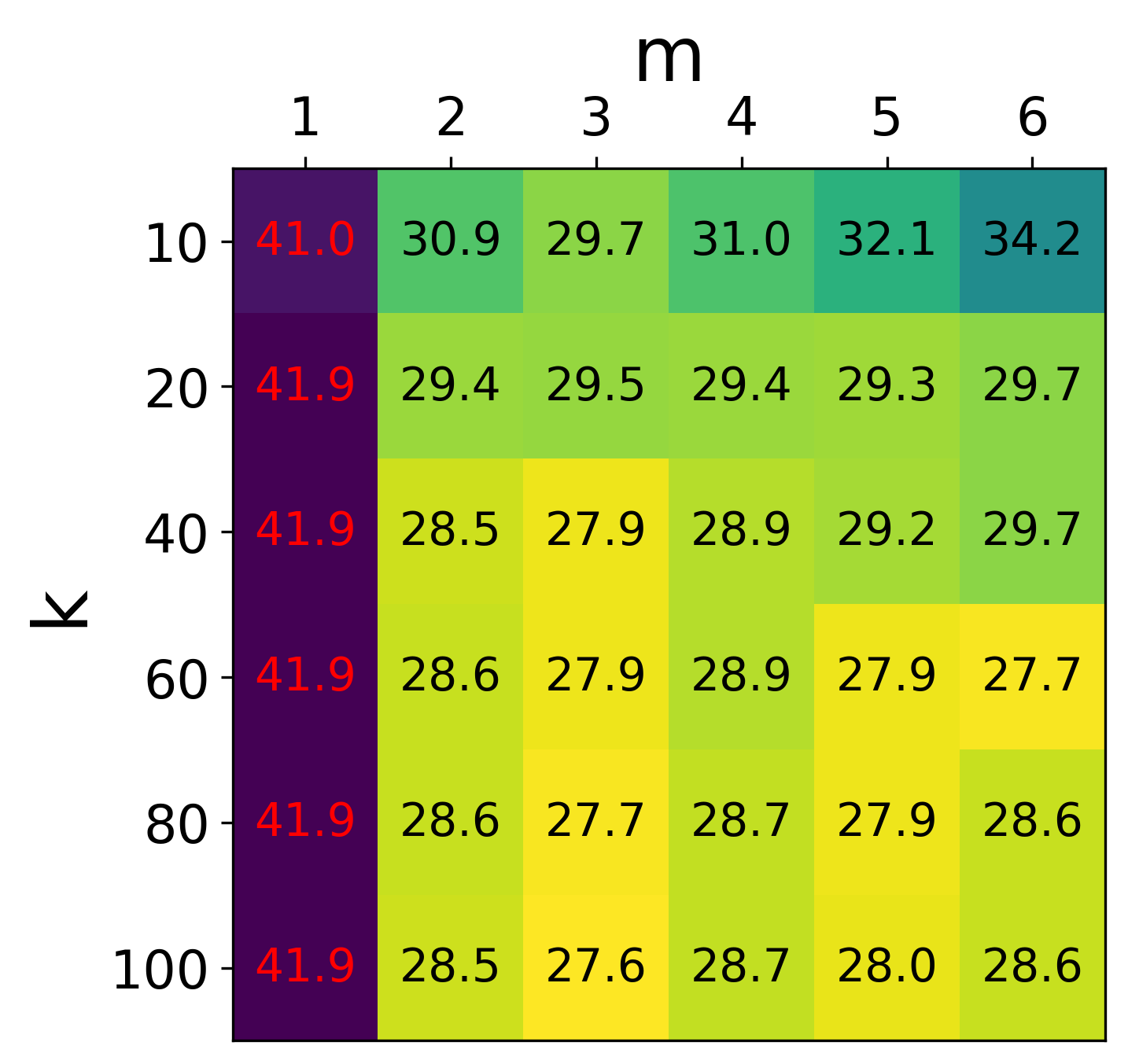}
	\caption{DDPM 100 steps}
	\label{}
\end{subfigure}
\caption{Hyperparameter Analysis for ParaTAA.}
	\label{fig:app_hyperparameter}
\end{figure}

\section{Qualitative Comparison}
\label{app:generated_image}

This section presents a qualitative visual comparison of the convergence behaviors of ParaTAA, FP, and FP+ as illustrated in Figure \ref{fig:main-result}. We feature examples from the following scenarios: DiT with DDIM at 100 steps, DiT with DDPM at 100 steps, SD with DDIM at 100 steps, and SD with DDPM at 100 steps. The images displayed showcase the convergence process at different iteration stages for each algorithm. Sequentially generated images are provided in Figure \ref{fig:demo_sequential} for comparison, while the images generated through parallel sampling are depicted in Figures \ref{fig:demo_imagenet_ddim100}, \ref{fig:demo_imagenet_ddpm100}, \ref{fig:demo_stable_ddim100}, and \ref{fig:demo_stable_ddpm100}, corresponding to the four aforementioned scenarios.

As is evident from the visualizations, it is clear that our proposed ParaTAA algorithm significantly outperforms the naive fixed-point iteration (FP) and its variant with optimal order (FP+). Moreover, for both DiT and SD models with DDIM at 100 steps, ParaTAA successfully produces images of similar quality to those obtained via sequential sampling within a mere 7 iterations. In the case of DDPM at 100 steps, ParaTAA achieves comparable results to sequential sampling within only 21 iterations.

\begin{figure}[htpb]
	\centering
	\begin{subfigure}[b]{0.2\textwidth}
	\centering
	\includegraphics[width=1\textwidth]{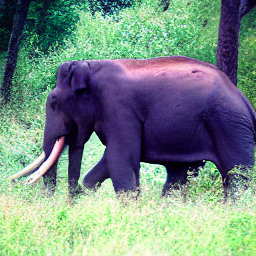}
 \caption{DiT DDIM-100}
\end{subfigure}
	\begin{subfigure}[b]{0.2\textwidth}
	\centering
	\includegraphics[width=1\textwidth]{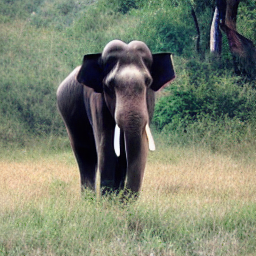}
    \caption{DiT DDPM 100}
\end{subfigure}
\begin{subfigure}[b]{0.2\textwidth}
	\centering
	\includegraphics[width=1\textwidth]{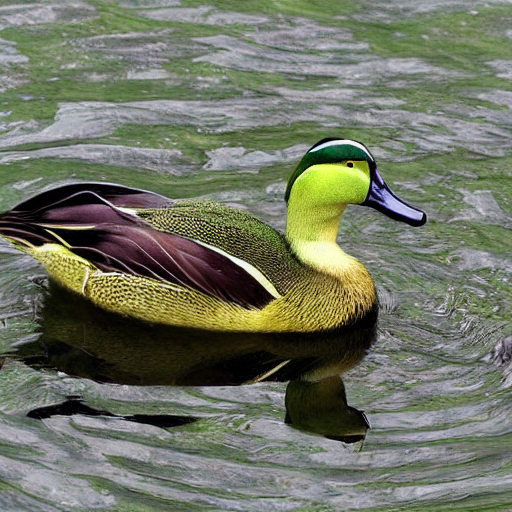}
    \caption{SD DDIM 100}
\end{subfigure}
	\begin{subfigure}[b]{0.2\textwidth}
	\centering
	\includegraphics[width=1\textwidth]{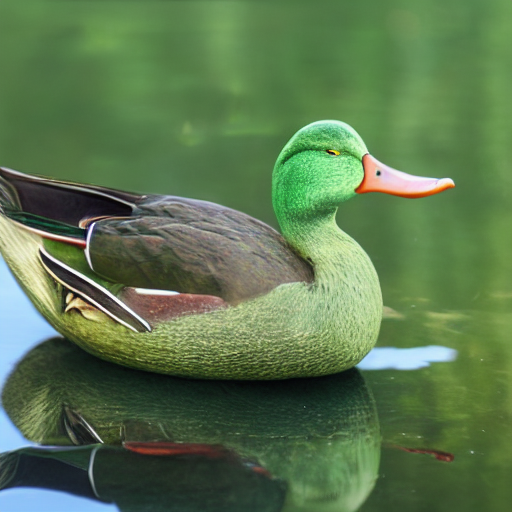}
    \caption{SD DDPM 100}
\end{subfigure}
\caption{Generated Images from Sequential Sampling. For DiT model, we use the class for "elephant" as the input condition. For SD model, we use the "green duck" as the text prompt.}
	\label{fig:demo_sequential}
\end{figure}

\begin{figure}[htpb]
	\centering
	\begin{subfigure}[b]{0.2\textwidth}
	\centering
	\includegraphics[width=0.9\textwidth]{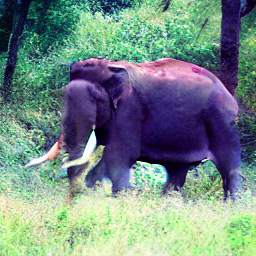}
 
\end{subfigure}
	\begin{subfigure}[b]{0.2\textwidth}
	\centering
	\includegraphics[width=0.9\textwidth]{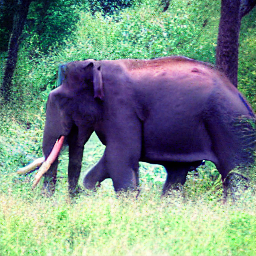}
\end{subfigure}
\begin{subfigure}[b]{0.2\textwidth}
	\centering
	\includegraphics[width=0.9\textwidth]{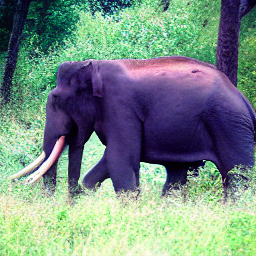}
 
\end{subfigure}
	\begin{subfigure}[b]{0.2\textwidth}
	\centering
	\includegraphics[width=0.9\textwidth]{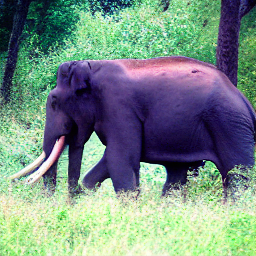}
\end{subfigure}
\begin{subfigure}[b]{0.2\textwidth}
	\centering
	\includegraphics[width=0.9\textwidth]{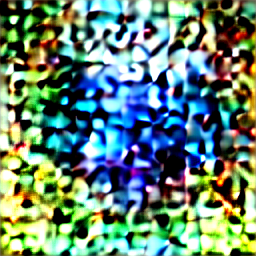}
 
\end{subfigure}
	\begin{subfigure}[b]{0.2\textwidth}
	\centering
	\includegraphics[width=0.9\textwidth]{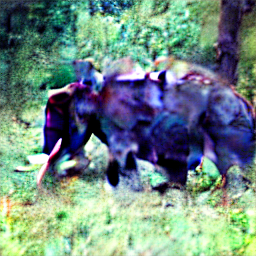}
\end{subfigure}
\begin{subfigure}[b]{0.2\textwidth}
	\centering
	\includegraphics[width=0.9\textwidth]{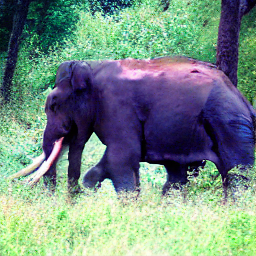}
 
\end{subfigure}
	\begin{subfigure}[b]{0.2\textwidth}
	\centering
	\includegraphics[width=0.9\textwidth]{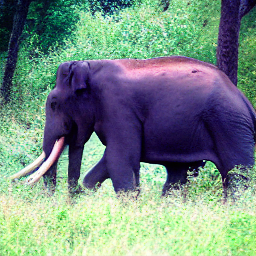}
\end{subfigure}
\begin{subfigure}[b]{0.2\textwidth}
	\centering
	\includegraphics[width=0.9\textwidth]{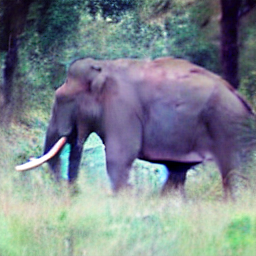}
    \caption{After 7 steps}
\end{subfigure}
	\begin{subfigure}[b]{0.2\textwidth}
	\centering
	\includegraphics[width=0.9\textwidth]{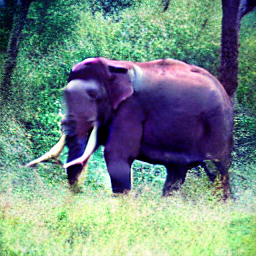}
    \caption{After 11 steps}
\end{subfigure}
\begin{subfigure}[b]{0.2\textwidth}
	\centering
	\includegraphics[width=0.9\textwidth]{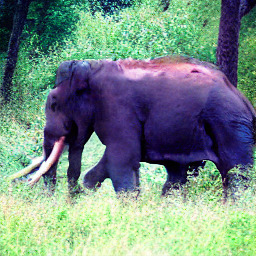}
    \caption{After 15 steps}
\end{subfigure}
	\begin{subfigure}[b]{0.2\textwidth}
	\centering
	\includegraphics[width=0.9\textwidth]{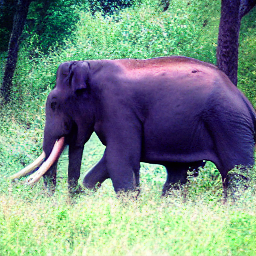}
    \caption{After 19 steps}
\end{subfigure}
\caption{Iterations of parallel sampling for DDIM 100 steps with DiT model. From top to bottom, the images are generated by ParaTAA, FP and FP+ respectively.}
	\label{fig:demo_imagenet_ddim100}
\end{figure}

\begin{figure}[htpb]
	\centering
	\begin{subfigure}[b]{0.2\textwidth}
	\centering
	\includegraphics[width=0.9\textwidth]{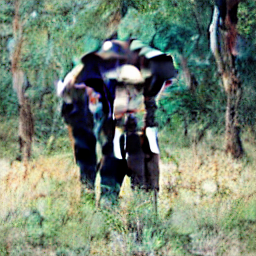}
 
\end{subfigure}
	\begin{subfigure}[b]{0.2\textwidth}
	\centering
	\includegraphics[width=0.9\textwidth]{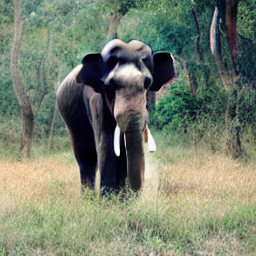}
\end{subfigure}
\begin{subfigure}[b]{0.2\textwidth}
	\centering
	\includegraphics[width=0.9\textwidth]{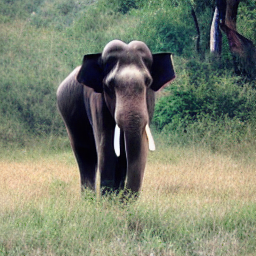}
 
\end{subfigure}
	\begin{subfigure}[b]{0.2\textwidth}
	\centering
	\includegraphics[width=0.9\textwidth]{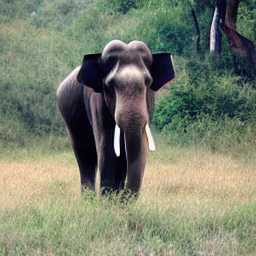}
\end{subfigure}
\begin{subfigure}[b]{0.2\textwidth}
	\centering
	\includegraphics[width=0.9\textwidth]{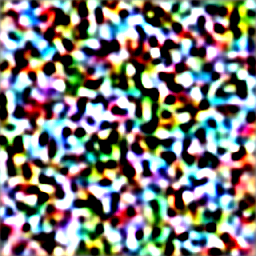}
 
\end{subfigure}
	\begin{subfigure}[b]{0.2\textwidth}
	\centering
	\includegraphics[width=0.9\textwidth]{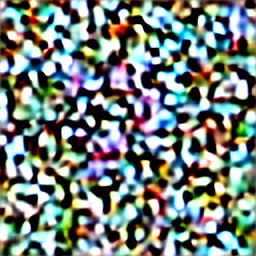}
\end{subfigure}
\begin{subfigure}[b]{0.2\textwidth}
	\centering
	\includegraphics[width=0.9\textwidth]{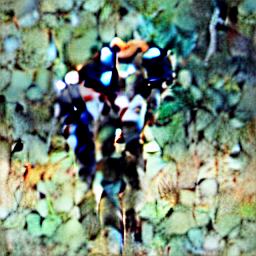}
 
\end{subfigure}
	\begin{subfigure}[b]{0.2\textwidth}
	\centering
	\includegraphics[width=0.9\textwidth]{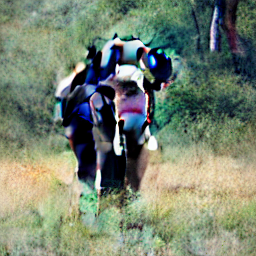}
\end{subfigure}
\begin{subfigure}[b]{0.2\textwidth}
	\centering
	\includegraphics[width=0.9\textwidth]{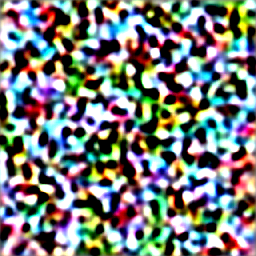}
    \caption{After 17 steps}
\end{subfigure}
	\begin{subfigure}[b]{0.2\textwidth}
	\centering
	\includegraphics[width=0.9\textwidth]{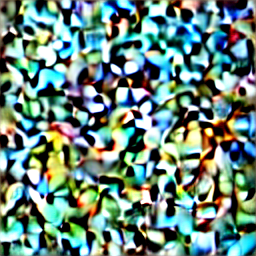}
    \caption{After 21 steps}
\end{subfigure}
\begin{subfigure}[b]{0.2\textwidth}
	\centering
	\includegraphics[width=0.9\textwidth]{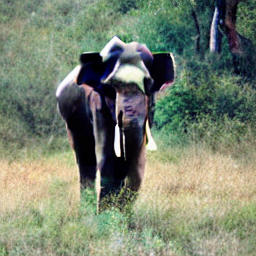}
    \caption{After 27 steps}
\end{subfigure}
	\begin{subfigure}[b]{0.2\textwidth}
	\centering
	\includegraphics[width=0.9\textwidth]{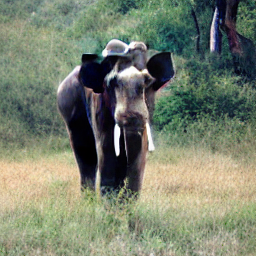}
    \caption{After 31 steps}
\end{subfigure}
\caption{Iterations of parallel sampling for DDPM 100 steps with DiT model. From top to bottom, the images are generated by ParaTAA, FP and FP+ respectively.}
	\label{fig:demo_imagenet_ddpm100}
\end{figure}

\begin{figure}[htpb]
	\centering
	\begin{subfigure}[b]{0.2\textwidth}
	\centering
	\includegraphics[width=0.9\textwidth]{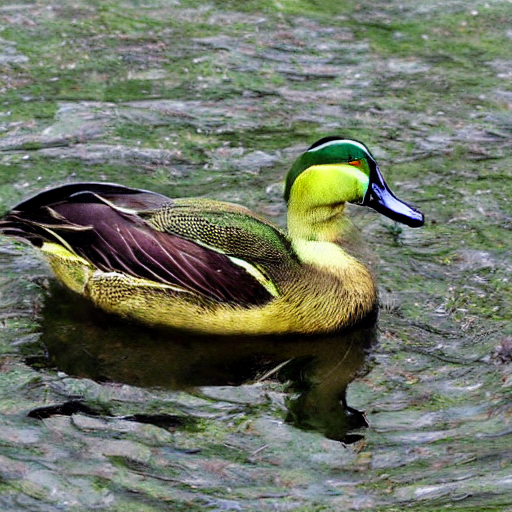}
 
\end{subfigure}
	\begin{subfigure}[b]{0.2\textwidth}
	\centering
	\includegraphics[width=0.9\textwidth]{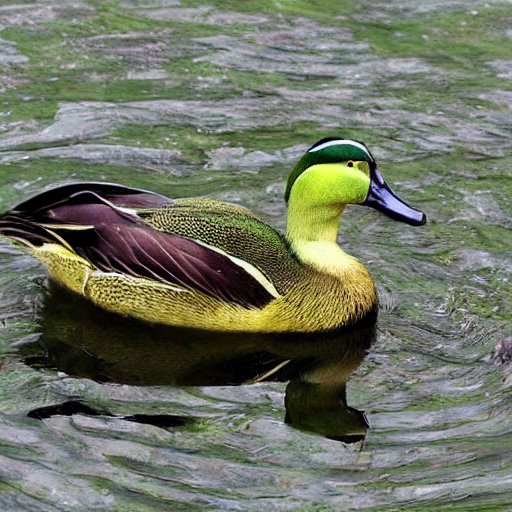}
\end{subfigure}
\begin{subfigure}[b]{0.2\textwidth}
	\centering
	\includegraphics[width=0.9\textwidth]{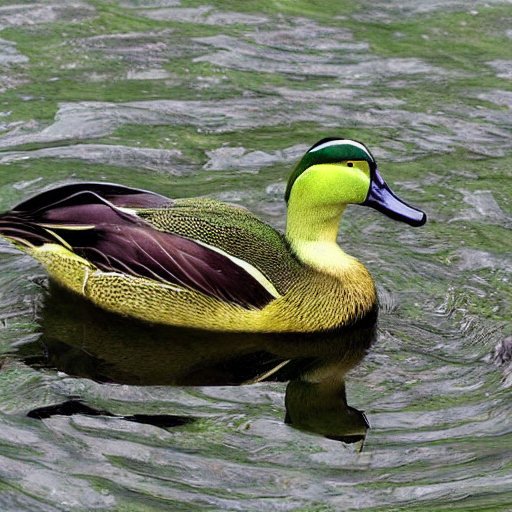}
 
\end{subfigure}
	\begin{subfigure}[b]{0.2\textwidth}
	\centering
	\includegraphics[width=0.9\textwidth]{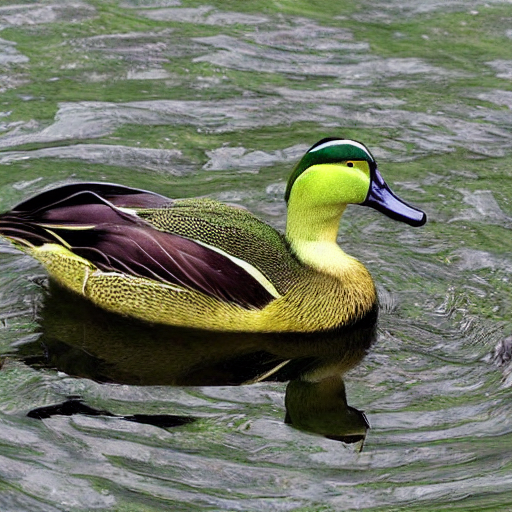}
\end{subfigure}
\begin{subfigure}[b]{0.2\textwidth}
	\centering
	\includegraphics[width=0.9\textwidth]{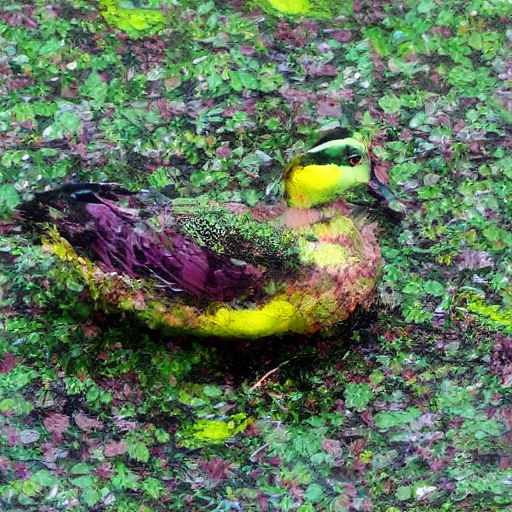}
 
\end{subfigure}
	\begin{subfigure}[b]{0.2\textwidth}
	\centering
	\includegraphics[width=0.9\textwidth]{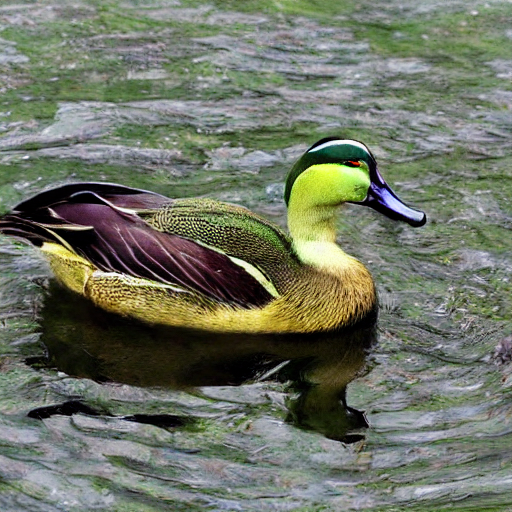}
\end{subfigure}
\begin{subfigure}[b]{0.2\textwidth}
	\centering
	\includegraphics[width=0.9\textwidth]{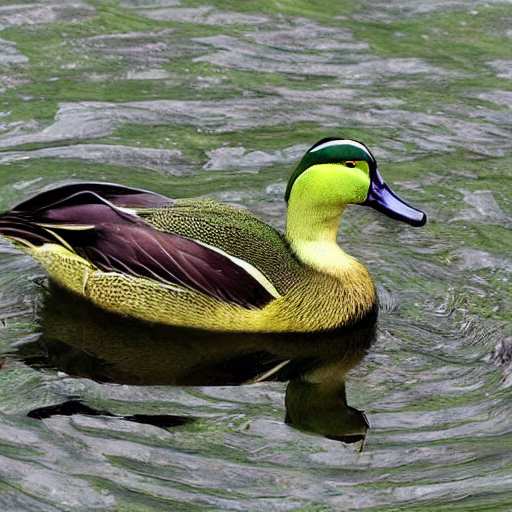}
 
\end{subfigure}
	\begin{subfigure}[b]{0.2\textwidth}
	\centering
	\includegraphics[width=0.9\textwidth]{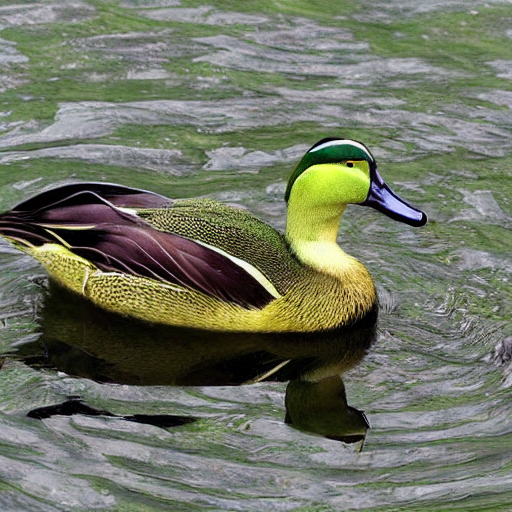}
\end{subfigure}
\begin{subfigure}[b]{0.2\textwidth}
	\centering
	\includegraphics[width=0.9\textwidth]{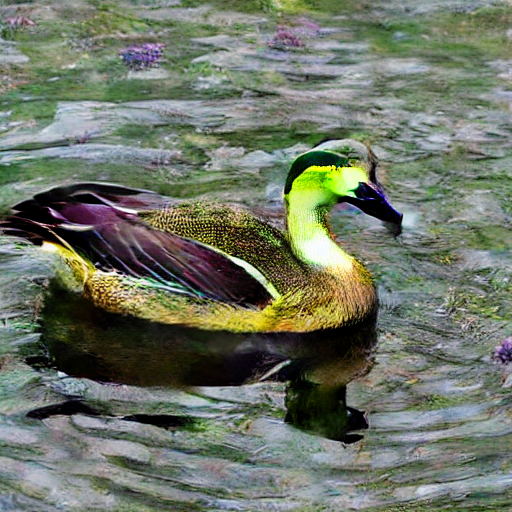}
    \caption{After 5 steps}
\end{subfigure}
	\begin{subfigure}[b]{0.2\textwidth}
	\centering
	\includegraphics[width=0.9\textwidth]{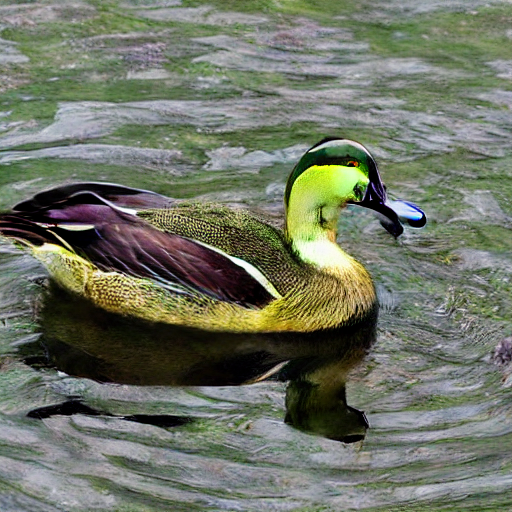}
    \caption{After 7 steps}
\end{subfigure}
\begin{subfigure}[b]{0.2\textwidth}
	\centering
	\includegraphics[width=0.9\textwidth]{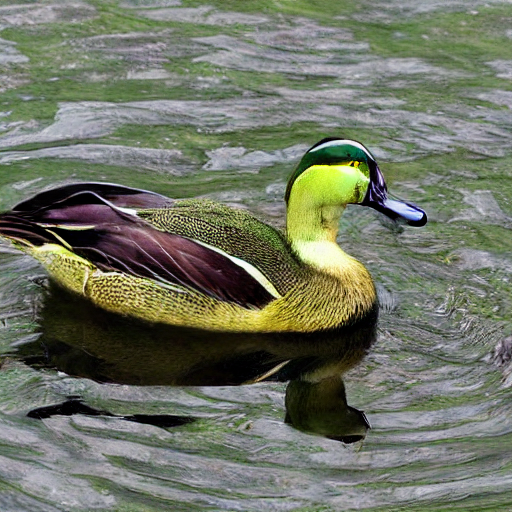}
    \caption{After 9 steps}
\end{subfigure}
	\begin{subfigure}[b]{0.2\textwidth}
	\centering
	\includegraphics[width=0.9\textwidth]{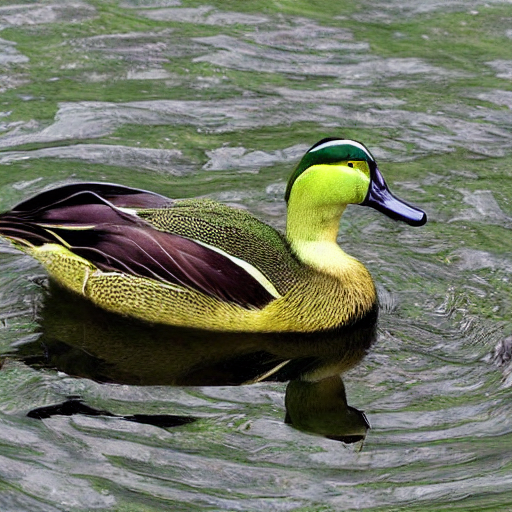}
    \caption{After 11 steps}
\end{subfigure}

\caption{Iterations of parallel sampling for DDIM 100 steps with SD model. From top to bottom, the images are generated by ParaTAA, FP and FP+ respectively.}
	\label{fig:demo_stable_ddim100}
\end{figure}

\begin{figure}[htpb]
	\centering
	\begin{subfigure}[b]{0.2\textwidth}
	\centering
	\includegraphics[width=0.9\textwidth]{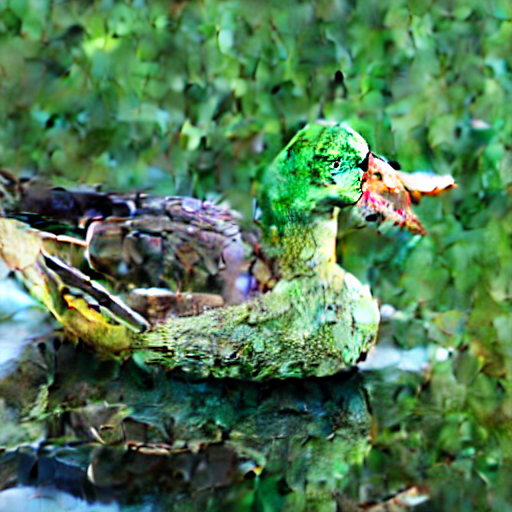}
 
\end{subfigure}
	\begin{subfigure}[b]{0.2\textwidth}
	\centering
	\includegraphics[width=0.9\textwidth]{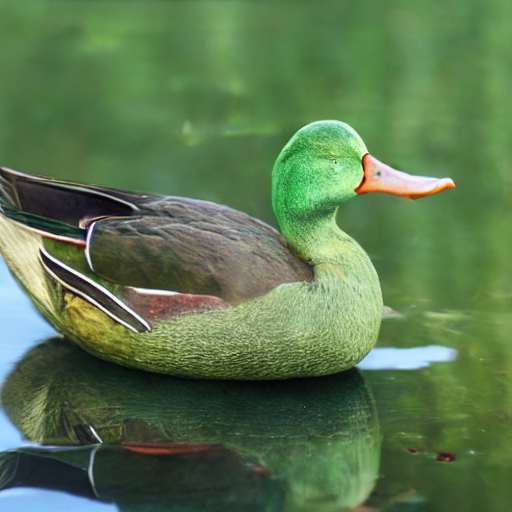}
\end{subfigure}
\begin{subfigure}[b]{0.2\textwidth}
	\centering
	\includegraphics[width=0.9\textwidth]{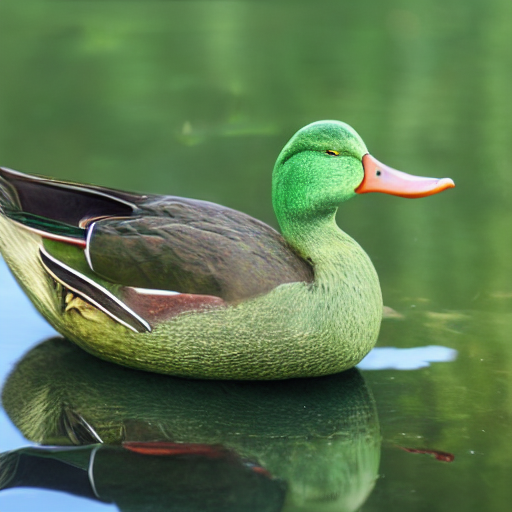}
 
\end{subfigure}
	\begin{subfigure}[b]{0.2\textwidth}
	\centering
	\includegraphics[width=0.9\textwidth]{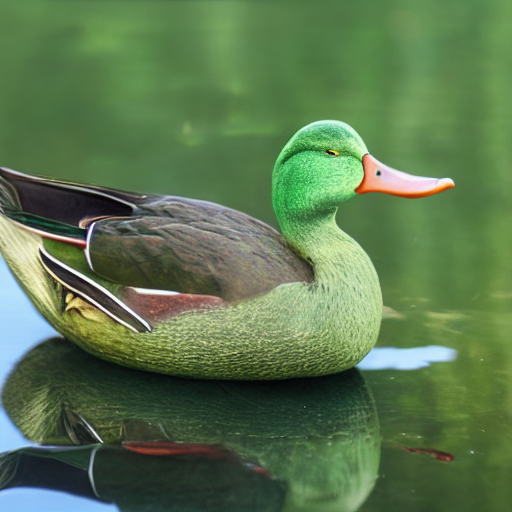}
\end{subfigure}
\begin{subfigure}[b]{0.2\textwidth}
	\centering
	\includegraphics[width=0.9\textwidth]{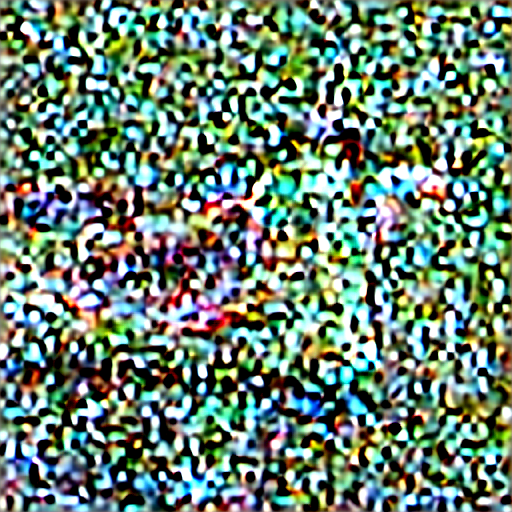}
 
\end{subfigure}
	\begin{subfigure}[b]{0.2\textwidth}
	\centering
	\includegraphics[width=0.9\textwidth]{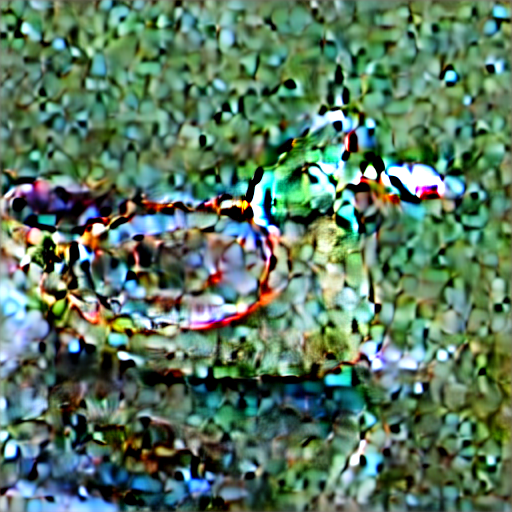}
\end{subfigure}
\begin{subfigure}[b]{0.2\textwidth}
	\centering
	\includegraphics[width=0.9\textwidth]{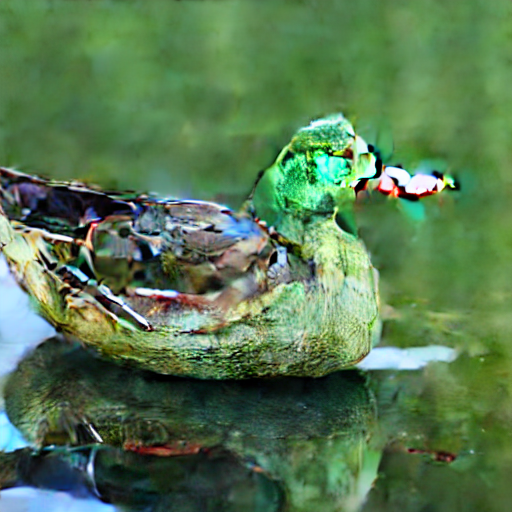}
 
\end{subfigure}
	\begin{subfigure}[b]{0.2\textwidth}
	\centering
	\includegraphics[width=0.9\textwidth]{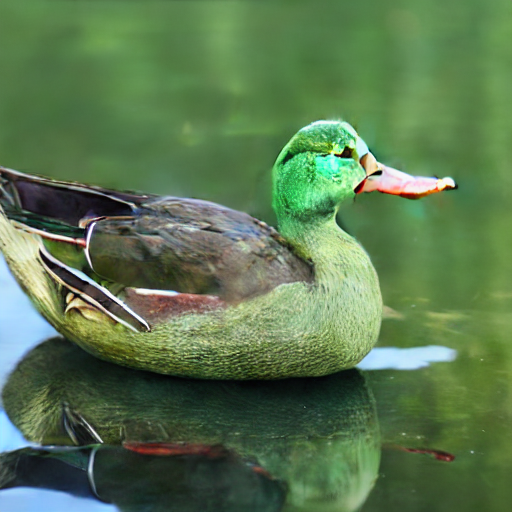}
\end{subfigure}
\begin{subfigure}[b]{0.2\textwidth}
	\centering
	\includegraphics[width=0.9\textwidth]{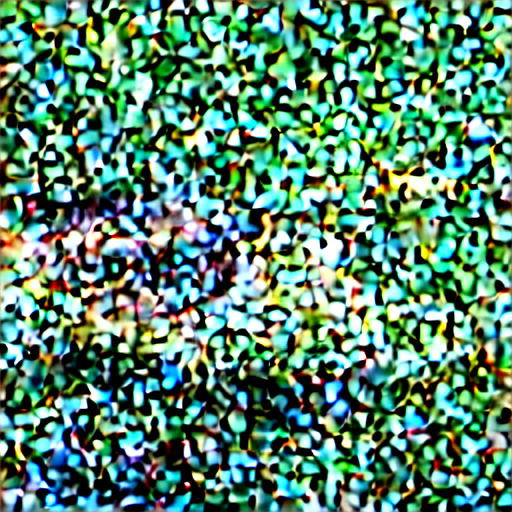}
    \caption{After 11 steps}
\end{subfigure}
	\begin{subfigure}[b]{0.2\textwidth}
	\centering
	\includegraphics[width=0.9\textwidth]{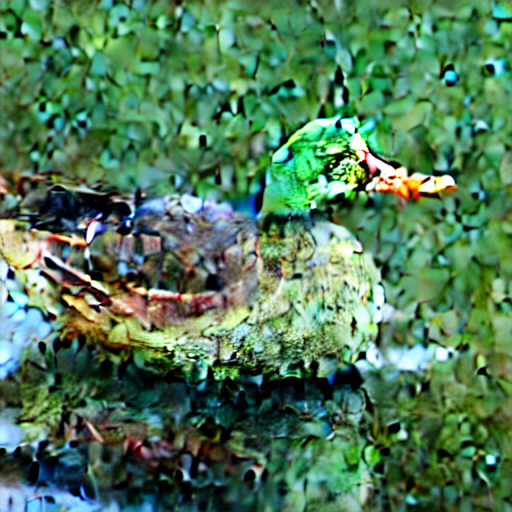}
    \caption{After 15 steps}
\end{subfigure}
\begin{subfigure}[b]{0.2\textwidth}
	\centering
	\includegraphics[width=0.9\textwidth]{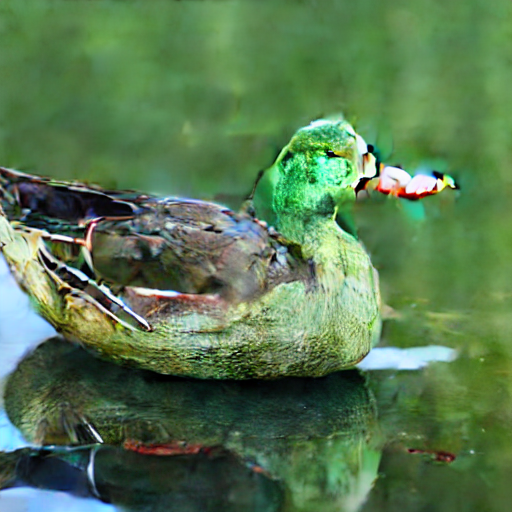}
    \caption{After 21 steps}
\end{subfigure}
	\begin{subfigure}[b]{0.2\textwidth}
	\centering
	\includegraphics[width=0.9\textwidth]{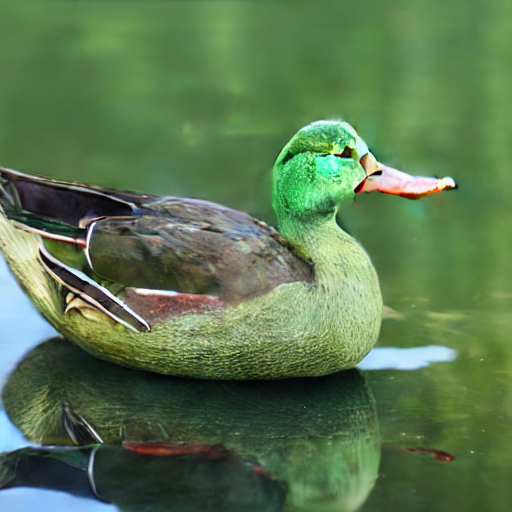}
    \caption{After 25 steps}
\end{subfigure}
\caption{Iterations of parallel sampling for DDPM 100 steps with SD model. From top to bottom, the images are generated by ParaTAA, FP and FP+ respectively.}
	\label{fig:demo_stable_ddpm100}
\end{figure}

\section{Quantitative Evaluation of Initialization from Existing Trajectory}
\label{app:quant_init}

In this section, we expand the results discussed in Section \ref{sec:init}. We present a more detailed set of convergence images in Figure \ref{fig:early_stop_app} as a fine-grained complement to Figure \ref{fig:early_stop}. This allows for a clearer observation on convergence when initialized with different methods. Additionally, we conduct a quantitative evaluation of the results shown in Figure \ref{fig:early_stop_app}, which is depicted in Figure \ref{fig:from_init_quant}. Specifically, Figure \ref{fig:from_init_quant} illustrates the progression of CLIP scores in relation to the second prompt P2. It is evident that initializing with the trajectory leads to significantly faster convergence in terms of the CLIP scores compared to initializing from noise.

\begin{figure}[htpb]
	\centering
	\begin{subfigure}[b]{0.16\columnwidth}
	\centering
	\includegraphics[width=\columnwidth]{figures/image_variation_stable/6/0_0.04375648498535156.png}
\end{subfigure}
\begin{subfigure}[b]{0.16\columnwidth}
	\centering
	\includegraphics[width=\columnwidth]{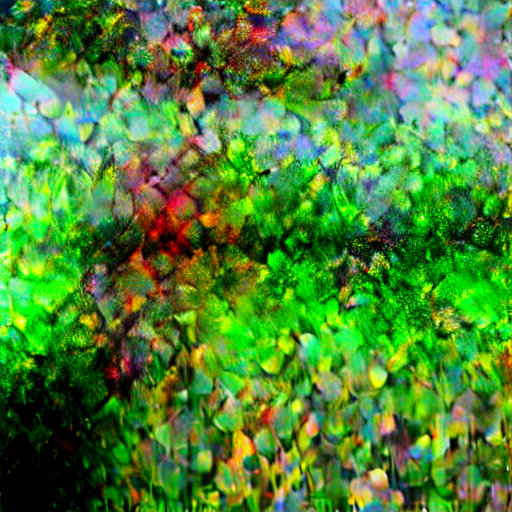}
\end{subfigure}
\begin{subfigure}[b]{0.16\columnwidth}
	\centering
	\includegraphics[width=\columnwidth]{figures/image_variation_stable/6/3_2.533912420272827.png}
\end{subfigure}
\begin{subfigure}[b]{0.16\columnwidth}
	\centering
	\includegraphics[width=\columnwidth]{figures/image_variation_stable/6/5_5.11636209487915.png}
\end{subfigure}
\begin{subfigure}[b]{0.16\columnwidth}
	\centering
	\includegraphics[width=\columnwidth]{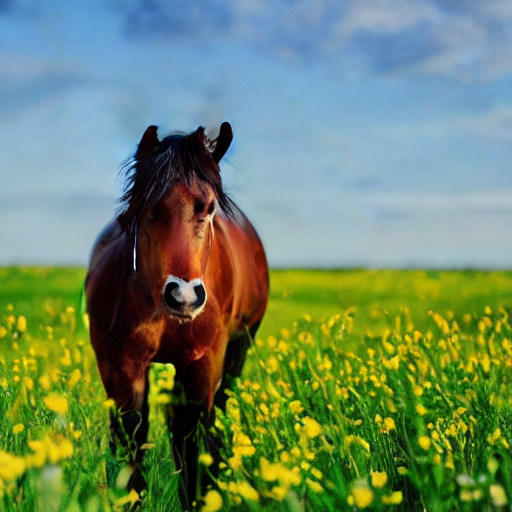}
\end{subfigure}
\begin{subfigure}[b]{0.16\columnwidth}
	\centering
	\includegraphics[width=\columnwidth]{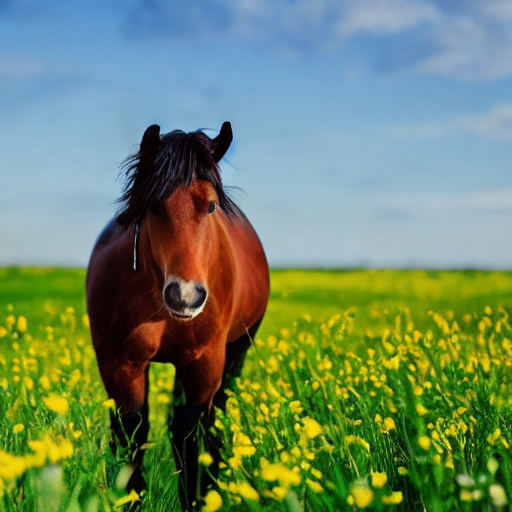}
\end{subfigure}

\begin{subfigure}[b]{0.16\columnwidth}
	\centering
	\includegraphics[width=\columnwidth]{figures/image_variation_stable/7/0_0.026067733764648438.png}
\end{subfigure}
\begin{subfigure}[b]{0.16\columnwidth}
	\centering
	\includegraphics[width=\columnwidth]{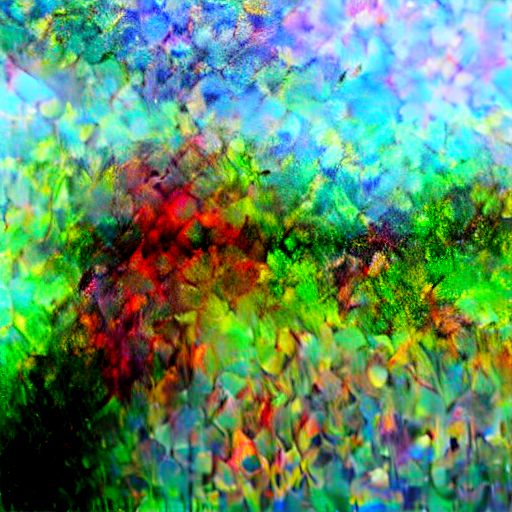}
\end{subfigure}
\begin{subfigure}[b]{0.16\columnwidth}
	\centering
	\includegraphics[width=\columnwidth]{figures/image_variation_stable/7/3_2.478393316268921.png}
\end{subfigure}
\begin{subfigure}[b]{0.16\columnwidth}
	\centering
	\includegraphics[width=\columnwidth]{figures/image_variation_stable/7/4_3.836852550506592.png}
\end{subfigure}
\begin{subfigure}[b]{0.16\columnwidth}
	\centering
	\includegraphics[width=\columnwidth]{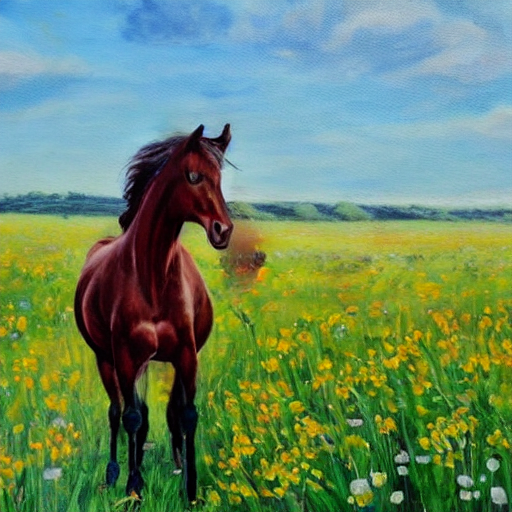}
\end{subfigure}
\begin{subfigure}[b]{0.16\columnwidth}
	\centering
	\includegraphics[width=\columnwidth]{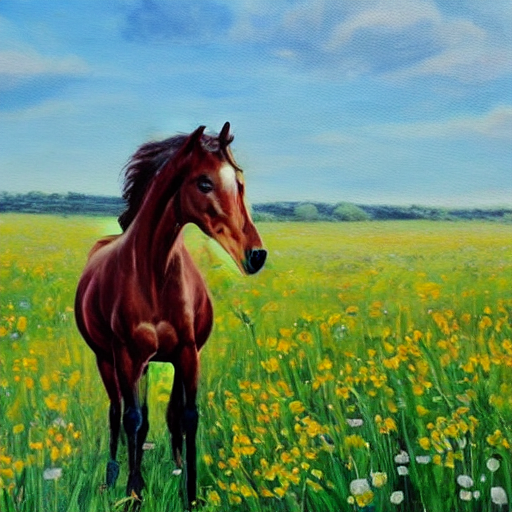}
\end{subfigure}

\begin{subfigure}[b]{0.16\columnwidth}
	\centering
	\includegraphics[width=\columnwidth]{figures/image_variation_stable/1/0_0.026752710342407227.png}
\end{subfigure}
\begin{subfigure}[b]{0.16\columnwidth}
	\centering
	\includegraphics[width=\columnwidth]{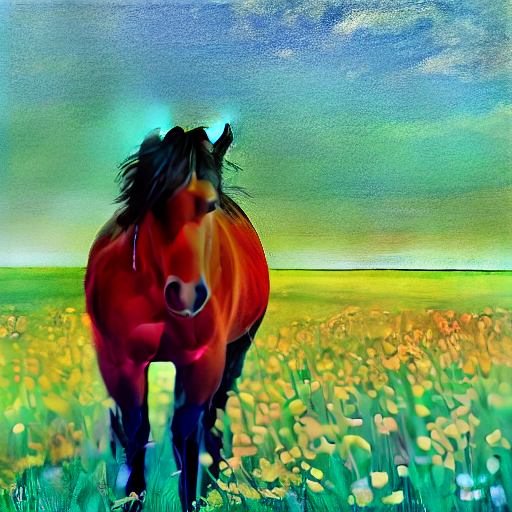}
\end{subfigure}
\begin{subfigure}[b]{0.16\columnwidth}
	\centering
	\includegraphics[width=\columnwidth]{figures/image_variation_stable/1/3_3.097738742828369.png}
\end{subfigure}
\begin{subfigure}[b]{0.16\columnwidth}
	\centering
	\includegraphics[width=\columnwidth]{figures/image_variation_stable/1/5_4.702269077301025.png}
\end{subfigure}
\begin{subfigure}[b]{0.16\columnwidth}
	\centering
	\includegraphics[width=\columnwidth]{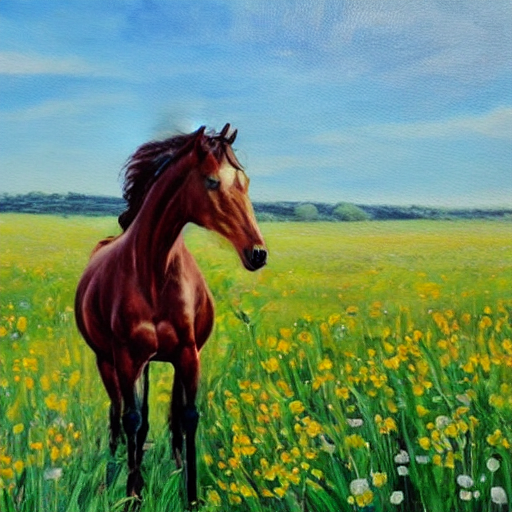}
\end{subfigure}
\begin{subfigure}[b]{0.16\columnwidth}
	\centering
	\includegraphics[width=\columnwidth]{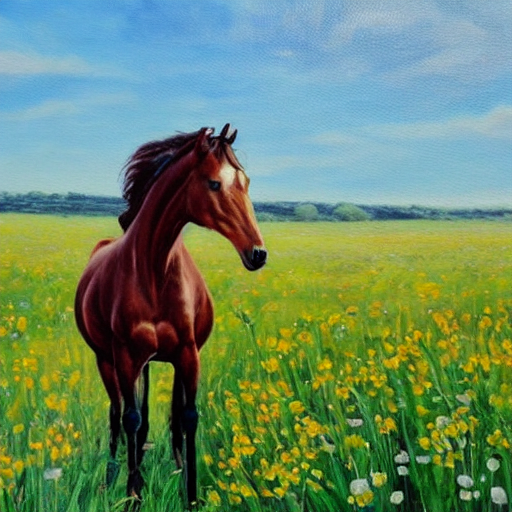}
\end{subfigure}

\begin{subfigure}[b]{0.16\columnwidth}
	\centering
	\includegraphics[width=\columnwidth]{figures/image_variation_stable/2/0_0.025284767150878906.png}
    \caption{Initialization}
\end{subfigure}
\begin{subfigure}[b]{0.16\columnwidth}
	\centering
	\includegraphics[width=\columnwidth]{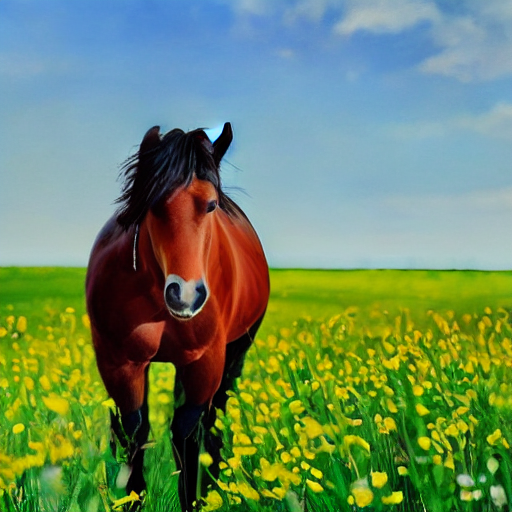}
    \caption{After 1 steps}
\end{subfigure}
\begin{subfigure}[b]{0.16\columnwidth}
	\centering
	\includegraphics[width=\columnwidth]{figures/image_variation_stable/2/3_2.1664271354675293.png}
    \caption{After 3 steps}
\end{subfigure}
\begin{subfigure}[b]{0.16\columnwidth}
	\centering
	\includegraphics[width=\columnwidth]{figures/image_variation_stable/2/5_4.052406311035156.png}
    \caption{After 5 steps}
\end{subfigure}
\begin{subfigure}[b]{0.16\columnwidth}
	\centering
	\includegraphics[width=\columnwidth]{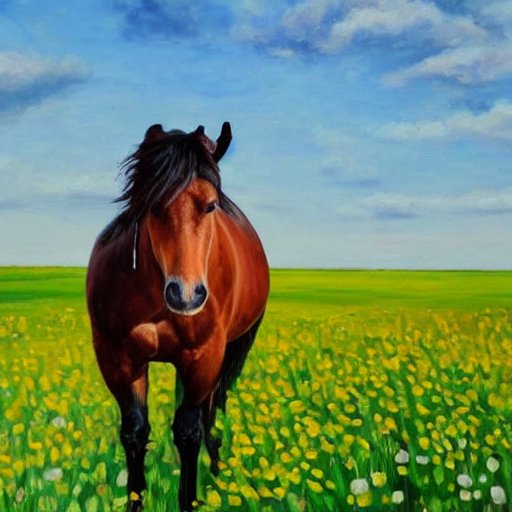}
    \caption{After 7 steps}
\end{subfigure}
\begin{subfigure}[b]{0.16\columnwidth}
	\centering
	\includegraphics[width=\columnwidth]{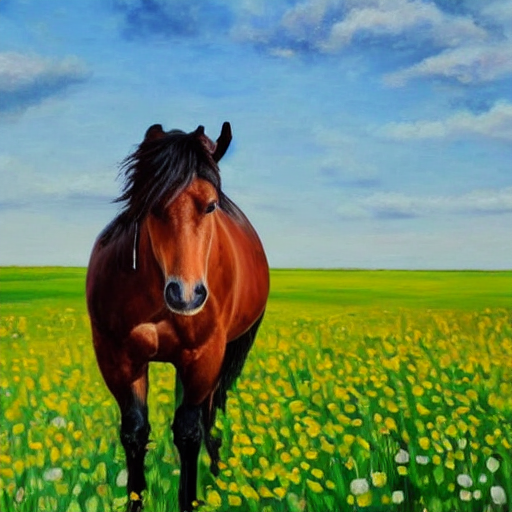}
    \caption{After 9 steps}
\end{subfigure}

\caption{Iterations of ParaTAA with different initaizliations. P1: "A 4k detailed photo of a horse in a field of flowers". P2: "An oil painting of a horse in a field of flowers".  From top to bottom, the rows represents: 1. Sampling with P1 with random initialization; 2. Sampling with P2 with random initialization. 3. Sampling with P2 with trajectory of P1 as initialization and $T_{\text{init}}=50$. 4. Sampling with P2 with trajectory of P1 as initialization and $T_{\text{init}}=35$. }
	\label{fig:early_stop_app}
\end{figure}

\begin{figure}[htpb]
	\centering
	\includegraphics[width=0.4\textwidth]{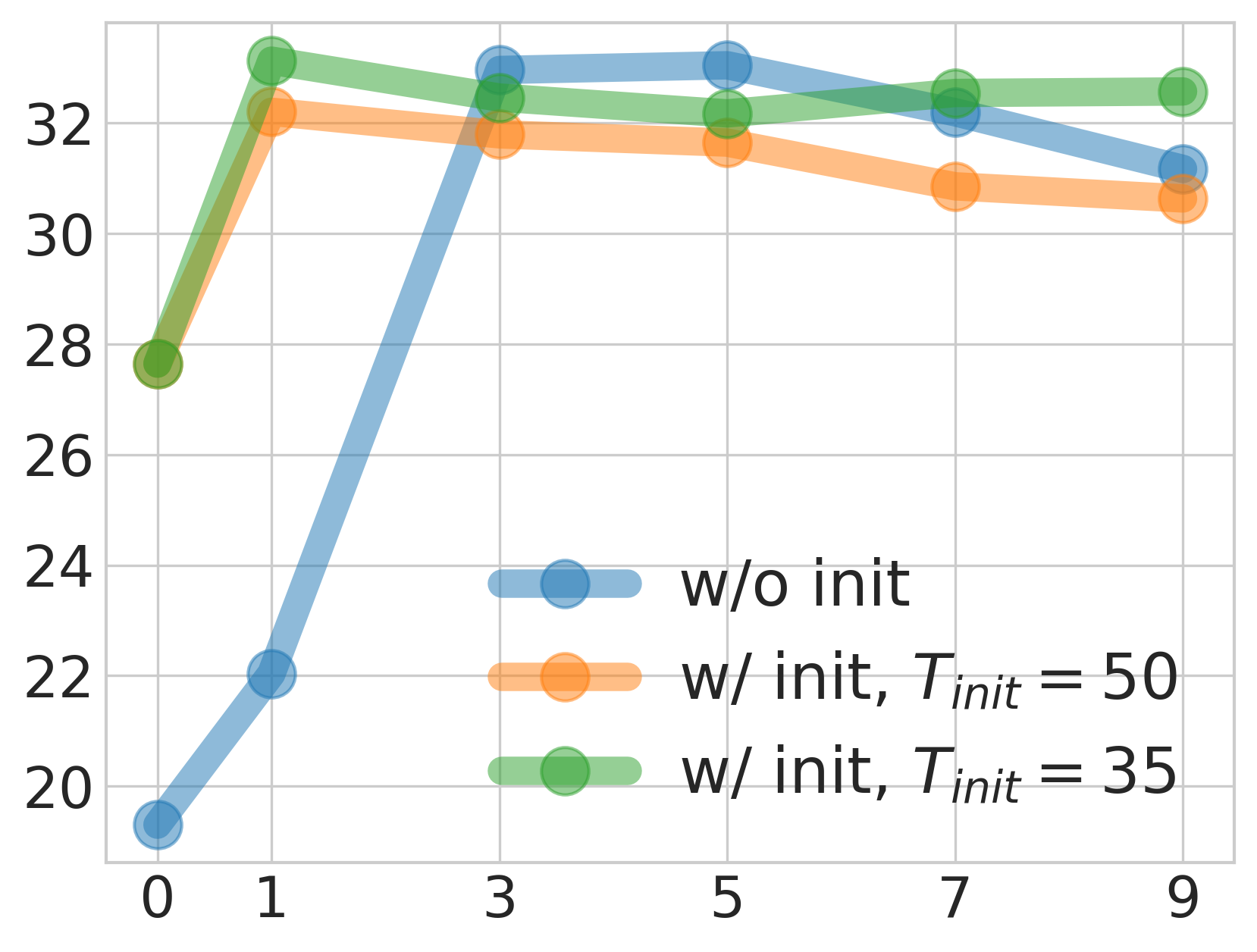}
\caption{Quantative evaluation for the three settings: 1. Sampling with P2 with random initialization. 2. Sampling with P2 with trajectory of P1 as initialization and $T_{\text{init}}=50$. 3. Sampling with P2 with trajectory of P1 as initialization and $T_{\text{init}}=35$. The y-axis is the CLIP scores w.r.t the prompt "An oil painting of a horse in a field of flowers".}
	\label{fig:from_init_quant}
\end{figure}

\section{Additional Examples of Smooth Image Variation}
\label{app:image_variation}

In Figure \ref{fig:demo_image_variation4}, additional examples are provided to demonstrate the capability of ParaTAA in facilitating smooth image transitions. Specifically, for DDIM with 50 steps, we utilize ParaTAA between two similar prompts, P1 and P2. Initially, we generate a trajectory from P1 using ParaTAA, which is then employed as the starting point for sampling from P2, with the initialization timestep $T_{\text{init}}$ set between 35 and 40. The results indicate that ParaTAA can lead  transformation from the source to the target image in a seamless manner along the image manifold within very few iteration steps.

\begin{figure}[htpb]
	\centering
	\begin{subfigure}[b]{0.85\columnwidth}
	\centering
	\includegraphics[width=0.245\columnwidth]{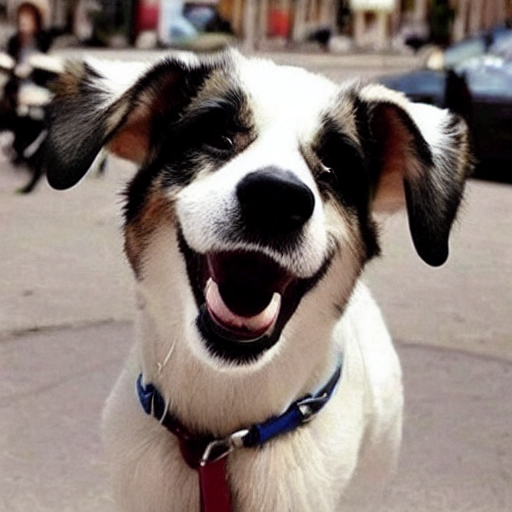}
\hfill
	\includegraphics[width=0.245\columnwidth]{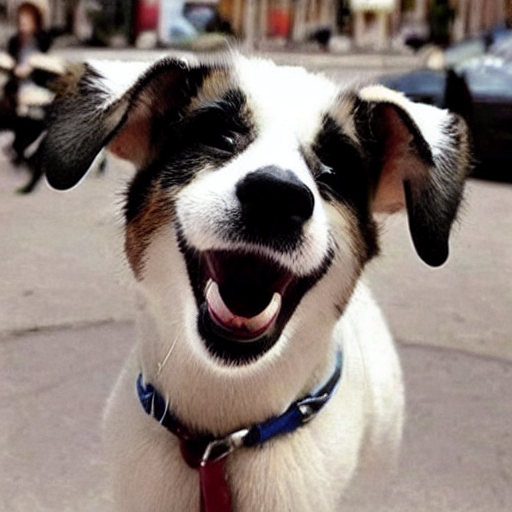}
\hfill
	\includegraphics[width=0.245\columnwidth]{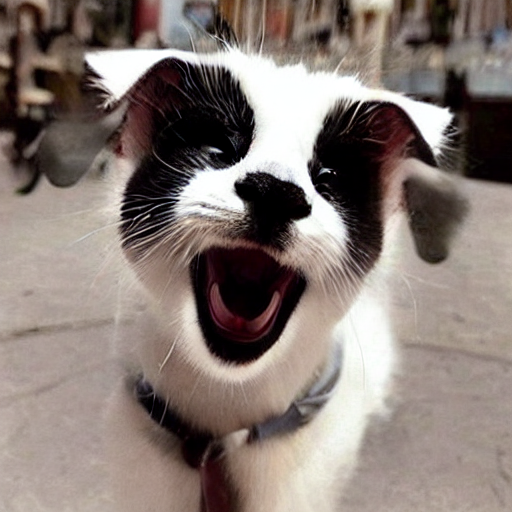}
\hfill
	\includegraphics[width=0.245\columnwidth]{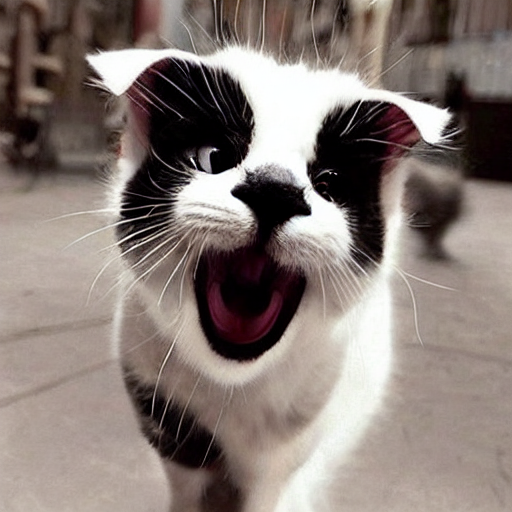}
    \caption{P1: "A cute dog" $\rightarrow$ P2: "A cute cat" }
\end{subfigure}
	\begin{subfigure}[b]{0.85\columnwidth}
	\centering
	\includegraphics[width=0.245\columnwidth]{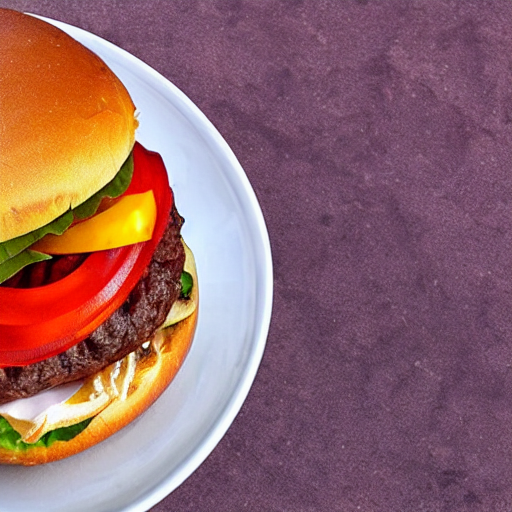}
\hfill
	\includegraphics[width=0.245\columnwidth]{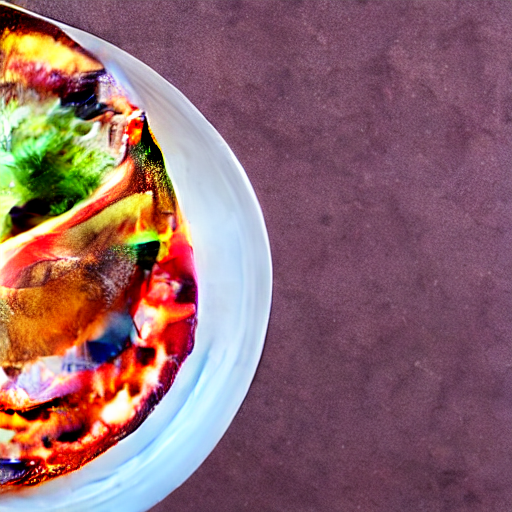}
\hfill
	\includegraphics[width=0.245\columnwidth]{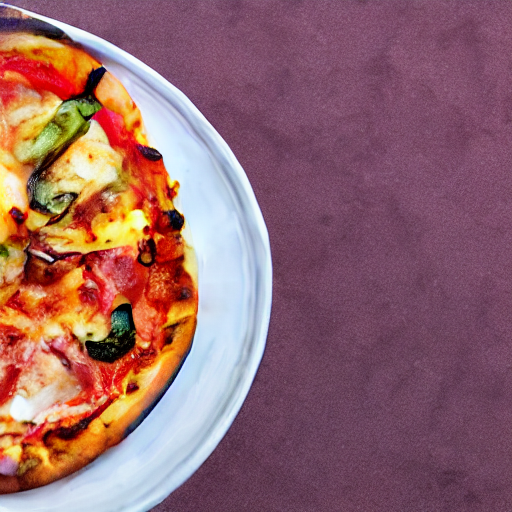}
\hfill
	\includegraphics[width=0.245\columnwidth]{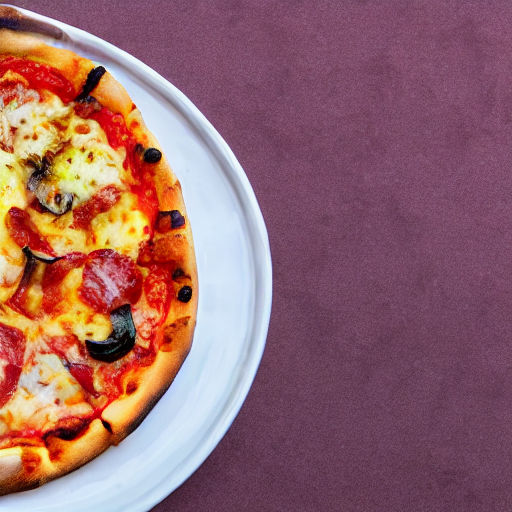}
    \caption{P1: "A delicious hamburger" $\rightarrow$ P2: "A delicious pizza" }
\end{subfigure}
	\begin{subfigure}[b]{0.85\columnwidth}
	\centering
	\includegraphics[width=0.245\columnwidth]{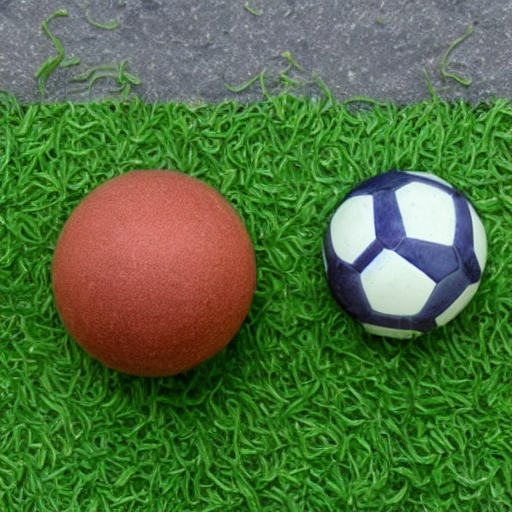}
    \hfill
	\includegraphics[width=0.245\columnwidth]{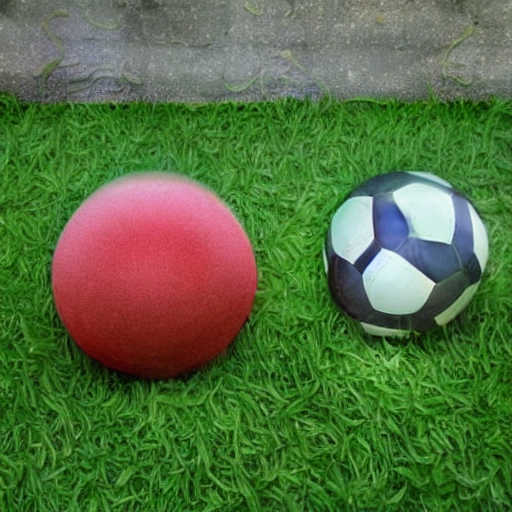}
    \hfill
	\includegraphics[width=0.245\columnwidth]{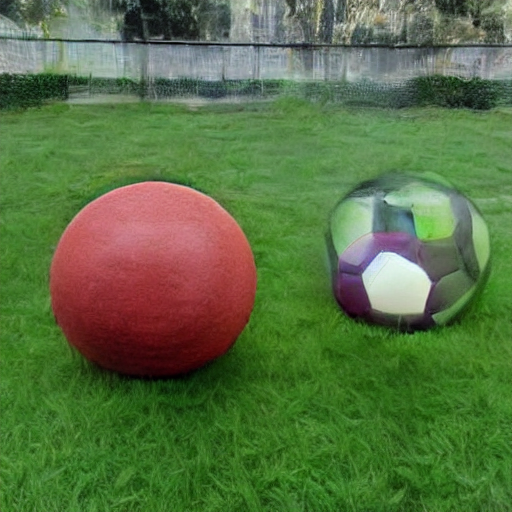}
    \hfill
	\includegraphics[width=0.245\columnwidth]{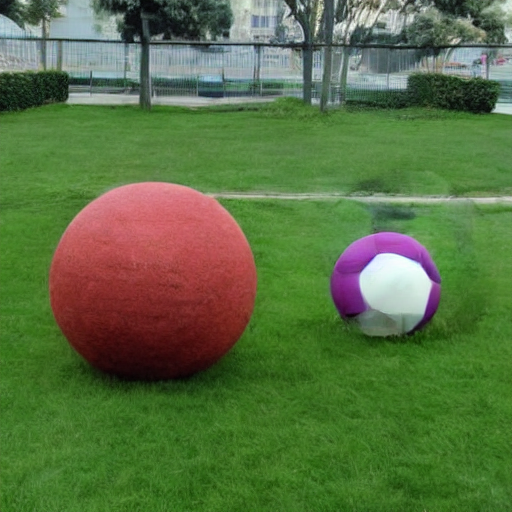}
    \caption{P1: "Two small balls" $\rightarrow$ P2: "Two huge balls" }
\end{subfigure}
	\begin{subfigure}[b]{0.85\columnwidth}
	\centering
	\includegraphics[width=0.245\columnwidth]{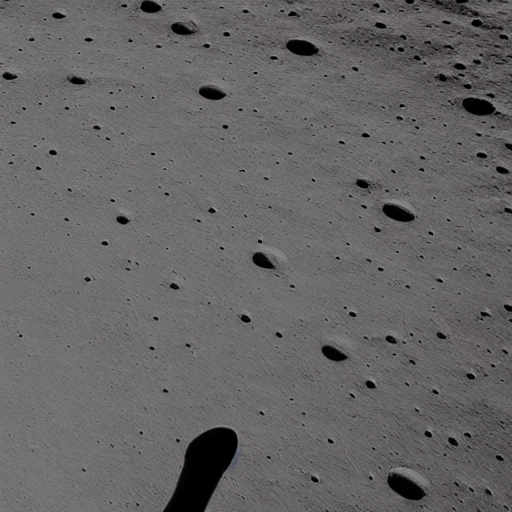}
    \hfill
	\includegraphics[width=0.245\columnwidth]{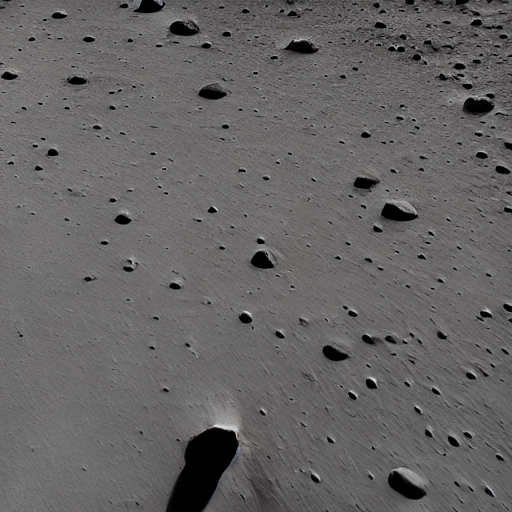}
    \hfill
	\includegraphics[width=0.245\columnwidth]{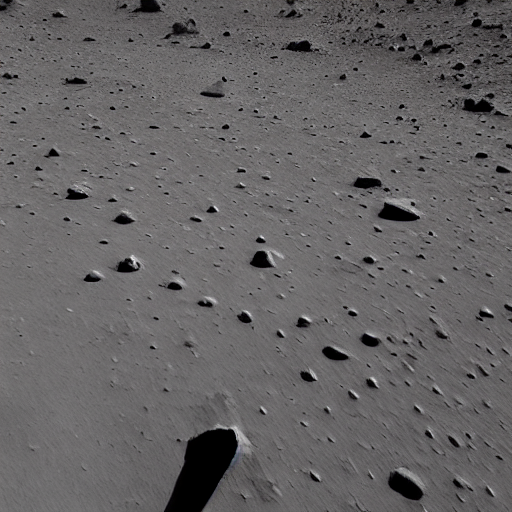}
\hfill
	\includegraphics[width=0.245\columnwidth]{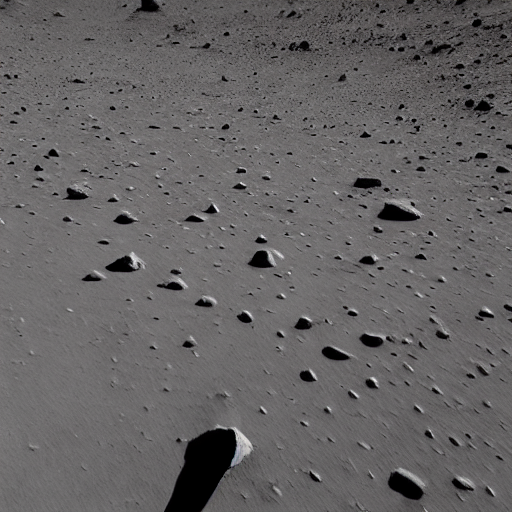}
    \caption{P1: "Walking on Moon" $\rightarrow$ P2: "Walking on Mars" }
\end{subfigure}
\caption{Iterations of ParaTAA using an existing trajectory for initialization. From left to right, the columns represent the initial image, the image after 1 step, the image after 3 steps, and the image after 5 steps.}
\label{fig:demo_image_variation4}
\end{figure}

\end{document}